\documentclass[10pt,journal,compsoc]{IEEEtran}
\usepackage{times}
\usepackage[pdftex]{graphicx}  % pami. PDF only
\usepackage{amsmath}
\usepackage{siunitx,soul,fix-cm} % for units of measurements
\DeclareSIUnit[] \pixel{\text{px}}
\usepackage{amssymb}
\usepackage{amsthm}
\usepackage{mathtools,empheq}
\usepackage[caption=false,font=footnotesize,labelfont=sf,textfont=sf]{subfig} % PAMI
\usepackage{verbatim}   % for the comment environment
\usepackage{color} % xcolor?
\usepackage{nomencl}
\usepackage{url}

\usepackage[nocompress]{cite}  % PAMI
\usepackage[ruled,vlined]{algorithm2e}
\usepackage{enumerate}
\usepackage{multirow}
\usepackage{macrosfabbri-basic}
\usepackage{macrosfabbri-dg}

\usepackage[,breaklinks=true,letterpaper=true,colorlinks,bookmarks=false]{hyperref}
\usepackage[utf8]{inputenc}
\usepackage{booktabs}

\usepackage{listings} % better verbatim env for sketching outlines/lists
\DeclareMathOperator{\rank}{rank}

\newcommand{\srot}{\mathtt{R}} % rotation in simplified notation
\newcommand{\stransl}{\mathbf{t}}  % translation vector in simplified notation
\newcommand{\D}{\mathbf{D}} % 3D tangent direction in simplified notation
\newcommand{\dir}{\mathbf{d}} % 2D tangent direction in simplified notation

\newcommand{\hide}[1]{}
\graphicspath{{figs/}}

\newtheorem{theorem}{Theorem}[section]
\newtheorem{corollary}[theorem]{Corollary}

\theoremstyle{definition}
\newtheorem{definition}{Definition}
\newcommand{\bl}{{\mathbf l}}

\newcommand{\bpi}{{\boldsymbol \pi}}
\newcommand{\CC}{\mathbb{C}}

\newcommand{\PP}{\mathbb{P}}
\newcommand{\LL}{\mathtt{L}}
\newcommand{\relu}[1]{\langle #1 \rangle}

\newtagform{coloured}[\bfseries]{\color{red}(}{\mdseries)}
\newtagform{emph}[\textbf]{(}{)}

\begin{document}
% Uncomment and include ieee-config.bib to control how refs appear
% \bstctlcite{MyBSTcontrol}
\setstcolor{red}
\title{Trifocal Relative Pose from Lines at Points}
\author{%
Ricardo~Fabbri,
Timothy~Duff,
Hongyi~Fan,
Margaret~H.~Regan,
David~da~C.~de~Pinho,\\
Elias~Tsigaridas,
Charles~W.~Wampler,
Jonathan~D.~Hauenstein,
Peter~J.~Giblin,\\
Benjamin~Kimia,
Anton~Leykin and
Tomas~Pajdla
\IEEEcompsocitemizethanks{
\IEEEcompsocthanksitem R.~Fabbri is with the Department of Computational
Modeling, Polytechnic Institute, Rio de Janeiro State University, Nova Friburgo,
Brazil. He is supported by UERJ Proci\^{e}ncia award, NSF IIS-1910530 and FAPERJ Jovem Cientista do
Nosso Estado E-26/201.557/2014.\protect\\
E-mail: rfabbri@gmail.com
\IEEEcompsocthanksitem A.~Leykin is at Georgia Tech and are supported by NSF DMS-1151297.
\IEEEcompsocthanksitem T.~Duff is at University of Washington.
He acknowledges support from an NSF Mathematical Sciences Postdoctoral Research Fellowship (DMS-2103310), as well as partial support from NSF DMS-1151297.
\IEEEcompsocthanksitem B.~Kimia and H.~Fan are with the School of Engineering, Brown University.
They are supported by the NSF grant IIS-1910530.
\IEEEcompsocthanksitem D.~de Pinho is with UENF, Brazil
\IEEEcompsocthanksitem P.~Giblin is with the University of Liverpool.
\IEEEcompsocthanksitem M.~Regan is with Duke University and is
supported by NSF CCF-1812746 and the Schmitt Leadership Fellowship in Science
and Engineering.
\IEEEcompsocthanksitem E.~Tsigaridas is with INRIA Paris.
\IEEEcompsocthanksitem C.~Wampler is with the University of Notre Dame.
\IEEEcompsocthanksitem J.~Hauenstein is with the University of Notre Dame and is
supported by NSF CCF-1812746 and ONR N00014-16-1-2722.
\IEEEcompsocthanksitem T.~Pajdla is with the Czech Institute of Informatics,
Robotics and Cybernetics, Czech Technical University in Prague, and is supported
by the EU Regional Development Fund IMPACT CZ.02.1.01/0.0/0.0/15 003/0000468 and
EU H2020 project ARtwin 856994.
\IEEEcompsocthanksitem This work initiated while the authors were
in residence at Brown University's Institute for Computational and Experimental
Research in Mathematics -- ICERM, in Providence, RI, during the Fall 2018 and
Spring 2019 semesters, with early versions at ar\textsc{x}iv 23Mar2019 04:26UTC
and CVPR 2020.  (NSF DMS-1439786 and the Simons Foundation grant
507536).%
}%
\thanks{Manuscript received Month Day, Year; revised Month Day, Year.}}


\markboth{Pattern Analysis and Machine Intelligence. \ \ \ \ \ \ \ \ \ \ Accepted version Copyright \copyright\ IEEE.}%
{Accepted}



\IEEEtitleabstractindextext{%
\begin{abstract}
We present a method for solving two minimal problems for relative camera pose
estimation from three views, which are based on three view correspondences of
({\em i}) three points and one line and the novel case of ({\em ii}) three points and two lines
through two of the points. These problems are too difficult to be efficiently
solved by the state of the art Gr\"obner basis methods. Our method is based on a
new efficient homotopy continuation (HC) solver framework MINUS, which dramatically speeds up
previous \hc\ solving by specializing \hc\ methods to generic cases of our problems.
We characterize their number of solutions and show with simulated experiments
that our solvers are numerically robust and stable under image noise, a key
contribution given the borderline intractable degree of nonlinearity of
trinocular constraints. We show in real experiments that ({\em i}) \sift\ feature
location and orientation provide good enough point-and-line correspondences for
three-view reconstruction and ({\em ii}) that we can solve difficult cases with
too few or too noisy tentative matches, where the state of the art structure
from motion initialization fails.
\end{abstract}}
%\begin{IEEEkeywords}
%Pose Estimation, Camera Resectioning, Differential Geometry
%\end{IEEEkeywords}}

% PAMI -------------------------------------------------------------------------
% \fi
% \IEEEpeerreviewmaketitle
% !PAMI-------------------------------------------------------------------------

\maketitle

\IEEEdisplaynontitleabstractindextext
% \fi
\IEEEpeerreviewmaketitle

%\tableofcontents
%\mynewpage
%\listoffigures
%\onehalfspacing

%
%%%%%%%%%%%%%%%%%%%%%%%%%%%%%%%%%%%%%%%%%%%%%%%%%%%%%%%%%%%%%%%%%%%%%%%%%%%%%%%%%%%%%%%%%%%%%%%%
%%%%%%%%%%%%%%%%%%%%%%%%%%%%% INTRODUCTION %%%%%%%%%%%%%%%%%%%%%%%%%%%%%%%%%%%%%%%%%%%%%%%%%%%%%
%%%%%%%%%%%%%%%%%%%%%%%%%%%%%%%%%%%%%%%%%%%%%%%%%%%%%%%%%%%%%%%%%%%%%%%%%%%%%%%%%%%%%%%%%%%%%%%%
\IEEEraisesectionheading{\section{Introduction}\label{sec:intro}}

\IEEEPARstart{S}{cientific} research on \textsc{3d} reconstruction from multiple views has made an impact~\cite{AppleKeynote:2018} by mostly relying
%\IEEEPARstart{3D}{reconstruction} has made an impact~\cite{AppleKeynote:2018} by mostly relying
on points in Structure from Motion
(\sfm)~\cite{Argarwal:Snavely:etal:ICCV09,schoenberger2016sfm,Furukawa:Ponce:PAMI2010,Nister04visualodometry}.
Still, even production-quality \sfm\ technology fails~\cite{AppleKeynote:2018}
when the images contain ({\em i}) large homogeneous areas with few or no
features; ({\em ii}) repeated textures, like brick walls, giving rise to a large
number of ambiguously correlated features; ({\em iii}) blurred areas, arising
from moving cameras or objects; ({\em iv}) large scale changes where the overlap
is not sufficiently significant; or ({\em v}) multiple and independently moving
objects each lacking a sufficient number of features.%
\begin{figure}[t]
\includegraphics[width=0.47\linewidth]{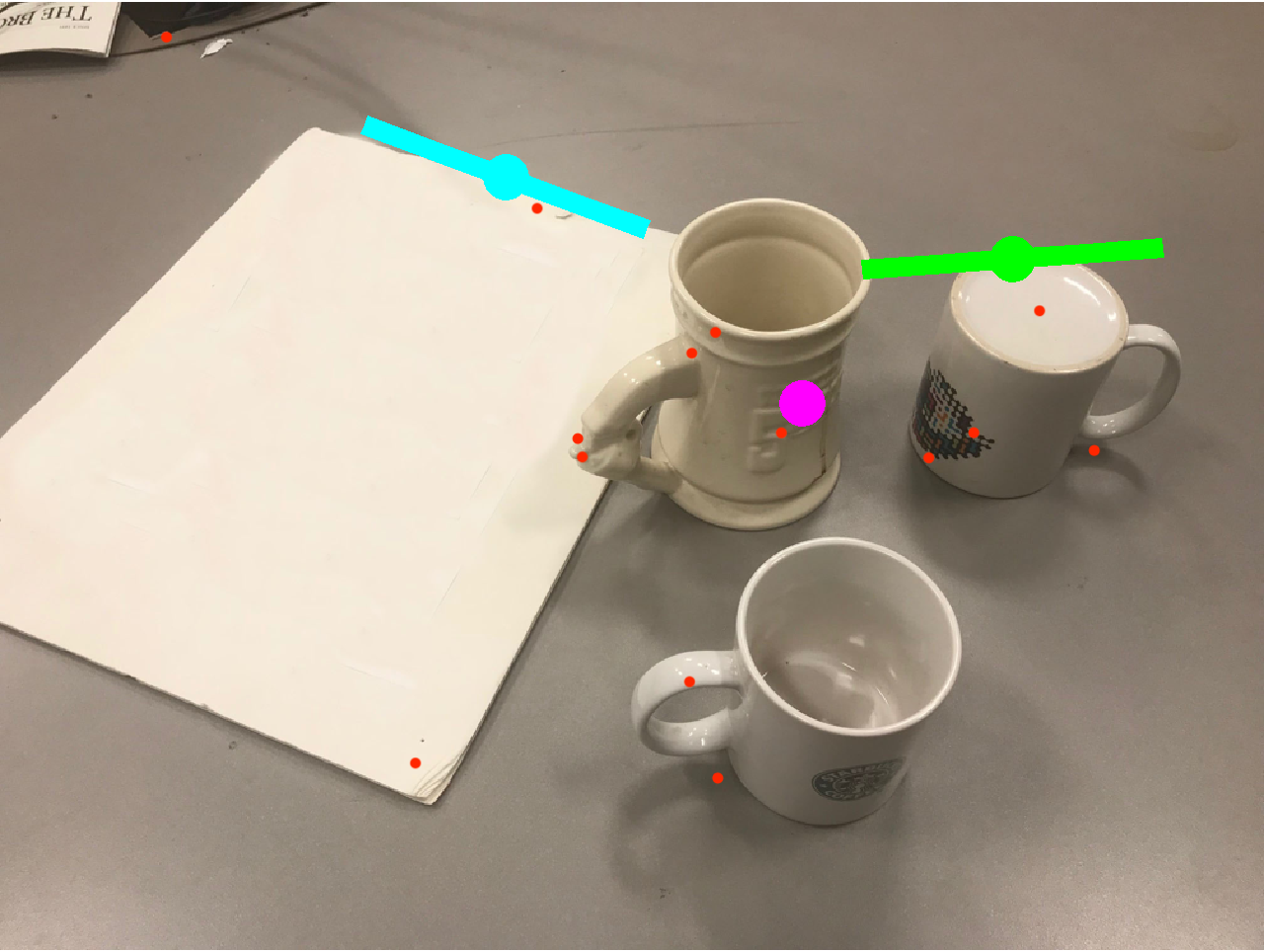}\hspace{0.45mm}%
\includegraphics[width=0.47\linewidth]{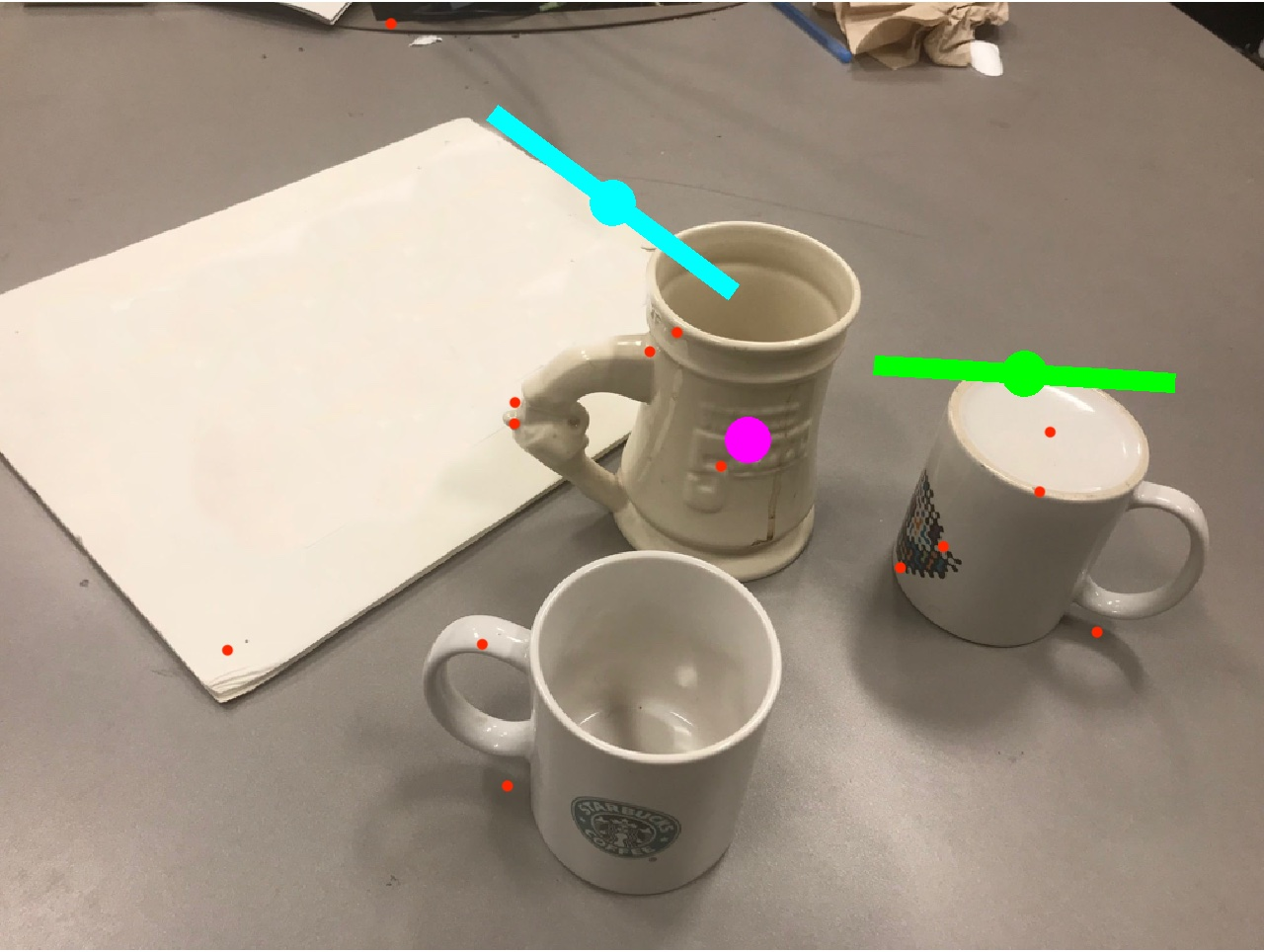}
  \includegraphics[width=0.47\linewidth]{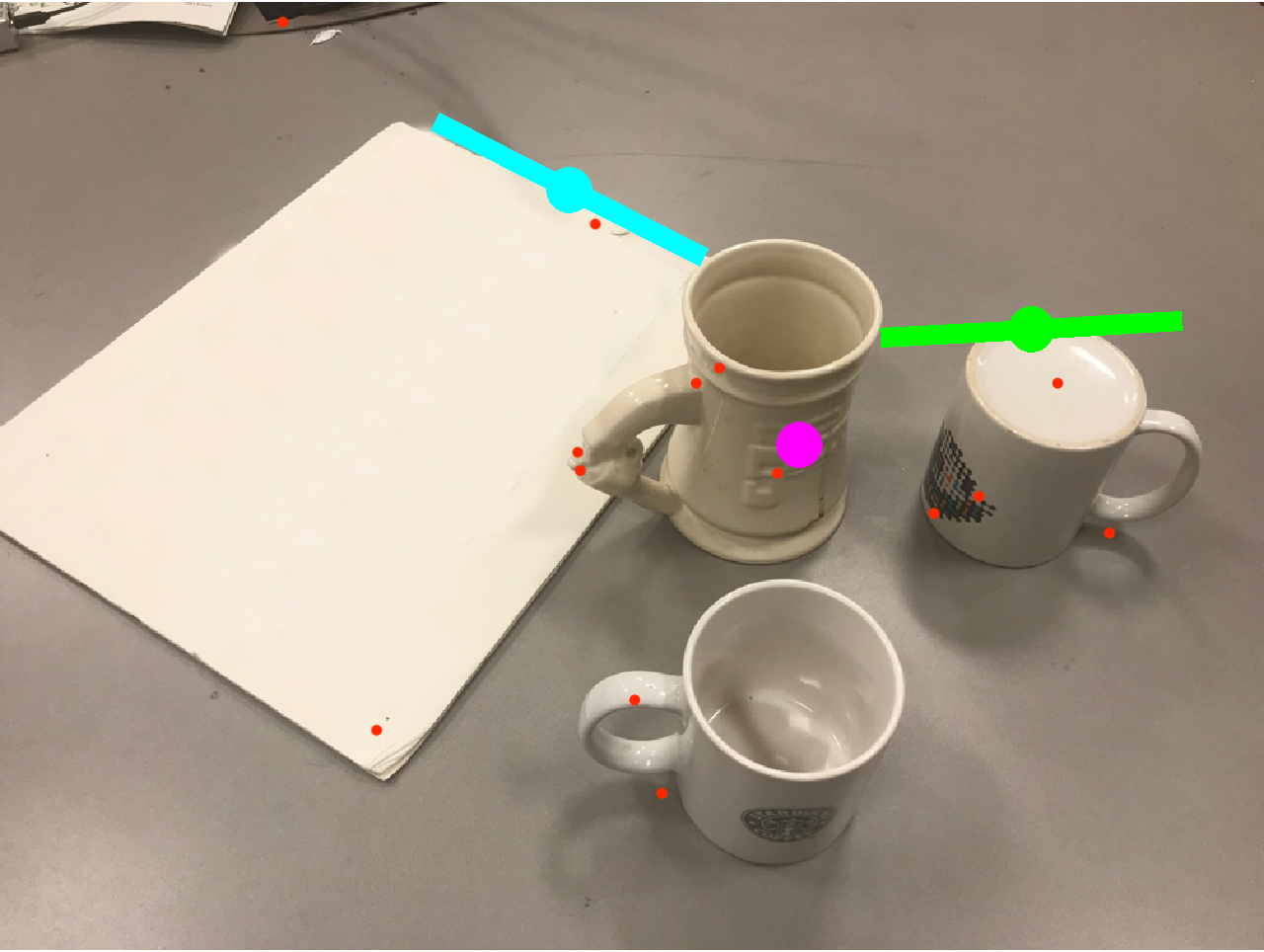}
  \includegraphics[width=0.47\linewidth]{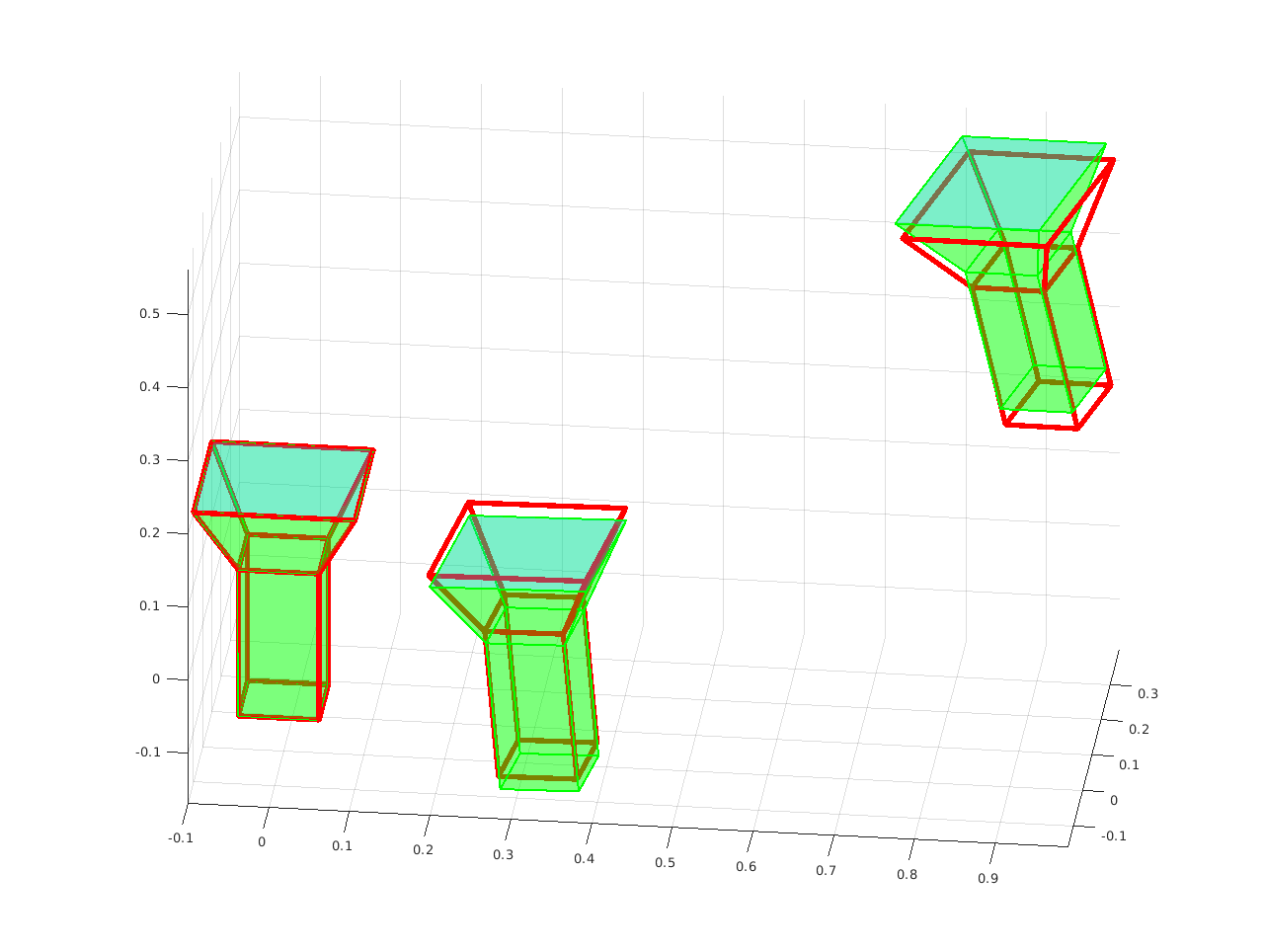}
  \caption{A deficiency of the traditional two-view approach to bootstraping
    \sfm: not enough features detected (red dots) and thus a \textsc{sota}
    \sfm\ pipeline \textsc{colmap} fails to reconstruct the
    relative camera pose. In contrast, the proposed trinocular method requires only
three features corresponding in three views: two point-tangents (points with \sift\
orientation, green and cyan) and one point without
orientation (purple, also \sift). Red cameras
are computed by our approach and green is ground truth.}
  \label{fig:teaser}
\end{figure}%

The failure of bifocal pose estimation using \ransac\ on hypothesized
correspondences, \textit{e.g.,} using 5 points~\cite{Nister:PAMI04}, is
highlighted in a dataset of images of mugs, Fig.~\ref{fig:teaser} (similar to
the dataset in~\cite{nurutdinova2015towards} but without a calibration board),
for which the failure rate using the standard \sfm\ pipeline \colmap~[63] is 75\%.
The failure of just directly applying the 5-point algorithm in this example is
even higher.  A similar situation exists for images containing repeated patterns
where there are plenty of features, but determining correspondences is
challenging. Most traditional multiview pipelines estimate the relative pose of
the two best views and then register the remaining views using a \textsc{p3p}
algorithm~\cite{Argarwal:Snavely:etal:ICCV09}, for robustness. The focus of
this paper is to address the failure of traditional bifocal algorithms
in such cases, while tackling strategic problems that have long-term potential
for
breakthrough for a myriad of other minimal problems we jointly discovered and
tackled~\cite{Duff:Kohn:Leykin:Pajdla:PLMP:2019,Duff:Korotynskiy:Pajdla:Regan:Arxiv2021,Fabbri:NSF:Highlight2020,Fabbri:Kimia:NSF:Grant},
and in the case of curve features for \sfm\ which critically depend on trifocal
geometry~\cite{Fabbri:Kimia:IJCV2016,Fabbri:Giblin:Kimia:PAMI19,Fabbri:Giblin:Kimia:ECCV12,Fabbri:PHD:2010,Schmid:Zisserman:2000}.

The failure of bifocal algorithms motivates the use of ({\em i}) more complex
features, \textit{i.e.,} having additional attributes and ({\em ii}) more
diverse features (partial visibilitiy also helps in robustness, see~\cite{Vasconcelos:etal:PAMI2017}). We propose that {\em orientation} (in the sense of
inclination) is a key attribute to disambiguate correspondences and we show that
\sift\ orientation in particular is a stable feature across views for trifocal
pose estimation. Orientation can also arise from curve
tangents~\cite{Fabbri:Giblin:Kimia:ECCV12,Fabbri:PHD:2010,Barath:Kukelova:ICCV2019},
and the {\em orientation} of a straight line in multiple views also constrains
pose. Observe, however, that orientation cannot be constrained in two views
alone: \sift\ orientation or line orientations in two views are uncorrelated, but
together can identify their \textsc{3d} counterparts and thus can constrain orientation
in a third view. This motivates {\em trinocular pose estimation} based on point
features endowed with orientation or including straight line
features~\cite{Fabbri:Kimia:IJCV2016,Fabbri:Giblin:Kimia:PAMI19,Fabbri:Giblin:Kimia:ECCV12}. 
% The solutions of these concrete cases also have potential for future use as local versions of pose from general curves~\cite{Fabbri:Giblin:Kimia:ECCV12,Fabbri:PHD:2010}, Figure~\ref{fig:trifocal:pose}, to tackle further challenges~\cite{Nurutdinova:Fitzgibbon:ICCV2015}.  
% \begin{figure}
%   \centering
%   \includegraphics[width=0.8\linewidth]{figs/trifocal-pose-from-curves-tangents.pdf}
%   \caption{
%     Determining trifocal pose from curves in the future can be locally encoded in three
%     coresponding points with tangent information (\textit{e.g.,} edgels)~\cite{Fabbri:PHD:2010}, but we show this
%     is overconstrained. Dropping one tangent gives the
%     \emph{Chicago} minimal problem, or considering only the three points and one
%     free line gives \emph{Cleveland}, both solved here.
%   }
%   \label{fig:trifocal:pose}
% \end{figure}

Camera estimation from trifocal tensors is long believed to augment two-view pose
estimation~\cite{Faugeras:Luong:Book,Trager:Hebert:Ponce:2019}. 
Although no significant improvements over bifocal pairwise
estimation have been documented~\cite{Julia:Monasse:SIVT2017}, recent work
reiterate the advantages of
well-crafted trifocal algorithms for relevant near-degenerate configurations
such as approximately collinear camera centers~\cite{Trager:Hebert:Ponce:2019,Ponce-IJCV-2016}. The calibrated trinocular relative
pose estimation from four points, \textsc{3v4p}, is notably difficult to solve~\cite{Nister:Schaffalitzky:IJCV2006,Quan:Triggs:Mourrain:JMIV06,Quan:Triggs:Mourrain:Ammeller:2003bUnpublished,Fabbri:PHD:2010}, and is not a minimal problem -- it is over-constrained. The first working trifocal solver~\cite{Nister:Schaffalitzky:IJCV2006} effectively
parametrizes the relative pose between two cameras as a curve of
degree ten representing possible epipoles. A third view is then used to select the epipole that minimizes reprojection errors. In this sense, trinocular pose estimation has not
truly been tackled as a minimal problem. 

Trifocal pose estimation requires the determination of 11 degrees of freedom:
six unknowns for each pair of rotation $\srot$ and translation $\stransl$, less
one for metric ambiguity. Three types of constraints arise in matching
triplets of point features endowed with orientation. First, the epipolar
constraint provides an equation for each pair of correspondences in two views.
Second, in a triplet of correspondences, each pair of correspondences are
required to match scale, providing another constraint; a total
of three equations per triplet. It is easy to see, informally, that three points
are insufficient to determine trifocal pose, while four points are too many.
Third, each triplet of oriented feature points provides one orientation
constraint. Thus, with three points, only two points need to be endowed with
orientation, giving a total of 11 actual constraints for the 11 unknowns. We refer to
this novel problem of three triplets of corresponding points, with two of the points
having oriented features as “{\em Chicago}”, which evolved out
of the work by Fabbri, Giblin and Kimia on absolute pose estimation from two
points endowed with tangents~\cite{Fabbri:Giblin:Kimia:PAMI19, Fabbri:Giblin:Kimia:ECCV12}. In
the second scenario, {\em i.e.}, using straight lines as features, with three
points, only one free (unattached to a point) straight line feature is required.
We refer to the problem of three triplets of corresponding points and one
triplet of corresponding free lines as “{\em Cleveland}.” This paper addresses
trifocal pose estimation for the above two scenarios, shows that both are
minimal problems, and develops efficient solvers for the resulting polynomial
systems.

Specifically, each problem comprises eleven trifocal constraints that
\emph{in principle} give systems of eleven polynomials in eleven unknowns.  These systems
are not trivial to solve and require techniques from numerical algebraic
geometry~\cite{BertiniBook,duff2018solving,NAG4M2} \textit{(i)} to probe whether
the system is over or under constrained or otherwise minimal; \textit{(ii)} to
understand the range of the number of real solutions and estimate a \emph{tight}
upper bound; and \textit{(iii)} to develop efficient and practically relevant
methods for finding solutions which are real and represent camera
configurations. This paper shows that the Chicago problem is minimal and has up
to 312 solutions (the area code of Chicago) of which typically 3-4 end up
becoming relevant to camera configurations. Similarly, we show that the
Cleveland problem is minimal and has up to 216 solutions (the area code of
Cleveland). The minimality
of combinations of points and lines for the general
case~\cite{duff2019plmp} is a parallel development to the more concrete treatment presented here.

The numerical solution of polynomial systems with several hundred
solutions is challenging. We devised a custom-optimized {\em Homotopy
Continuation (HC)} framework MINUS which iteratively tracks solutions with a guarantee
of global convergence~\cite{duff2018solving}. Our framework specializes the
general \hc\ approach to minimal problems typical of multiple view geometry,
thereby dramatically speeding up the implementation. Specifically, our Chicago
and Cleveland solvers are not only the first solvers for such high degree
problems, but are orders of magnitude faster than solvers for such scale of
problems: 660ms on average on an Intel core i7-7920HQ processor with four
threads. They share the same generic core procedure with plenty of room to be
further optimized for specific applications.  Most significantly, since finding each solution is a
completely independent integration path from the others, the solvers are suitable
for implementation on a \textsc{gpu}, as a batch for \ransac, which would then reduce the
run time by the number of tracks, \textit{i.e.,} by two orders of magnitude. We
hope that our developments can be a template for solving other computer vision
problems involving systems of polynomials with a large number of solutions, and
in fact the provided \textsc{c}++ framework is fully templated to include new
minimal problems seamlessly.

Experiments are initially reported on complex synthetic data to demonstrate that
the system is robust and stable under spatial and orientation noise and under a
significant level of outliers. Experiments on real data first demonstrates that
\sift\ orientation is a remarkably stable cue over a wide variation in view. We
then show that our approach is successful in all cases where the traditional
\sfm\ pipeline succeeds, but of course at higher computational cost. What is
critically important is that the proposed approach succeeds in many other cases
where the \sfm\ pipeline fails, \textit{e.g.,} on the
\textsc{epfl}~\cite{Strecha:etal:CVPR08} and Amsterdam Teahouse
datasets~\cite{Usumezbas:Fabbri:Kimia:ECCV16}, Fig.~\ref{fig:realData}
and~\ref{fig:fail}. Those cases where the bifocal scheme fails -- flagged by the
number of inliers, for example -- can consider the application of a currently
more expensive but more capable trifocal scheme to allow for reconstructions
that would otherwise be unsolved. 
%
%%%%%%%%%%%%%%%%%%%%%%%%%%%%%%%%%%%%%%%%%%%%%%%%%%%%%%%%
\subsection{Literature Review}
\noindent \textbf{Trifocal Geometry.}
Calibrated trifocal geometry estimation is a hard problem~\cite{Nister:Schaffalitzky:IJCV2006,Quan:Triggs:Mourrain:JMIV06,Quan:Triggs:Mourrain:Ammeller:2003bUnpublished,Rodehorst:ICACS2015}.
There are no publicly available solvers we are aware of. The state of the
art solver~\cite{Nister:Schaffalitzky:IJCV2006}, based on four corresponding
points (\textsc{3v4p}), has not yet found many practical applications~\cite{Kuang:Astrom:etal:ICPR2014}.
A solver for a relaxed version of this problem has been recently made available by
our coauthors based on techniques originated in the present paper~\cite{Hruby:Duff:Leykin:Pajdla:CVPR2022}.

%Relevant previous work on estimation of trifocal geometry can be divided into
%calibrated, uncalibrated, or partially calibrated, with the uncalibrated
%(projective) case being more common. The correspondence data can
%be all point correspondences, all line correspondences, or, less commonly,
%mixed.  These correspondences typically appear in all views.  More complex
%features such as conics, higher order algebraic curves, or points with incident
%lines, or curvature measurements can also be employed, but is yet
%under-explored.

For the projective case, 6 points are needed~\cite{Heyden:ICCV1995}, and
Larsson \textit{et al.}\ solved the longstanding trifocal minimal problem using 9
lines~\cite{larsson2017efficient}. The case of mixed
points and lines is less common~\cite{Oskarsson:etal:IVC2004}, but has seen a
growing interest in related
problems~\cite{Salaun:etal:ECCV2016,Quadir:Neubert:2017,Vakitov:etal:ECCV2018}. 
The calibrated cases beyond \textsc{3v4p} are largely unsolved, spurring 
sophisticated theoretical
work~\cite{Aholt:Oeding:MathComp2014,Alzati:Tortora:JMIV2010,Kileel:SIAMJAAG2017,Leonardos:Tron:Daniilidis:CVPR2015,Martyushev:JMIV:2017,Mathews:Arxiv:AG2016,Oeding:Quadfocal:Arxiv2015}.
Kileel~\cite{Kileel:SIAMJAAG2017}
studied minimal problems in this setting, such as the
Cleveland problem solved in the present paper, and reported 
studies using \hc. He stated 
that the \emph{full} set of ideal generators, \ie, a
set of polynomial equations provably necessary and sufficient to describe 
calibrated trifocal geometry, was unknown.

Seminal works used curves and edges in three views to transfer differential
geometry for matching~\cite{Ayache:Lustman:1987,Robert:Faugeras:1991}, and
for pose and trifocal tensor
estimation~\cite{Giblin:Motion:Book,Schmid:Zisserman:2000}, beyond 
straight lines for uncalibrated~\cite{Hartley:Zisserman:multiple:view,Bartoli:Sturm:CVIU2005}
and calibrated~\cite{Salaun:etal:3DV2017,Salaun:etal:ECCV2016} \sfm.
Point-tangents -- not to be confused with point-rays~\cite{Camposeco:Pollefeys:etal:ECCV2016} -- can be framed as
\emph{quivers} (1-quivers), or feature points with attributed
directions (\textit{e.g.,} corners), initially proposed in the context of uncalibrated
trifocal geometry but de-emphasizing the connection to tangents to
curves~\cite{Johansson:Oskarson:Astrom:2002,Zhao-PAMI-2019}. 
Point-tangent fields can be framed as vector fields, so related
technology applies to surface-induced correspondence data~\cite{Fabbri:PHD:2010}.
In the calibrated setting, point-tangents were first used for absolute pose estimation
by Fabbri \etal~\cite{Fabbri:Giblin:Kimia:ECCV12,Fabbri:Giblin:Kimia:PAMI19},
from only two points, later relaxed
for unknown focal length~\cite{Kuang:Astrom:ICCV2013}. The trifocal problem
with three point-tangents as a local version of trifocal pose for global curves was
first formulated by Fabbri~\cite{Fabbri:PHD:2010}, presented here as a
minimal version codenamed Chicago.\\[1em]
\noindent \textbf{Homotopy Continuation.} 
The basic theory of polynomial \hc~\cite{BertiniBook,MorganBook,SWbook} was developed in 1976,
and guarantees algorithms that are \emph{globally} convergent with probability one
from given start solutions.  A number of general-purpose \hc\ softwares have
considerably evolved over the past decade~\cite{BertiniBook,HOM4PS3,NAG4M2,PHCpack}. The computer vision
community has used \hc\ most notably in the nineties for \textsc{3d}
vision of curves and surfaces for tasks such as computing \textsc{3d} line drawings from surface
intersections, finding the stable singularities of a \textsc{3d} line drawing under
projections, computing occluding contours, stable poses, hidden line removal by
continuation from singularitities, aspect graphs, self-calibration, and pose
estimation~\cite{Bruckstein:jvcir94,Faugeras:Luong:Maybank:ECCV1992,Holt:Netravali:IMA:1994,Holt:Netravali:JVCIP:1992,Holt:etal:SPIECurvesandSurfaces:1990,Kriegman:Ponce:CurvesSurfaces1991,Kriegmann:Ponce:SPIECurvesAndSurfaces:1992,Luong:thesis:1992,Maybank:Faugeras:IJCV1992,Petitjean:JMIV1999,Petitjean:Ponce:Kriegman:IJCV92,Pollefeys:VanGool:PAMI1999},
as well as for \textsc{mrf}s~\cite{Bruckstein:jvcir94,Nanda:etal:ISCAS94},
and in more recent work~\cite{Ecker:jepson:CVPR2010,HRadaptive,Salzmann:CVPR2013}.
An implementation of the early continuation solver of Kriegman and
Ponce~\cite{Kriegman:Ponce:CurvesSurfaces1991} by Pollefeys is still widely
available for low degree systems~\cite{Pollefeys:VXLContinuation}.

As an early example, \hc\ was used to find an early
bound of 600 solutions to trifocal pose with 6 lines~\cite{Holt:Netravali:IMA:1994}.
In the vision community \hc\
is mostly used as an offline tool to carry out studies of a problem before crafting
a symbolic solver. Kasten~\etal~\cite{Kasten:Galun:Basri:Arxiv2019} recently compared
a general purpose \hc\ solver~\cite{PHCpack} against their symbolic solver. However, their problem
is one order of magnitude lower degree than the ones presented here, and the
\hc\ technique chosen for our solver~\cite{duff2018solving} is more specific than their use of polyhedral homotopy, in the sense that fewer paths are tracked~(\cf.~the start system hierarchy in~\cite{SWbook}).

\section{Two Trifocal Minimal Problems}\label{sec:trifocal:problems}
We formulate a new minimal problem for points and incident lines in three
views, codenamed \emph{Chicago}. We present its fundamental equations in explicit
parametric form that shed light
on the geometric properties relevant to vision, as well
as a more specific set of equations with 14 unknowns used in our
best-performing solver \minus. 
While we focus on the Chicago problem, our formulations, analysis and solver
framework generalize to important similar problems, and has lead to companion
work by our coauthors~\cite{duff2019plmp}.
To illustrate this, we present a second trifocal problem for
points and a free line, codenamed \emph{Cleveland}.
The formulation, characterization and practical solver approach for Cleveland, in
direct analogy to Chicago, are also a contribution of this paper.
Specific details on Cleveland are left for the appendix, since our focus is on Chicago and the analysis
is analogous. %Moreover, 
%contrasting these problems elucidates the use of incident lines \emph{vs.}\ a free
%line for relative pose.

\subsection{Formulation and Notation}
\noindent We follow notational style from Hartley and Zissermann~\cite{Hartley:Zisserman:multiple:view}
with explicit projective scales. A more elaborate
notation~\cite{Giblin:Motion:Book,Fabbri:Giblin:Kimia:ECCV12} can be used to
express the equations in terms of tangents to curves and derivatives of relevant quantities such as depth.
Fig.~\ref{fig:trifocal:pajdla:notation} illustrates the notation for a single
feature consisting of a point and an incident line in three views.
Symbols may be given two subscripts $p,v=1,2,3$ to index multiple
feature points and views, respectively; indices $p$ may be omitted for brevity.
\begin{figure}
  \centering
  \includegraphics[width=\linewidth]{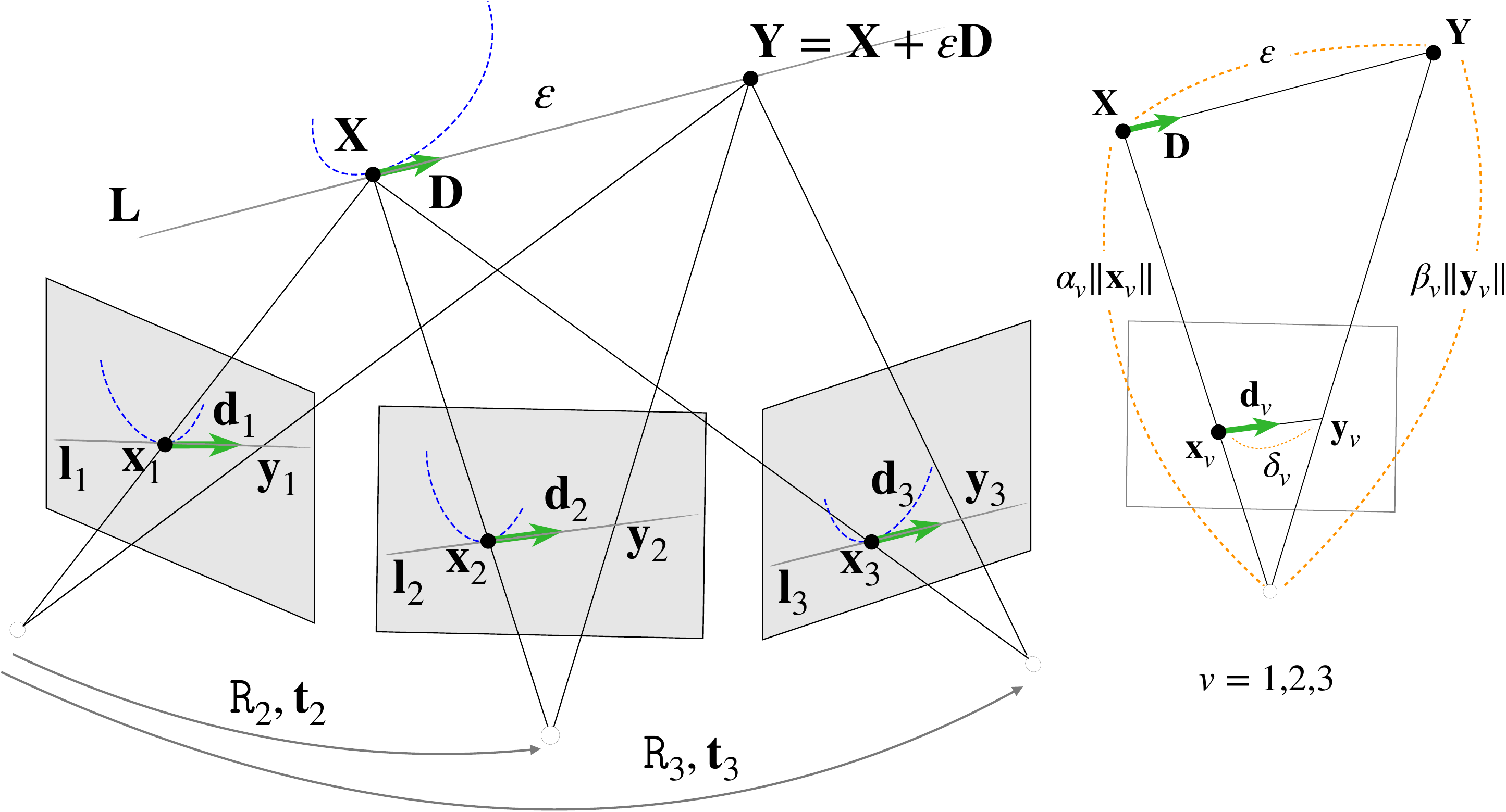}
  \caption{Notation illustrated for a single point with a curve tangent
    vector or feature orientation, \eg, \sift. Multiple 
  features may be explicitly indexed with an additional first subscript.} 
  \label{fig:trifocal:pajdla:notation}
\end{figure} 

Let $\srot_v,\stransl_v$ denote the
rotation and translation transforming coordinates from camera $1$ to camera
$v$ (so that $\srot_1$ is identity and $\stransl_1 = 0$).
Symbols $\X_p$ and $\Y_p$
denote inhomogeneous coordinates of \textsc{3d} points, and $\x_{pv},\y_{pv}$ 
homogeneous coordinates of their respective projections on $\PP^2$ at view $v$, 
with $\alpha_{pv}$,\,$\beta_{pv}$ their respective depths. 
Let $\bl_{pv}$ and
$\mathbf L_{p}$ denote column vectors of homogeneous coordinates of image lines and
underlying \textsc{3d} lines in $(\PP^2)^\vee$ and $(\PP^3)^\vee$, resp.
We use both parametric and homogeneous equations for lines,
the latter obtained by eliminating the line parameter from the former.
Symbol $\dir_{pv}$ represents a line
direction or unit curve tangent vector in homogeneous
coordinates at view $v$ (point at infinity, \ie, third coordinate is zero); and $\D_p$ is the underlying 
\textsc{3d} line direction or space curve tangent in
inhomogeneous world coordinates. Displacements $\varepsilon_p$ 
along $\D_p$ correspond to displacements $\delta_{pv}$ along
$\dir_{pv}$. Let $\boldsymbol \pi_{pv}$ denote the homogeneous coordinates
of the backprojection plane in $(\PP^3)^\vee$ of $\bl_{pv}$.
For simplicity, we use concrete coordinate representations even in
coordinate-independent statements. By default, all coordinates are assumed
real and without the action of intrinsic parameters.

%\todo{perhaps backprojected planes $\pi$ here}

\begin{definition}[Chicago
  trifocal problem]
  Given three corresponding points $\x_{1v}, \x_{2v}, \x_{3v}$ and two lines $\bl_{1v}$, $\bl_{2v}$ in views $v =
  1,2,3$, such that $\bl_{pv}$ meets $\x_{pv}$ for $p = 1,2$ and $v = 1,2,3$,
  compute relative pose $\srot_2,\srot_3,\stransl_2,\stransl_3$.
\end{definition}
\noindent\textbf{Examples of data for Chicago:} 1) Three oriented features (\eg, \sift) corresponding across three
    views, using feature orientations;  2) General curves in three views (\eg, linked subpixel edges), and three
  corresponding curve points (\eg, subpixel edgels), using tangent vectors; 3)
 Trajectories of three moving points observed by three cameras, using
    velocity vectors. While a third orientation triplet is usually available and
    exploited in practice, we show the core pose solution requires only two.

\begin{definition}[Cleveland trifocal problem] 
  Given three points $\x_{1v}, \x_{2v}, \x_{3v}$ in views $v =
  1,2,3$, and given a free line $\bl_{1v}$ in each image, 
  compute $\srot_2,\srot_3,\stransl_2,\stransl_3$.
\end{definition}

\subsection{Essential Equations}\label{sec:basic:eqs}
The essential equations of Chicago (and Cleveland) are obtained by writing constraints per feature
independently, while keeping the pose unknowns in
general form. They are used for analyzing the fundamental properties of the new
problems and as a basis
for further variable elimination and exploring other formulations.
See~\cite{Argawal:Duff:etal:arxiv:jun2022} for a general framework for
navigating different formulations. The final solver that offered the best
performance uses a formulation that further eliminates variables across these
per-feature equations using specific algebraic manipulations connecting features
pairwise, as described further in Section~\ref{sec:minors}. 

\begin{theorem}[Essential trifocal constraints for
points and incident lines, parametric form]\label{th:essential:constraints}
  The constraints on relative pose from points and incident
  lines observed in three views are given by
%\begin{empheq}[left=\empheqlbrace]{align}
  \begin{align}
  \alpha_{v} \x_v &= \srot_v \alpha_{1} \x_1 + \stransl_v,\label{eq:points:simplified:notation}\\
  \eta_v\x_v + \mu_v \dir_v &= \srot_v \left( \eta_1\x_1 + \mu_1 \dir_1
  \right),\label{eq:tangent:final:v3:parametric}
  \end{align}
  for $v = 2,3$ (point indices omitted, $\srot_1 = \mathtt I$ and $\stransl_1 =
  0$). We call~\eqref{eq:points:simplified:notation} the parametric essential
  trifocal point constraints, and~\eqref{eq:tangent:final:v3:parametric} the
  parametric essential trifocal incident line constraint.
  Moreover,~\eqref{eq:points:simplified:notation} imposes
  three constraints per triplet point,
  while~\eqref{eq:tangent:final:v3:parametric} imposes one constraint per incident line triplet:
  \begin{enumerate}
\item \emph{Point epipolar constraints:}
  Solving~\eqref{eq:points:simplified:notation} for $v=2$ and $v=3$.
\item \emph{Point relative scale constraint:} Enforcing depth $\alpha_1$ to be equal in~\eqref{eq:points:simplified:notation}
  for $v=2$ and $v=3$.
\item \emph{Incident line constraint:} Jointly expressed by~\eqref{eq:tangent:final:v3:parametric} for $v=2,3$.
\end{enumerate}
\end{theorem}
\begin{proof}
Eliminate $\X$ from the projections of points
%The equations for point correspondences are 
%\begin{empheq}[left=\empheqlbrace]{align}
%  \alpha_1 \x_1 &=  \X \label{eq:3C:point:cam1}\\
%  \alpha_2 \x_2 &=  \srot_2 \X + \stransl_2 \label{eq:3C:point:cam2}\\
%  \alpha_3 \x_3 &=  \srot_3 \X + \stransl_3. \label{eq:3C:point:cam3}
%\end{empheq}
$\alpha_v \x_v =  \srot_v \X + \stransl_v$, $v=1,2,3$
 to get~\eqref{eq:points:simplified:notation}.
Lines in space through $\X$ are modeled here in parametric form by a displacement parameter
$\epsilon$ and points $\Y = \X + \varepsilon \D$, which are projected
as $\beta_v \y_v =  \srot_v \Y + \stransl_v$, $v=1,2,3$. Eliminate
$\stransl_v$ by subtracting the projection equations of $\X$ and $\Y$,
$\beta_v\y_v - \alpha_v \x_v = \varepsilon \srot_v \D$, and eliminate $\varepsilon \D$ using the equation for $v=1$ and $\y_v = \x_v +
\delta_v \dir_v$,
\begin{equation}\label{eq:tangent:intermediate:v3:parametric}
(\beta_v - \alpha_v)\x_v + \beta_v \delta_v \dir_v = \srot_v \left( (\beta_1 - \alpha_1)\x_1 + \beta_1 \delta_1 \dir_1 \right),
\end{equation}%
for $v = 2,3$. We set $\eta_v
\doteq \beta_v - \alpha_v$ and $\mu_v \doteq \beta_v\delta_v$,
yielding~\eqref{eq:tangent:final:v3:parametric}.

It follows that the trifocal essential point constraints in parametric
form~\eqref{eq:points:simplified:notation} are logically equivalent to the
existence of a triangulation $\X$ from views 1 and 2 \emph{equal} that from views 1
and 3. In parametric form, it simply means that these solutions can be linked by the \emph{same}
depth $\alpha_1$. By construction, these imply the existence of a
triangulation from views 2 and 3, also equal to $\X$,
so ~\eqref{eq:tangent:final:v3:parametric} for views 2 and 3 does not provide an additional constraint.\footnote{Conversely, having
  three pairwise epipolar constraints is \emph{not} equivalent to two pairwise epipolar
constraints and a relative scale constraint~\cite{Ponce-IJCV-2016}.}

The trifocal essential incident line constraints in parametric form are
logically equivalent to the existence of a \textsc{3d} line direction $\D$ that, when rooted at $\X$, projects to
direction $\dir_1$ and $\dir_2$, and that $\D$ also projects to $\dir_3$. 
In the point case the equation from views 1 and 2 provides a constraint,
\ie,~\eqref{eq:points:simplified:notation} for $v=2$ does not always have a solution, while
the incident line equation from views 1 and 2 does not provide a
constraint on pose -- there is always a solution $\mu$ and $\eta$
for~\eqref{eq:tangent:final:v3:parametric} for $v=2$ that parametrizes some
consistent $\D$ irrespective of $\srot$ and the data $\x$ and $\dir$. 
Each triplet of oriented point features provides a single orientation constraint
expressed as two coupled
equations~\eqref{eq:tangent:final:v3:parametric} in $\eta$ and
$\mu$ in addition to pose.
\end{proof}

\begin{corollary}
  The correspondence of points across three views constrain relative rotations and
  translations, while the additional correspondence of an incident line 
  constrains only rotation.
\end{corollary}
\begin{proof}
This is a direct consequence of Theorem~\ref{th:essential:constraints}.
\end{proof}
%Indeed, in projective jargon, incident lines reduce to
%directions corresponding to ideal points or points at infinity, which do
%not constrain translations.
% TODO: we might want to mention that tangents are points at infinity,
% and there are works that show that points of infinity constrain only rotation.
% See http://www.ipb.uni-bonn.de/pdfs/Steffen2010Relative.pdf 
% Avoiding the use of a 3D point in a formulation allows more stability
% specially when the 3D points are far away; Far away points constrain only
% rotations anyways.
Having an incident line thus works like an additional point correspondence
-- in a precise sense like a third of a point -- yet constraining only rotations.
This allows us to construct Chicago as an exactly constrained trifocal problem that 
can be applied, \eg, with conventional \sift\ features.
We can get an expression of these constraints free of auxiliary parameters by
further elimination.

The parametric point epipolar constraints of
Theorem~\ref{th:essential:constraints}, \emph{in particular}, state
that $\x_1$, $\x_v$ and the first camera center $\stransl_v$ are coplanar when written in the coordinates
of camera $v$; this is the classical Essential constraint, readily expressed without parameters via a scalar triple
product trilinear in $\stransl_v$ and the points, the standard expression
that is bilinear in image coordinates. Although we arrived at this constraint
explicitly from first principles through but the simplest
logic, it is a general constraint of two-view
geometry with recent results in trifocal geometry~\cite{Trager:Hebert:Ponce:2019}.  
Algebraically, the classical expression for the Essential constraint ammounts to eliminating depths
$\alpha_v$ from~\eqref{eq:points:simplified:notation} while keeping $\srot_v$ and $\stransl_v$.
However, there are successful arguments for eliminating $\srot_v$ and
$\stransl_v$ \emph{first} in camera pose problems, writing the equations in terms
of depths only
$\alpha$~\cite{Fabbri:Giblin:Kimia:PAMI19,Fabbri:Giblin:Kimia:ECCV12} (\eg,
the classical \textsc{p3p} equations). 
%For instance, $\srot_v$ is readily
%eliminated in~\eqref{eq:tangent:final:v3:parametric} by requiring the length of
%both sides to be equal.
Though not performed here, this further motivates
stating the trifocal essential constraints in parametric form. 
Moreover, the parametric form more readily lends itself to modeling general
curves~\cite{Fabbri:Kimia:IJCV2016} for which trifocal geometry plays a pivotal role.
% Not only that, but the parametric form is closer to first principles and more
% suited for the reasoning done here

% Personal note: points and camera center or translation t play a symmetric role
% in the epipolar constraint: the two corresponding points must be coplanar,
% but written in the same camera coordinate, so that the point vectors x_1, x_v
% _and_ translation vector (which is just the pinhole center written in camera v
% coordinates) must be coplanar, an example of the Carlsoon-Weinshall duality in
% MVG~\cite{Trager:etal:TechReport2018,Trager:Hebert:Ponce:2019}

The trifocal relative scale constraint in 
Theorem~\ref{th:essential:constraints} guarantees
that \textsc{3d} rays converge, which may not be the case if we had used three pairwise epipolar
constraints instead; in fact, this scale constraint is a fundamental and
classical condition of photogrammetry, called the \emph{scale-restraint} equations,
see~\cite{Ponce-IJCV-2016} for general results. Such a constraint may be substituted by an
additional epipolar constraint between views 2 and 3, but it turns out that this is only adequate
for oriented points, \ie, \emph{together} with the incident line constraint, which guarantees
a consistent \textsc{3d} incident line. Without this, having three pairwise epipolar constraints is not enough to
guarantee there is a \textsc{3d} point $\X$ that projects to the observed points,
specially near non-generic configurations~\cite{Ponce-IJCV-2016}, namely
1) if the camera centers are far from 
collinear, when the corresponding rays lie in or near the trifocal plane 2) if
the centers are approximately collinear, when the rays lie near any plane
containing the baseline~\cite{Ponce-IJCV-2016}. In this sense, points with incident
lines are \emph{natural} features in trifocal geometry. 

% INTERESTING
% Note that incident lines are
% classically artificially constructed to derive point constraints on
% trifocal geometry,
%
% This is akin to having an arbitrary point chosen on a line in order to define
% its parametric equation.
%

%In addition to the incident line constraints, this reinforces the that
%fact that Theorem~\ref{th:essential:constraints} provides a minimal set of constraints for
%calibrated trifocal geometry from first principles.  While it is well-known that incident lines or
%tangent directions do not provide a constraint on two-view relative
%pose~\cite{Fabbri:Kimia:IJCV2016}, the equations in the explicit form presented
%here emphasizing trifocal relative pose and derived from first principles (point
%projection alone) have not appeared in the literature, to the best of our
%knowledge. The incident line constraint may be expressed as a single equation that more
%directly represents the constraint by eliminating the auxiliary parameters from
%the coupled equations~\eqref{eq:tangent:final:v3:parametric}. An explicit
%expression is left for companion papers.

% COOL
%\indraftnote{
%The trifocal relative scale costraint defines a trinocular
%line~\cite{Ponce-IJCV-2016}.
%}
%
%\todo{put more direct form for constraints here instead of coupled eqs}
%
%Equivalently, it states that the backprojected planes $\pi_1$ through $\bl_1$
%and $\pi_2$ through $\bl_2$ meet at a common line that must equal the
%intersection of $\pi_1$ and $\pi_3$.  \todo{doing: reintepret incident line
%constraints, normal}

\begin{corollary}[Chicago Essential Equations, Parametric
  Form]\label{th:chicago:essential:eqs}
The Chicago problem is equivalent to finding the solutions of
\begin{align}
  \alpha_{pv} \x_{pv} &= \srot_v \alpha_{p1} \x_{p1} + \stransl_v, \ \ p = 1, 2,
  3 \label{eq:chicago:essential:point}\\
  \eta_{pv}\x_{pv} + \mu_{pv} \dir_{pv} &= \srot_v \left( \eta_{p1}\x_{p1} +
    \mu_{p1} \dir_{p1}\label{eq:chicago:essential:tangent}
\right), \ \ p = 1, 2,
\end{align}
for $v=2,3$, 
which are 30 scalar equations in the relative camera pose $R_2$, $t_2$,
$R_3$, $t_3$, along with 9 unknown depths ($\alpha_pv$) and 12 unknown line
parameters (6 each for $\eta_pv$ and $\mu_pv$).
\end{corollary}
\begin{proof}
  Theorem~\ref{th:essential:constraints} lists all the
  available constraints that arise.
\end{proof}

That actual equations used in our solver amount to an elimination of
the auxiliary parameters in~\eqref{eq:chicago:essential:point}
and~\eqref{eq:chicago:essential:tangent}, leading to vanishing minors, Section~\ref{sec:minors}.
Note that~\eqref{eq:chicago:essential:point} are homogeneous in $\alpha$ and
$\stransl$, so that a multiple of a particular solution are also solutions, \ie,
translations and depths are constrained up to scale, giving 11 constrainable
degrees of freedom. By
Theorem~\ref{th:essential:constraints}, the essential
equations used in Chicago express 3 independent constraints per point, and 1 per
incident line, yielding 11 constraints on 11 degrees of freedom.  Rigorous
computational arguments in Section~\ref{sec:analysis} confirm that these
constraints are also independent across points. In other words, Chicago is a
minimal problem.

One can also see the parametric trifocal essential equations for Chicago as a square system of 30 scalar
equations in the $30$-dimensional space 
$SO(3)  \times SO(3) \times \mathbb{P}^{14} \times \mathbb{P}^5 \times
\mathbb{P}^5$ of unknowns
\begin{align*}
(\srot_2,  \srot_3, [\stransl_2, \stransl_3, \alpha_{11},\hdots, \alpha_{33}], &[\eta_{11},\mu_{11},\eta_{12},\mu_{12},\eta_{13},\mu_{13}], \\ &[\eta_{21},\mu_{21},\eta_{22},\mu_{22},\eta_{23},\mu_{23}]).
\end{align*}
We model the $9$ depths $\alpha_v$ and $\stransl_2,
\stransl_3$ as a single point in $\PP^{14}$, since they are unknown up to a
common scale. Similarly, since only the directions of tangents matter, we regard
these solution components as points in two $\mathbb{P}^5$ factors, one per oriented feature.

There are many ways to proceed with elimination from the essential parametric equations
to obtain alternate formulations, as discussed above. A particular eliminated
formulation based on vanishing minors, which produced the first working solver
for Chicago, and which are used in \minus, is described in Section~\ref{sec:minors}.

%\todo{Ideally, for a complete theory in this paper we should arrive at a 11x11 formulation per the Introduction,
%  where each equation strictly represents a constraint on the degrees of freedom, and not auxiliary
%variables. The problem is that rotation is in a manifold, so that depending on
%the patch on the manifold a different set of coordinates, and hence equations,
%will have to be expressed. Thus, we can only go so far towards a 11x11
%formulation and perhaps Tim's formulation gets there, just not explicitly but
%by blindly selecting minors that make it work.}

\subsection{Equations based on minors used in our solver}\label{sec:minors}

Experiments show that judicious elimination of additional variables from the
basic equations leads to faster and more reliable solvers, with 
tradeoffs, \eg, in the number of variables \emph{vs.}\ nonlinearity and
degeneracy of the resulting representations. This section describes a particular
way to eliminate variables down to a $14\times14$ system that has proven most
successful and general to date.

Futher elimination of certain variables from the basic equations leads to minor-based
constraints, \ie, enforcing the determinants of certain sub-matrices to vanish.
Examples are coplanarity or multilinear constraints, \eg, the essential constraint. In particular, this
eliminates parameters describing coordinates of vectors in
constraints on lines (depths $\alpha$'s) and
planes ($\eta$'s and $\mu$'s). While this approach has long been used
for describing trifocal constraints for
points~\cite{Ponce-IJCV-2016}, in full generality it is novel and has spawned
companion work by our coauthors~\cite{duff2019plmp}.  Additionally,
equations based on minors are multilinear, allowing for
possible numerical improvements, Section~\ref{sec:homotopy-algorithm}.
\begin{figure}
\centering
\includegraphics[width=0.85\linewidth]{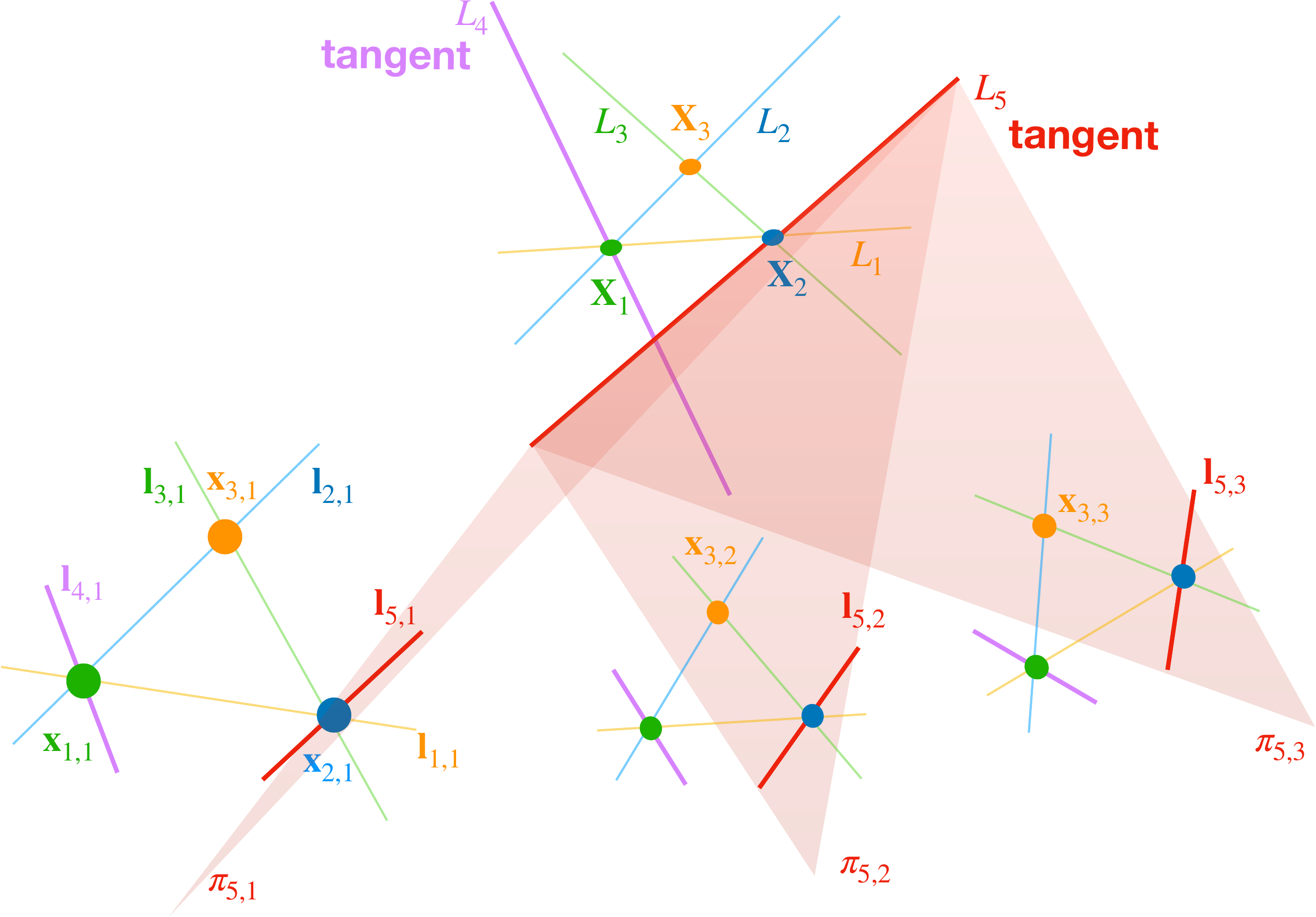}
\caption{Visible line diagrams for Chicago. Cleveland uses the same numbering
for pairwise lines and $\bl_4$ is a free line.}
\label{fig:line-diagrams}
\end{figure}

An instance of Chicago may be described by a configuration of $5$ \emph{visible
lines} in each view, Fig.~\ref{fig:line-diagrams}.
We denote each line by $\bl_{1v} , \ldots , \bl_{5v}$ for $v = 1,2,3$, where
the first three $\bl_{1 v}, \bl_{2 v}, \bl_{3 v}$ pass
through all pairs of points in each view, and the last two
$\bl_{4 v}, \bl_{5 v}$ represent the associated point-tangent pairs.
The minor-based equations split into three sets summarized as:
\begin{itemize}[\IEEEsetlabelwidth{Pairwise lines m}]
  \item[Lines correspond:] $\bpi_{i,1},\bpi_{i,2},\bpi_{i,3}$ meet at a
    \textsc{3d} line $\mathbf L_i$.
\item[Pairwise lines meet:]   $\mathbf L_1, \mathbf L_2,\mathbf L_3$ meet
  pairwise in \textsc{3d}.
\item[Incident tangents:]  $\mathbf L_1, \mathbf L_2, \mathbf L_4$ and  $\mathbf
  L_1, \mathbf L_3, \mathbf L_5$ meet at a point.
\end{itemize}
The latter two are so-called \emph{common point constraints}.\\[1em]
\noindent\textbf{\emph{Line correspondence} constraint.}
These equations express that there must be an underlying \textsc{3d} line $\mathbf L_j$,
$j=1,\dots,5$ associated to the set 
of
backprojection planes $\boldsymbol \pi_{j,v} = [\srot_v | \stransl_v]^\top
\bl_{j,v}$, $v=1,2,3$, which are gathered into a $4\times 3$ matrix
$\LL_j \doteq \left[\boldsymbol\pi_{j,1} \ \ \ \boldsymbol \pi_{j,2} \ \ \ \boldsymbol \pi_{j,3}\right]$.
These planes define a single line if the underlying system of equations has a
1\textsc{d}
solution, leading to the \emph{rank constraint}
\usetagform{emph}
\begin{equation}\label{eq:rank:line-correspondence}
\rank \LL_j \le 2 , \quad j=1 , \ldots , 5.
\end{equation}
\usetagform{default}%
Equivalently, we obtain a polynomial system by setting all $3\times 3$ minors of
each $\LL_j$ to zero.
As explained Section~\ref{sec:solver}, \minus\ employs a
heuristic to select one such minor per
$\LL_j$, fixed for given \hc\ starting 
solutions, yielding \emph{5 final equations} for this constraint.\\[1em]
%\todo{Which one exactly the current code picks?}.\\[1em]
% On the choice of minors by Tim: Thu21jul22
%
%At the time, rows from the 290 x 14 Jacobian were chosen greedily -- starting with no rows from the top, add a row, compute SVD, and keep it if the matrix stays full rank. Note that this depends on the ordering of equations, starting problem-solution pair, etc. So it makes sense to put "simpler equations" first. Putting the line correspondence equations first, each line correspondence gave 2 equations for a total of 10 using this heuristic. Then one equation was chosen from the common point equations, plus the three chart equations necessary to kill the scale in t and both sets of quaternions. 
%
%Since we first worked on this, I've realized that there are better greedy approaches than the one we used, eg. what I call "squaring down" on page 26 here https://smartech.gatech.edu/bitstream/handle/1853/65027/DUFF-DISSERTATION-2021.pdf?sequence=1&isAllowed=y. This is meant as a bit of a joke, since "squaring up" usually refers to multiplying an N x n system with N > n by a random n x N matrix (yet another strategy one could use to get 14 equations here.)
%
%Note that we could have expressed this
%constraint explicitly as coplanarity of plane normals to guide the
%choice of equations, at the cost of losing some generality of the minors technique for
%automatically generating equations for a series of similar
%problems~\cite{Duff:Kohn:Leykin:Pajdla:PLMP:2019}.\\[1em]
%
\noindent\textbf{\emph{Pairwise line intersection} constraint.}
That $\mathbf L_1, \mathbf L_2, \mathbf L_3$ intersect pairwise can be expressed by
\begin{equation}\label{eq:rank:pairwise}
  \rank\,[\mathtt L_i \ \ \mathtt L_j] \leq 3, \qquad i < j \in \{1,2,3\},
\end{equation}
or that all maximal  $4\times 4$ minors vanish.
We use only $\rank\, [\mathtt L_2 \ \mathtt L_3] \leq 3$ corresponding to
$\X_3$, as the
other pairwise intersections will be implicit in the constraint of
\emph{incident tangents}. For \minus\ we pick only one minor equation for this
constraint
%\begin{equation}
%  |\bpi_{1,1}\ \ \ \bpi_{1,3}\ \ \   \bpi_{2,3}\ \ \   \bpi_{3,1}| = 0,
%\end{equation}
% XXX TODO: check this from source code
using the aforementioned heuristics.\\[1em]
\noindent\textbf{\emph{Incident tangents} constraint.}
That tangents intersect at the same point with two other
lines can be expressed by forming matrices
$ \mathtt X_1 \doteq [\mathtt L_1 \ \mathtt L_2\ \mathtt L_4],\ \mathtt X_2
\doteq [\mathtt L_1 \ \mathtt L_3\ \mathtt L_5]$, and requring
\begin{align}\label{eq:rank:tangents}
  \rank\, \mathtt X_j &\leq 3,\ j = 1,3.
\end{align}
All $4\times 4$ minors must vanish, 5 of which are used in \minus.
\emph{The final number of equations consists of 11} fixed, specific vanishing minors.
The total number of minors associated with the rank
constraints~\eqref{eq:rank:line-correspondence},\eqref{eq:rank:pairwise},\eqref{eq:rank:tangents}
far exceeds the number of unknowns used in our formulation of Chicago.  The
number of unknowns, as described in the next section, is $14$, and the total
number of equations implied by these rank constraints is $287 = 5 \binom{4}{3}
+ 2 \binom{9}{4} + \binom{6}{4}$. Nevertheless, these $287$ equations together
with $3$ dehomogenization equations~\eqref{eq:charts} will have $312$
solutions for \emph{almost all} line configurations encoding an instance of
Chicago.  In our \hc\ solvers, we work with a
$14\times 14$ subsystem of these equations which determine a full-rank submatrix of the
$290 \times 14$ Jacobian matrix. 
%These are the equations currently used in
%\minus, listed explicitly in supplemental material.

In this approach, the selection of the actual equations out of a large pool of
possibilities is done through computer-assisted
heuristics, Section~\ref{sec:solver}. 
While these general tools aid in understanding the underlying geometry, this
becomes concealed. Selecting the appropriate subset of minors, \eg,
that ensures the \textsc{3d} rays for matching points always intersect, is a
known problem in the projective case~\cite{Ponce-IJCV-2016}. In that scenario,
a different subset of minors may be used depending on \emph{a priori}
assumptions on camera configuration (\eg, collinear vs non-collinear camera
centers)~\cite{Trager:Hebert:Ponce:2019}. An explicit set of vanishing minors
for point trifocal geometry and the resulting constraints is studied in a
general setting by Trager~\etal~\cite{Trager:Hebert:Ponce:2019}.
A geometric interpretation is that four minors encode constraints that are
trilinear in image coordinates and express that \textsc{3d} rays must meet at a single
point. When \textsc{3d} rays are viewed from four different appropriate image planes,
each vanishing minor may be expressed as requiring three coplanar
projected lines meeting at a point~\cite{Trager:Hebert:Ponce:2019}. 
We verify experimentally that our chosen set of minors 
provides a working solver.

\section{Problem Analysis}\label{sec:analysis}

\noindent A general camera pose problem is defined by a list of labeled features in each image, which are in correspondence. The image coordinates of each feature are given, and we aim to determine the relative poses of the cameras.  
%The labeling tells which features correspond from one image to another. A feature is not necessarily seen in all cameras. 
The concatenated list of all the feature coordinates from all cameras is a point
in the image space $Y$, while the concatenated list of the features' locations
and orientations in the world frame or camera~1 is a point in the world feature
space $W$. The scale of the relative translations is
indeterminate, so relative translations are treated as in projective space.  For
$N$ cameras, the combined poses of cameras $2,\ldots,N$ relative to camera~1 are
points in $SE(3)^{N-1}$.  Let the pose space be $X$, the projectivized version
of $SE(3)^{N-1}$, and so $\dim X=6N-7$. Given the \textsc{3d} features and the camera
poses, we can compute the image coordinates of the features by considering a
viewing map $V\colon W\times X \to Y$.  A camera pose problem is: given $y\in
Y$, find $(w,x)\in W\times X$ such that $V(w,x)=y$.  The projection
$\pi\colon(w,x)\mapsto x$ is the set of relative poses we seek.

\begin{definition}
  A camera pose problem is \emph{minimal} if $V\colon W \times X \to Y$ is invertible and nonsingular at a generic $y\in Y$.
\end{definition}
\noindent A necessary condition for a map to be invertible and nonsingular is
that the dimensions of its domain and range must be equal.  Let us consider
three kinds of features: a point, a point on a line (equivalently a point with
tangent direction), and a free line (a line with no distinguished point on it).
For each feature, say $F$, let $C_F$ be the number of cameras that see it.  The
contributions to $\dim W$ and $\dim Y$ of each kind of feature are in the table
below, where a point with a tangent counts as one point and one tangent.  Thus,
a point feature has several tangents if several lines intersect at it.
\begin{center}
  \small
\begin{tabular}{ccc}
  \textbf{Feature} & $\dim W$ & $\dim Y$ \\ \hline
  Point, $P$ & 3 & $2\cdot C_P$ \\
  Tangent, $T$ & 2 & $1\cdot C_T$ \\
  Free Line, $L$ & 4 & $2\cdot C_L$ \\
\end{tabular}
\end{center}
Summing all the contributions to $\dim Y - \dim W$, we have
\begin{theorem}
Let $\relu{x}\doteq\max(0,x)$.  A necessary condition for a $N$-camera pose
problem to be minimal is
\[
\sum_P\relu{2C_P-3} + \sum_T\relu{C_T-2} + \sum_L \relu{2C_L-4} = 6N-7.
\]
\end{theorem}
For trifocal problems where all cameras see all features, \textit{i.e.,}
$C_P=C_T=C_L=3$, a pose problem with 3 feature points and 2 tangents meets the condition.  A pose problem with 3 feature points and 1 free line also meets the condition. Adding any new features to these problems will make them overconstrained, having $\dim Y>\dim W\times X$.

\begin{definition}
The \emph{algebraic degree} of a minimal pose problem is 
the number of solutions $(w,x) \in V^{-1} (y)$ for
generic $y\in Y$.
\end{definition}

Both Gr\"{o}bner bases and \hc\ offer
probability-one methods for computing all solutions for a particular problem
instance specified by $y\in Y$.  Gr\"{o}bner bases also offer an exact method,
when working over $\mathbb{Q}$. However, it is
difficult to say when any particular $y\in Y$ will satisfy the necessary
genericity conditions to have have this many solutions without knowing the
algebraic degree \emph{a priori}. Thus, the following statement has
two components: that both problems are minimal (rigorously proven) and that
their algebraic degrees are as stated (true with probability one).

\begin{theorem}[Computational]
The Chicago trifocal problem is minimal with algebraic degree~312, and the Cleveland problem is minimal with algebraic degree~216.
\end{theorem}
\begin{proof}
To show that a $N$-camera pose problem is minimal, it is enough to find
$(w,x)\in W\times X$ where the Jacobian of $V(w,x)$ is full rank.  For exact
values of $(w,x) \in W \times X$ in rational arithmetic, we compute the
exact rank of this Jacobian. This proves that the
problem is minimal. 
To compute the algebraic degree of a given problem, we write down a system of
polynomial equations in unknowns $(w,x) \in W \times X$ for a randomly chosen
$y$.
Since the problem is minimal, we expect that the ideal generated by these polynomials is $0$-dimensional.
Gr\"{o}bner bases give standard methods~\cite{Cox-IVA-2015} both for checking
that this ideal is $0$-dimensional and computing its degree.
Finally, to verify that the degree of the ideal is equal to degree of the
minimal problem, we have computed all solutions to the system of polynomials
specified by $y\in Y$ and verified that they correspond to valid points $(w,x) \in
W \times X$. We carried out this procedure with the minors equations
and confirmed the degree using the essential equations and \hc.
\end{proof}

\textbf{Remark:} The previous argument depends on the system of equations chosen to model the problem.
For instance, if~\eqref{eq:chicago:essential:point},\eqref{eq:chicago:essential:tangent} are used, then there exist $312$ solutions corresponding to valid points in $W \times X$, plus a small number of degenerate solutions where certain values of the depths $\alpha$ equal zero.
Additional polynomial equations which exclude these solutions may be generated using the symbolic technique of~\emph{saturation}~\cite[Sec~4.4]{Cox-IVA-2015}.
Such a saturation step is also necessary if rotation matrices are modeled with
the quaternion parametrization in~\eqref{eq:cayley}, since we must rule out
degenerate solutions with $w_i^2 + x_i^2 + y_i^2 + z_i^2=0$.

A companion work by our coauthors~\cite{Duff:Kohn:Leykin:Pajdla:PLMP:2019}
provides \textsc{m}acaulay2 tutorial for the Gr\"{o}ner basis degree proof
and other general techniques presented in this section for analyzing Chicago, Cleveland,
and a number of related minimal problems using the minors approach.
Since Gr\"{o}bner bases can be used to compute the algebraic degrees of both minimal problems, it is natural to hope that they also can be used to design effective minimal solvers.
However, the current leading methods for building minimal solvers (eg.~\cite{kukelova2008automatic,Larsson-Syzygy-CVPR-2017,Larsson-CVPR-2018}) do not scale well for problems of degree $100$ or larger.
This is our main motivation for using optimized \hc.

% xxx journal TODO: include stuff from the rebuttal discussion here, such as digressions on
% why the system is stable, and what exactly was run
%\begin{draf}
%- Complexity?
%- check poster and figs folder for anything useful
\section{Optimized Homotopy Continuation Solver}\label{sec:solver}
\noindent 

Like other minimal problems in vision, the Cleveland and Chicago problems require us to solve a system of polynomial equations.
Crucially, these equations are polynomial in both the input data (points and lines in images) and the unknown quantities to be estimated (cameras and world features.)
It is common to call these systems \emph{parametrized polynomial systems}, as the input data parametrize the space of all instances of a given problem.
In Section~\ref{sec:homotopy-algorithm}, we review basic facts about 
\emph{coefficient parameter homotopy}, a very general framework for solving
parametrized polynomial systems based on \hc\ methods.
The \emph{parameter homotopies} arising in this framework lie at the core of our
\hc\ solvers.
To make this general framework concrete, Section~\ref{sec:minors} describes in
precise detail one possible strategy for formulating the Cleveland and Chicago
problems, in which the depths and displacements are eliminated from the
essential equations of Section~\ref{sec:basic:eqs}.
Although these formulations are used in our best-performing solvers to date, we stress that the exact formulation is not essential to the underlying technique.
Other formulations of the problem will also give rise to parameter homotopies
which can be successfully used within general-purpose software~\cite{BertiniBook,NAG4M2} or
within our optimized \textsc{c++} framework \minus\ described in Section~\ref{sec:cpp-implementation}. 

Acknowledging the promise of further speedups brought by experimenting with
different formulations, we observe that our specific parameter homotopies can
already be used to solve Chicago and Cleveland in a relatively efficient manner,
Section~\ref{sec:experiments}.
We attribute relatively good
run times to two factors. First, the inherent specificity of
parameter homotopies when compared to other \hc\ methods; the number of paths to
track in a parameter homotopy is precisely the algebraic degree of the problem.
Second, we optimize various aspects of \hc, such
as polynomial evaluation and numerical linear algebra, 
Section~\ref{sec:cpp-implementation}, along with more aggressive optimization
opportunities and tradeoffs.

\subsection{Algorithm}
\label{sec:homotopy-algorithm}
\noindent 
%From the previous section, we may define a specific system of polynomials
%$F(\mathcal{R};\mathcal{A})$ in the unknowns $\mathcal{R} = (\srot_2, \srot_3,
%\stransl_2, \stransl_3)$ parametrized by $\mathcal{A} = (\bl_{11},\ldots )$. Many representations for rotations were explored, but our main implementation employs quaternions. 
%A fundamental technique for solving such systems, fully described in~\cite{SWbook}, is \emph{coefficient-parameter homotopy.}   \st{Algorithm 1 summarizes homotopy continuation from a known set of solutions for given parameter values to compute a set of solutions for the desired parameter values. It assumes that solutions for some starting parameters $\mathcal{A}_0$ have already been computed via some offline, \emph{ab initio} phase. For our problems of interest, the number of start solutions is precisely the algebraic degree of the problem.} 
We assume that
$F(\mathcal{R};\mathcal{A})$ is a system which is polynomial in both the
variables $\mathcal{R}$ and the parameters $\mathcal{A}$. One is
interested in efficiently computing the solutions for many instances of the
parameters. To compute all nonsingular complex isolated solutions of
$F(\mathcal{R};\mathcal{A}) = 0$ for any given set of target parameters
$\mathcal{A}^*$, one may use the parameter homotopy 
\begin{equation}\label{eq:parameter-homotopy}
H(\mathcal{R};s) = F(\mathcal{R}; (1-s)\mathcal{A}_0 + s\mathcal{A}^*),
\end{equation} 
for $s \in [0,1)$, Algorithm~\ref{alg:homotopy}.  It is assumed that solutions for some starting
parameters $\mathcal{A}_0$ have already been computed via some offline, \emph{ab
initio} phase, described below, by default hardcoded in \textsc{minus}.  This initial
phase determines representatives of nonsingular isolated solutions, making for faster, more
efficient solves for any other parameter values desired, \eg, within \ransac.

Generically, the homotopy paths are smooth and do not intersect each other.
To ensure this (genericity) condition
for every homotopy path with probability 1,
we may employ the so-called \emph{gamma trick}. 
This consists in choosing a (random) $\gamma \in \mathbb{C}$ so that the homotopy equation becomes
\begin{equation*}
H(\mathcal{R};s) = F(\mathcal{R}; \phi(s)),
\end{equation*}
where $\phi (s)$ parametrizes an arc, depending on $\gamma$, connecting $\mathcal{A}_0$ to $\mathcal{A}^*$ in the parameter space.
More explicitly, we define $\phi(s) = (1-\tau(s))\mathcal{A}_0 + \tau(s)\mathcal{A}^*$, with $\tau(s) = \frac{\gamma s}{1+(\gamma-1)s}$, as in Algorithm~\ref{alg:homotopy}.
In this way, $\phi(s)$ is a generic path in the  complex space 
without singularities, even if the endpoints are real.  
However, even though the circular arc depending on $\gamma$
misses the non-generic points in $\mathbb{C}$ with probability 1, it might
happen that the arc is close to these
non-generic points; this can cause instability, increase the error or decrease
speed in computations.
If we run \minus\ multiple
times with the same data but using different (random) $\gamma$'s, it results in
a dispersion of run times and even occasional failures. 
The slower running times and the occasional failures happen when $\gamma$
lands close to certain rays in $\mathbb C$ which intersect an
appropriately-defined discriminant in the tracking parameter $s$.

For systems which are linear in the parameters $\mathcal{A}$, it is possible to adapt the gamma trick to work with a simpler \emph{linear segment homotopy}, due to the following calculation:
\begin{align}
    H(\mathcal{R};s) &= F(\mathcal{R}; (1-\tau(s))\mathcal{A}_0 +
    \tau(s)\mathcal{A}^*) \notag \\
    &= (1-\tau(s))F(\mathcal{R};\mathcal{A}_0) +
    \tau(s)F(\mathcal{R};\mathcal{A}^*) \label{eq:linear:segment:h}\\
    &= \frac{1}{1+(\gamma-1)s}\left[(1- s)F(\mathcal{R};\mathcal{A}_0) + \gamma
    sF(\mathcal{R};\mathcal{A}^*)\right], \notag
\end{align}
where the coefficient $\frac{1}{1+(\gamma-1)s}$ is never zero for real $s\in [0,1)$ and can thus be ignored when solving $H(\mathcal{R};s) = 0$.
This variant of the gamma trick may be preferable to the general version, since it results in cheaper evaluation of the homotopy and its derivatives, and may also lead to better numerical stability. 

Minor-based constraints are \emph{multilinear} in the coordinates of each
line $\bl$ suggesting that a simple variant of the aforementioned ``linear"
gamma trick will work for related formulations.
This will indeed work for Cleveland, where we may treat each coordinate of each line as an independent parameter.
However, for Chicago there is an additional subtlety due to the fact that the
associated configuration of lines is not general and must satisfy
\[
\rank \begin{bmatrix}
\bl_{1 v} & \bl_{2 v} & \bl_{4 v}
\end{bmatrix} \le 2, \quad
\rank \begin{bmatrix}
\bl_{1 v} & \bl_{3 v} & \bl_{5 v}
\end{bmatrix} \le 2.
\]
For Chicago, treating each coordinate of each line as an independent parameter will not give a valid parameter homotopy; even if $\mathcal{A}$ and $\mathcal{A}^*$ encode valid configurations of lines, points on a circular arc or linear segment connecting them will not.
We thus represent the lines encoding tangents
\[
\bl_{4 v} = a_{1 v} \bl_{1 v} + a_{2 v} \bl_{2 v}, \quad
\bl_{5 v} = b_{1 v} \bl_{1 v} + b_{2 v} \bl_{3 v},
\]
with $2$ independent parameters as a pencil of lines.

A full accounting of the variables and parameters used for
Chicago in \minus\ is as follows:\\[1em]
\noindent{\bf $14$ variables:} Each translation vector
has three unknown components, and the entries of matrices $\srot_2$ and $\srot_3$
are written as rational homogeneous functions in four unknowns (homogenized
Cayley):
\begin{equation}\label{eq:cayley}
\srot_v = 
\begin{bmatrix}
w_v & -z_v & y_v\\
z_v & w_v & -x_v\\
-y_v & x_v & w_v
\end{bmatrix}
\begin{bmatrix}
w_v & z_v & -y_v\\
-z_v & w_v & x_v\\
y_v & -x_v & w_v
\end{bmatrix}^{-1}.
\end{equation}\\
\noindent{\textbf{$56$ parameters:}}
$27=3\times 3 \times 3$ parameters represent three
      independent lines  $\bl_{1 v}, \bl_{2 v}, \bl_{3 v}$ in each view;
$12 = 3\times 2 \times 2$ parameters of the form $a_{i v}$ $b_{i v}$ 
represent two dependendent lines $\bl_{4 v}, \bl_{5 v}$ in each view;
      The remaining $17=56 - 39$ parameters consist of $\mathbf v_1,\mathbf v_2
      \in \CC^5$ and $\mathbf v_3 \in \CC^7$ which are random coefficients of 3
      inhomogeneous linear equations 
      \begin{align}
      (\mathbf r_1\ 1)\ \mathbf v_1  = 0, \ \ \  \ 
      (\mathbf r_2\  1)\ \mathbf v_2  = 0, \ \ \ \  \label{eq:charts}
      (\stransl_2^\top \, \stransl_3^\top \ 1)\ \mathbf v_3 = 0
      \end{align}
that determine affine charts on homogeneous coordinates given by $\mathbf r_1 = (w_2,x_2,y_2,z_2)$, 
$\mathbf r_2 = (w_3,x_3,y_3,z_3)$, and $(\stransl_2^\top \ \stransl_3^\top)$.
% XXX todo: These random v's are sampled from S^n not C^n in MINUS

In summary, Chicago may be formulated as a system of $290$ equations in $14$ variables and $56$ parameters.
A similar accounting lets us formulate Cleveland as a system of $64$ equations in $14$ variables and $53$ parameters.
As previously remarked, we may select a square subsystem $F$ to define the homotopy in~\eqref{eq:parameter-homotopy}, provided that the Jacobian $\frac{d \, F}{d \, \mathcal{R}} (\mathcal{R}_0 ; \mathcal{A}_0)$ has full rank for every starting solution $\mathcal{R}_0$.
We note that the $276$ excess equations need not be algebraic consequences of the $14$ that are selected.
Nevertheless, the fact that each initial solution $\mathcal{R}_0$ satisfies all
$290$ equations implies that we do not need to enforce these excess equations
explicitly -- see, \eg, the discussion in~\cite[SM Section 16]{hruby2021learning}, or the discussion of ``side conditions" in~\cite[Section 7.4]{SWbook}.

For Chicago, a precomputed set of $312$ starting solutions to the $290 \times
14$ system for starting parameters $\mathcal{A}_0$ may be numerically continued
to $312$ solutions for target parameters $\mathcal{A}$ via the parameter
homotopy~\eqref{eq:parameter-homotopy}, where $F$ is a suitable $14 \times 14$ square subsystem.
To obtain the starting solutions, we first compute a single, random
problem-solution pair $(\mathcal{R}_0, \mathcal{A}_0)$, first computing
$\mathcal{R}_0$ by fabricating a random scene and cameras, then 
$\mathcal{A}_0$ by projecting features in each image.
From this initial problem-solution pair, we may then generate a complete set of $312$ solutions by parameter continuation along random monodromy loops in the space of parameters.
Such monodromy-based heuristics are standard in numerical algebraic geometry.
A complete description is beyond the scope of this paper, see
\eg,~\cite{Bates:Hauenstein:Sommese:Wampler:SIAM2008} or~\cite{Duff:Hill:Jensen:Lee:Leykin:Sommars:Monodromy:IMA2018}, where the latter work describes the implementation we used.

For the minors-based formulation of Chicago, an \emph{ad-hoc}
variant of the gamma trick may be be used with the linear segment
homotopy~\eqref{eq:linear:segment:h}.
The variant is used in the implementation of \minus, and is based on the following idea: pick $\gamma_1, \gamma_2, \ldots , $ at random from the complex unit circle, and consider the parameter values $\mathcal{A}^{\gamma_1, \gamma_2, \ldots }$ obtained by the following replacements
\begin{equation}
  \begin{split}
\bl_{1 v} &\to \gamma_1 \, \bl_{1 v}\\
\bl_{2 v} &\to \gamma_2 \, \bl_{2 v}
  \end{split}
  \ \ \ \ \ \ \ \ \ \ 
  \begin{split}
a_{1 v} &\to \overline{\gamma_1} \, a_{1 v}\\
a_{2 v} &\to \overline{\gamma_2} \, a_{2 v}\\
  \end{split}
    \ \ \ \ \ \ \ldots
\end{equation}
These replacements are designed so that systems parametrized by $\mathcal{A}$ and $\mathcal{A}^{\gamma_1, \gamma_2, \ldots }$ have the same solution sets.
Thus, for generic starting and target parameters $\mathcal{A}_0$ and
$\mathcal{A}^*$, real or complex, we may numerically continue the solutions of
$F(\mathcal{R}; \mathcal{A}_0) = 0$ to those of $F(\mathcal{R} ; \mathcal{A}^*)
= 0$ using the linear segment connecting $\mathcal{A}_0^{\gamma_1, \gamma_2, \ldots }$ and $(\mathcal{A}^*)^{\gamma_1, \gamma_2, \ldots }$ in the space of parameters.

We conclude this section with Algorithm~\ref{alg:homotopy}, which contains a
high-level description of our \hc\ solver in pseudocode.

\begin{algorithm}[h]
  \small
\label{alg:homotopy}
	\SetKwInput{Input}{input}\SetKwInOut{Output}{output}
	\SetKwFunction{Union}{Union}\SetKwFunction{FindCompress}{FindCompress}
	\DontPrintSemicolon
	
	\Input{Square polynomial system $F(\mathcal{R};\mathcal{A})$, where 
    $\mathcal{R} = (\srot_2, \srot_3, \stransl_2, \stransl_3)$,
	and $\mathcal{A}$ parametrizes the data; Start parameters $\mathcal{A}_0$; 
	start solutions $\mathcal{R}_0$ where $F(\mathcal{R}_0;\mathcal{A}_0)=0$; Target parameters $\mathcal{A}^*$; Random $\gamma\in\mathbb{C}$}
  \Output{Set of target solutions $\mathcal{R}^*$ where $F(\mathcal{R}^*;\mathcal{A}^*)=0$}
	%\hrulefill\\
	%\linegoal
	%\algrule
	\vspace{1em}
  Setup homotopy $H(\mathcal{R};s) = F(\mathcal{R}; (1-s) \mathcal{A}_0^{\gamma_1, \gamma_2, \ldots }+ s (\mathcal{A}^*)^{\gamma_1, \gamma_2, \ldots })$.
\vspace{1em}\\
  \For{each start solution}{
  $s \longleftarrow 0$\;
	  \While{$s < 1$}{
	  Select step size $\Delta s \in (0,1-s]$. \\
      \textbf{Predict:} Runge-Kutta Step from $s$ to $s+\Delta s$ such that
      $dH/ds = 0$.\\
      \textbf{Correct:} Newton step st.\ $H(\mathcal{R};s+\Delta s) = 0$.\\
      $s \longleftarrow s + \Delta s$
	  }
	}
	\Return Computed solutions $\mathcal{R}^*$ where $H(\mathcal{R}^*,1) = 0$.
	\caption{Homotopy continuation solution tracker}
	\label{alg}
\end{algorithm}

\subsection{Implementation}
\label{sec:cpp-implementation}
\noindent We devised an optimized open source package called MINUS -- MInimal problem
NUmerical Solver, available at \url{github.com/rfabbri/minus}. %~.
This is an \hc\ framework specialized for minimal problems, templated in
\textsc{c++} enabling efficient specialization for different problems,
formulations, and precisions. The most reliable and high-quality solver to date
uses a $14\times14$ minors formulation in double precision (64-bit).
The most important optimization is exploiting fixed-length \textsc{c}-style arrays to
optimize memory layout for size and locality. 
We also hardcoded evaluators and used Eigen~\cite{eigenweb}'s LU decomposition
with partial pivoting for linear algebra solves, which proved accurate as
long as double precision is used. The most important compile flag is
\texttt{--ffast-math}; despite aggressive
floating point optimizations, this only affected output within $10^{-10}$ error.

%The \minus\ source code evolved out of a series of optimizations on top of
%generic \hc\ \textsc{c++} code from \macaulay's \textsc{nag} package~\cite{leykin:NAG}.
%The basic C code for the evaluators were obtained using
%from a symbolic representation of the homotopy and its partial derivatives in
%\macaulay, hand-optimizing
%for memory access and removing parameter aliasing. 
% Macaulay2 in smallcaps

%TODO: fail -> succes rate throughout
As shown in Section~\ref{sec:experiments}, \minus\ runs on
average at hundreds of miliseconds and up to $100\times$ faster than
general-purpose \hc. It can run at a few miliseconds at the cost
of reduced success rate in finding the solution, due to more aggressive optimization parameters. 
Such reduced success rate might be mitigated within \ransac, if adequately
assessed. For instance, we successfully devised a ``lossy'' \hc\ parameter to constrain the number of
predictor iterations per solution path, which have
yielded an effective speedup at negligible loss in success rates, Section~\ref{sec:experiments}. 

%This is
%because paths leading to actual solutions tend to be more limited. The safe
%default is $500$ iterations per path.

The second most important algorithm parameter to vary is the maximum number of
correction steps; 4 is the current safe default. Increasing it to
5--7 cuts the runtime down to $\SI{280}{\ms}$.  Another is corrector
tolerance, which affects how many correction iterations are performed:
increasing it $10^4\times$ brings the runtime down to less than $\SI{200}{\ms}$.  The
error rate for these extreme cases can be as high as 50\%, although testing
reprojection error to larger practical levels of $\SI{1}{\pixel}$ precision may bring this
figure up.
%More experiments are required in order to assess the tradeoffs of
%aggressive settings.

%H Future routes for optimizations ----------------------------------------------
%\noindent{\bf Possible optimizations.}
%The conditioning of the linearized homotopies (Jacobian matrices) as
%one varies the formulation should be assessed to search for runtime and stability.

%% 
%% Success rate can also can be increased by performing tests on the input points
%% and lines to pre-filter near-collinear configurations, which make the system
%% close to underconstrained; these may be blindly designed by machine-learning a
%% configuration-speed hash as in~\cite{pollefeys:configuration} and skipping 
%% configurations leading to slow speeds.

%A classic optimized numerical code that bears striking similarities with \minus,
%in contrast with symbolic algorithms, is \sqrtalg, square root (and inverse
%square root). %In fact, \minus\ can be thought of computing a $\sqrt[312]{\ \cdot\ }$ for Chicago.
Like \minus, widespread fast numerical algorithms to compute simple functions such as \sqrtalg\ solve
polynomial equations iteratively, and the key lies in the
starting point~\cite{Parilla:etal:sqrt:IoT:2022}.
% (which can be seen as a crude and fast approximation to the solution)
%The actual start solution that requires the least number of iterations for any target
%input evolved from decades of experimentation by the real-time 3D graphics community and
%became known as ''the magic number''.% The \hc\ technology in
%\minus\ benefits from similar fine-grain tuning of start solutions.
The start system in \minus\ is by default precomputed from random parameters; it could instead be sampled from
our synthetic data, and the closest camera could be selected
matching a similar configuration of correspondences.
See also companion work by our
coauthors~\cite{Hruby:Duff:Leykin:Pajdla:CVPR2022}.
Varying the problem formulations also has potential for speedup.
Further eliminating variables to, say $6\times 6$, could
bring improvements since linear solves could be explicitly inverted. 
A GPU implementation is explored in companion work by our coauthors~\cite{Chien:Fan:Tsigaridas:Kimia:etal:CVPR2022}.
%Further opportunities for
%speedups are discussed in Section~\ref{sec:experiments}.

% DONE UNTIL HERE -------------------------------------------------------

% Originated from Macaulay2 NAG 
%
%
% Other compile flags
%
% --ariel?
% fraction of compute effort in linear solves vs evaluators

% Success x  Speed Tradeoff

% TO Experiments

%% Future; incremental numerical  methods

% TODO perhaps a simple plot of when it is worth reducing a nonlinear system of
% equations to a number n x linear solves, to see how many iterations we can
% afford for our system to get it running very fast. This would give an idea of
% how much to run.

% Evaluator

\section{Experiments} \label{sec:experiments}
\noindent Experiments are conducted first for synthetic data for a 
controlled study, followed by challenging real data. We present results for the
more challenging Chicago problem, since the exact same core solver is used for
Cleveland.

\noindent \textbf{Synthetic data experiments:} The synthetic data from~\cite{Fabbri:Kimia:IJCV2016,Fabbri:Giblin:Kimia:ECCV12} consists of
3D curves in a $4\times 4\times \SI{4}{\cubic\cm}$ volume projected to 100 cameras
%and sampled to $500\times 600$ at sub-pixel precision with an average of one
%sample per pixel from analytical functions,
(Fig.~\ref{fig:synth:data:sample}), and sampled to get 5117 points endowed with orientations (tangents of curves) that are projections of the same 3D analytic points and
tangents, and then degraded with noise and outliers. Camera centers are randomly
sampled on an average sphere around
the scene along normally distributed \emph{radii} of mean $\SI{1}{\m}$ and $\sigma =
\SI{10}{\mm}$. Rotations are constructed via normally distributed look-at directions with mean
along the sphere radius looking to the object, and $\sigma = \SI{0.01}{\radian}$
such that the scene does not leave the viewport, followed by uniformly
distributed roll. This sampling is filtered such that no two cameras are within
$\ang{15}$ of each other. Each camera encompasses a $500\times \SI{500}{\pixel}$
viewport, where the entire dataset is visible at sub-pixel precision with no more
than one sample per pixel.

Our first experiment studies the numerical stability of the \minus\ solver.
%When veridical and accurate correspondences are available, can the system reliably and accurately determine trifocal camera pose? Since ground-truth correspondence is available, 
% .We randomly generate 1000 triplets of true point and line correspondences (which are constructed from tangents to the curves passing through the points). Then, from each sample, we select three random point correspondences and two random line correspondences through two of the points. \textbf{Tomas: this is not correct! tangents are forced on the point when the point is selected.}
The dataset provides veridical point correspondences, which inherit an orientation
from the tangent to the analytic curve. For each sample set, three triplets of
point correspondences are randomly selected with two endowed with the
orientation of the tangent to the curve. Only real 
solutions that generate positive depth are retained.  The
unused tangent of the third triplet is used to verify the solution as it
provides an unused equation. For each of the remaining solutions only one pose
is determined.

The error in pose estimation is compared with ground-truth as the
angular error between the normalized translation vectors and between the quaternions. The 
process of generating the input to pose computation is repeated $10^3$ times and
averaged. This experiment demonstrates that: ({\em i}) pose estimation errors are
negligible, Fig.~\ref{fig:numerical-stability}(a); ({\em ii}) the number of
actual solutions is small: 35 real solutions on average, pruned down to 7
on average by enforcing positive depth, and even further to about 3-4 physically realizable solutions on average employing the unused tangent of the third
point as verification, Fig.~\ref{fig:numerical-stability}(b); note that these
extra solutions can then be detected by \ransac;
%These can then be handed over to \ransac\ for verification.
({\em iii}) the solver fails in about 1\% of cases, which,
while not a problem for \ransac, can be eliminated by running the solver for that solution path
with higher accuracy or more parameters at a higher computational cost.%
\begin{figure}
   \centering
   \begin{tabular}{cc}
   \includegraphics[width=0.6 \linewidth]{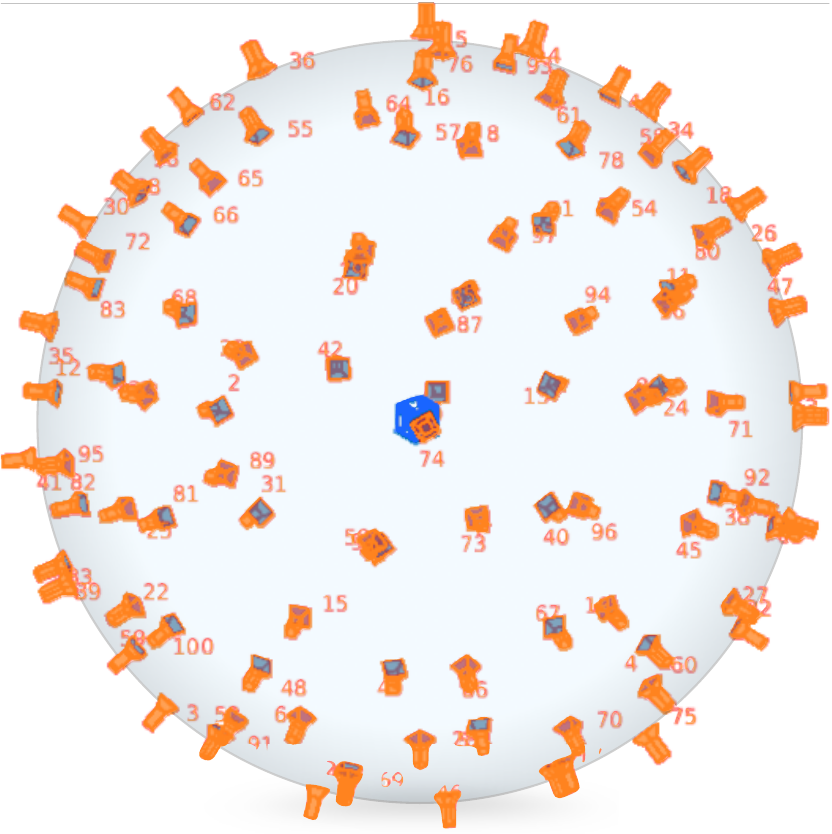}&
   \includegraphics[height=4.8 truecm]{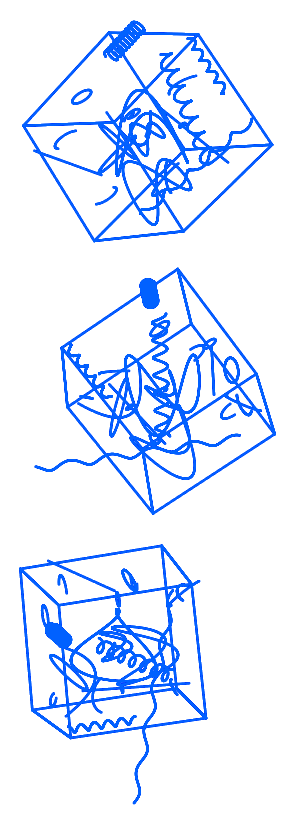}
   \end{tabular}
   \caption{Sample views of our synthetic dataset. Real datasets have also been used in
     our experiments. (3D curves are from~\cite{Fabbri:Giblin:Kimia:ECCV12,Fabbri:Kimia:IJCV2016}).
 }
\label{fig:synth:data:sample}
\end{figure}%
% \paragraph{Numerical stability}
% To check the numerical stability of the minimal solver, we randomly pick 3 pairs of point-tangent correspondences. And use three points and two tangents to build a test case without any noise. We run 1000 random cases to check the numerical stability of the proposed solver. The average number of real solutions is 38. We showed errors of computed parameters with respect to the ground truth in Figure~\ref{fig:numerical-stability}. The errors are computed as follows: for translation we compute the Euclidian distance between ground truth and estimated translations; for rotation we use normalized quaternion to represent rotations so that the angle between quaternions is used as error measurement. Figure~\ref{fig:numerical-stability} shows that the numerical stability of our solver.
\begin{figure}
  \centering
%  \begin{tabular}{cc}
  \includegraphics[width=0.95\linewidth]{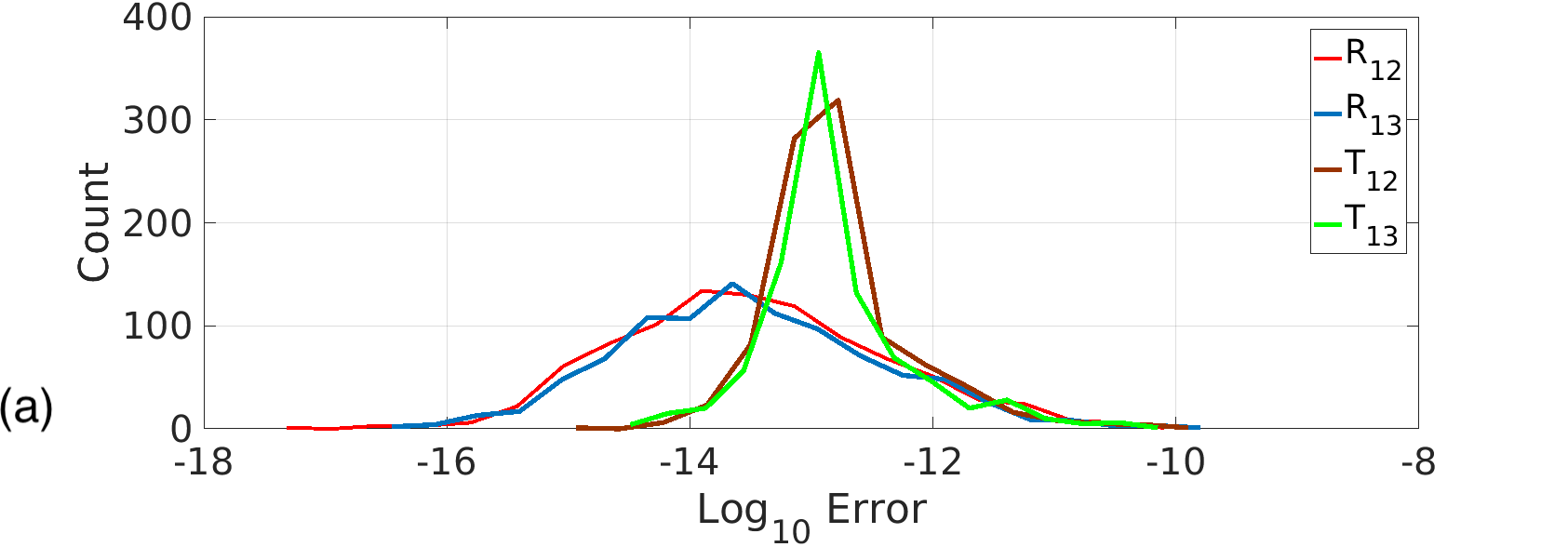}\\
  \includegraphics[width=0.95\linewidth]{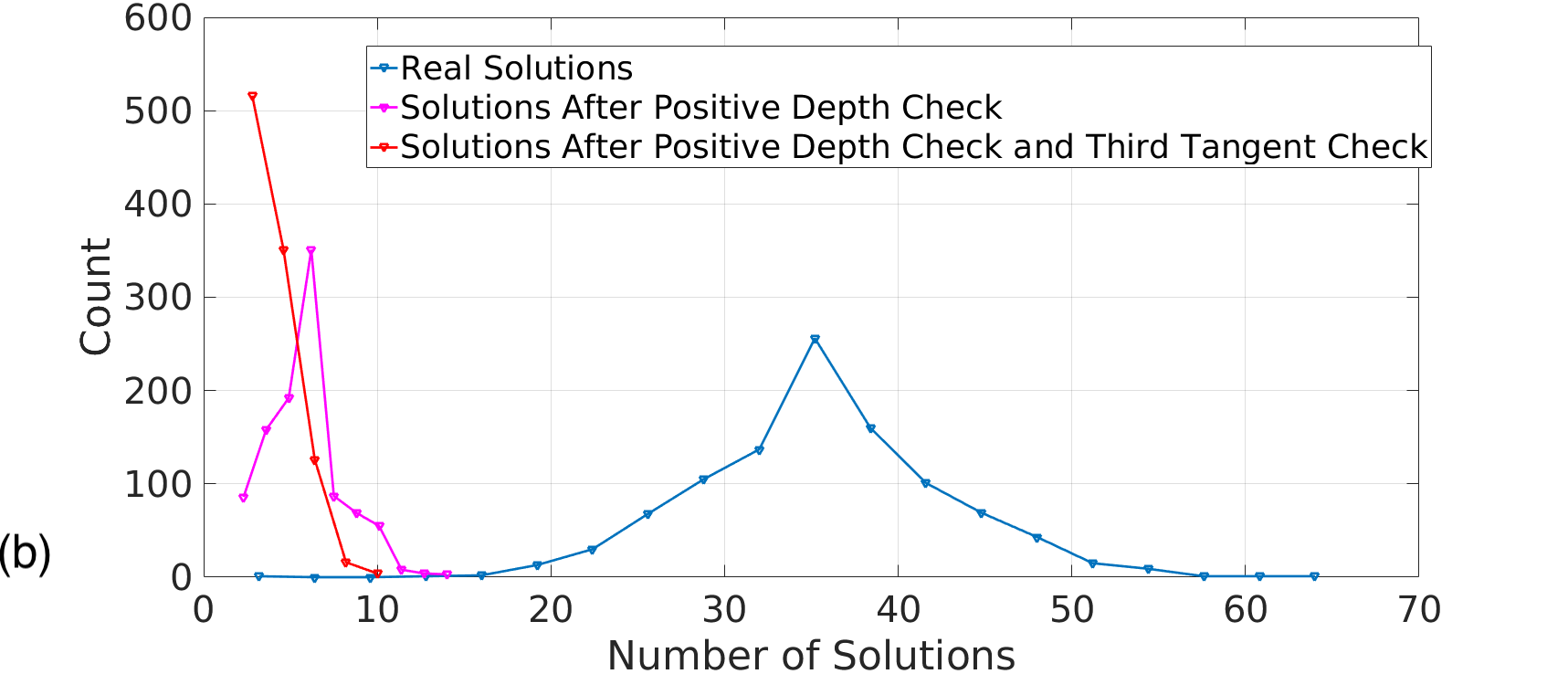}
%  \end{tabular}
  \caption{(a)  Errors of computed pose are small showing that the solver is
  numerically stable. (b) The distributions of the numbers of solutions.}
  \label{fig:numerical-stability}
\end{figure}%
%\paragraph{Feasible solutions and runtime}
% As shown in the previous sections, for one test case, the minimal solver will return averagely 52 real solutions. Not all of them are correct, in that the geometric constraint is not considered in this minimal solver. In practice, the real solutions will be filtered using the positive depth constraint. The way to do that is to triangulate each correspondence and check if the depth is positive or not. The points behind the camera will be deleted. After this procedure, $90\%$ of the solutions will not be considered as feasible solutions. The multiple feasible solutions, in practice, will be filtered using \ransac\ scheme. 
%
% Each step of minimal solve takes averagely 30s in Macaulay2 called from Matlab, on a standard 2.5GHz computer. The \ransac\ scheme and the filtering of the solution are implemented in Matlab and each step of filtering non-feasible solution takes 50 ms in Matlab. A C++ parallel computing implementation of homotopy continuation will drop the solution time to 1.9 second on a standard 2.5GHz computer.

The second experiment shows that we can reliably and accurately determine
camera pose with correct but noisy correspondences. Using the same dataset and
a subset of the selection of three triplets of points and tangents -- 200 in
total -- zero-mean Gaussian noise at various levels was added both to the feature locations
and to the orientation of the tangents, reflecting expected feature localization
and orientation localization error. The noise levels on points
and tangents reflect those found in curve extraction methods~\cite{Kimia:Li:Guo:PAMI18}. A \ransac\ scheme
determines the feature set that generates the highest number of inliers. 
Experiments indicate that the translation and rotation errors are
reasonable. Fig.~\ref{fig:NoiseAnalysis} (top) shows how localization error
affects pose under a fixed orientation perturbation of $\SI{0.1}{rad}$;
Fig.~\ref{fig:NoiseAnalysis} (bottom) shows how the extent of orientation error
affects pose under a fixed localization error of $\SI{0.5}{\pixel}$ (pixels).
\begin{figure}
\centering
\includegraphics[width=0.505\linewidth]{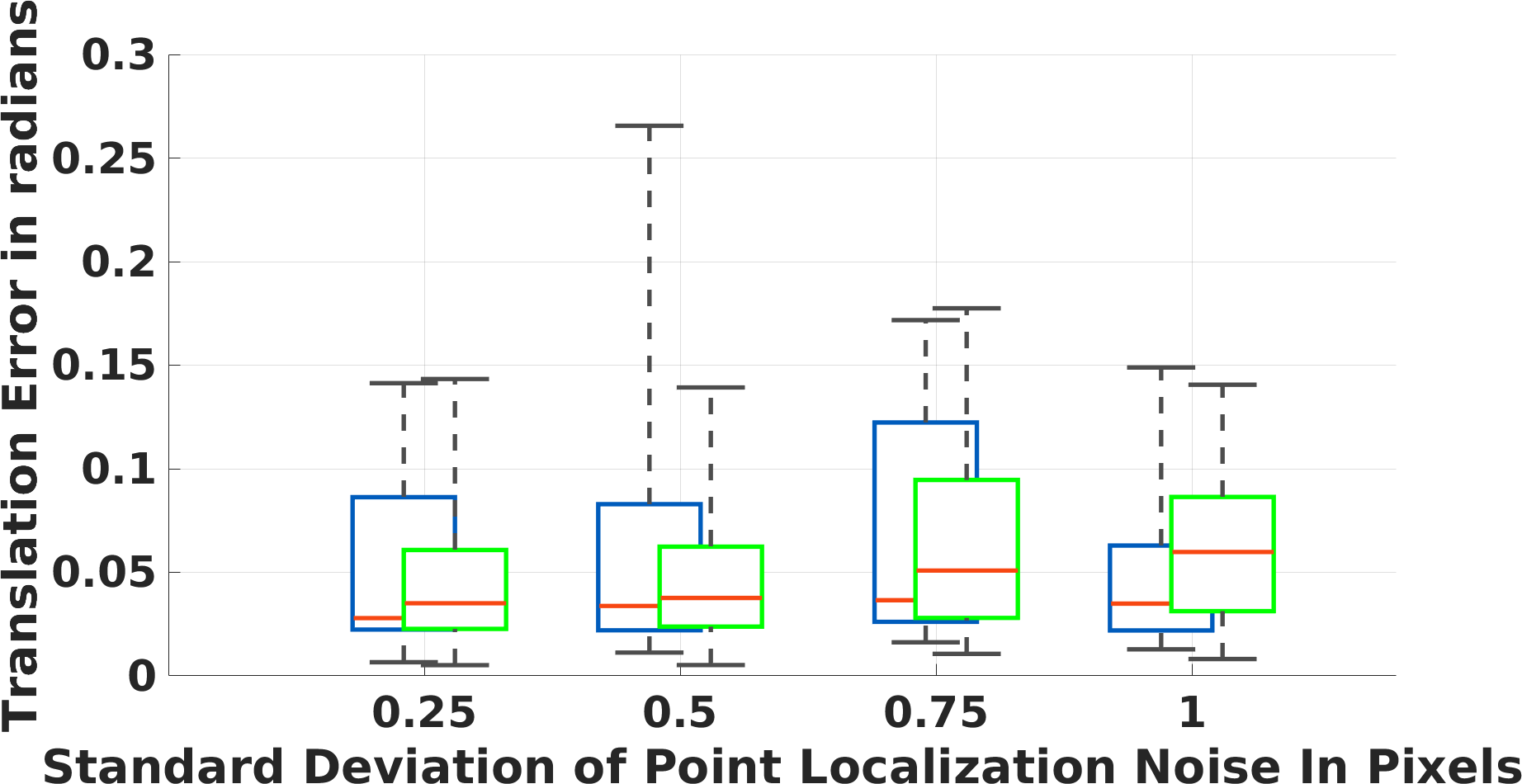}%
\includegraphics[width=0.505\linewidth]{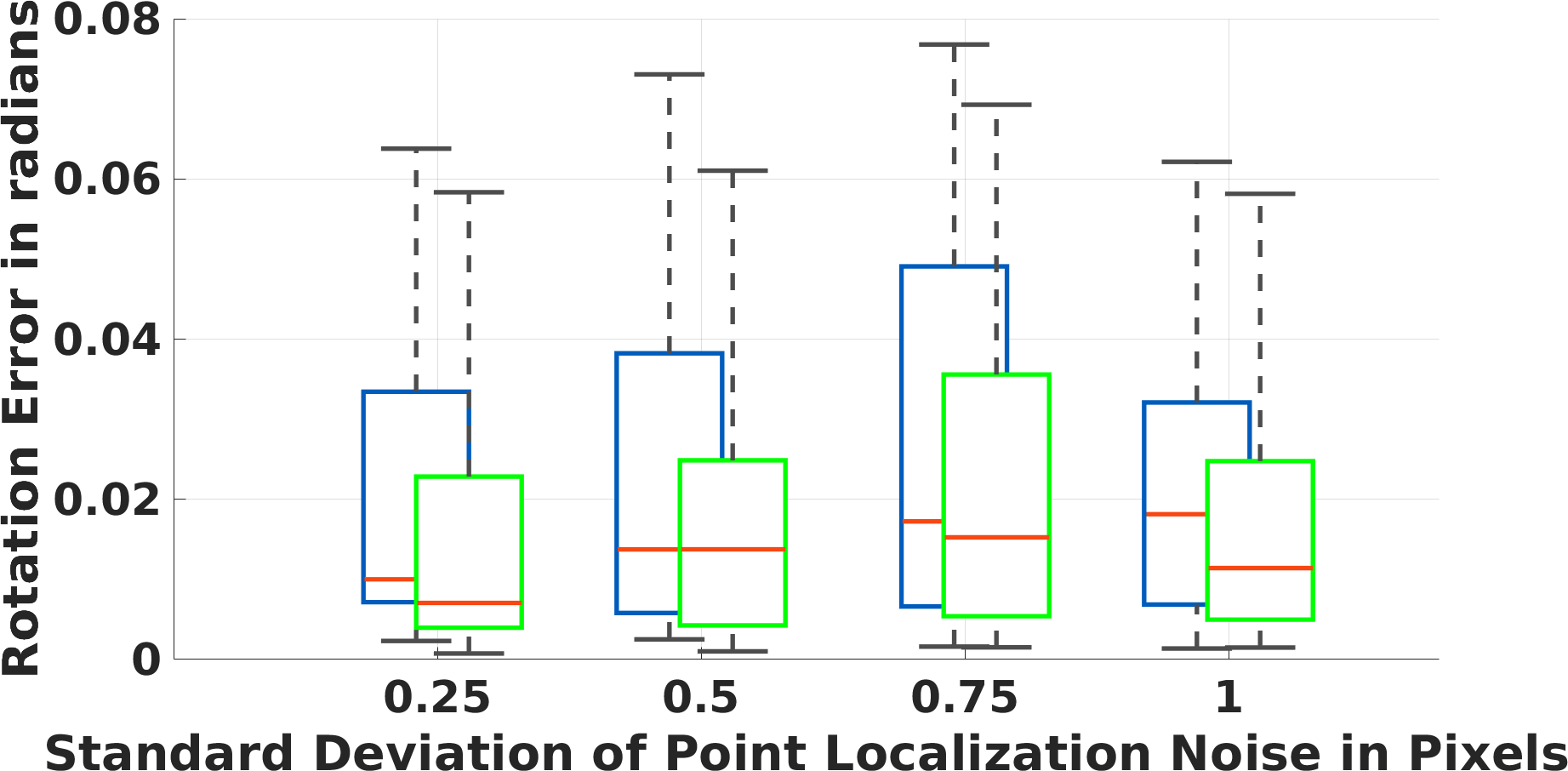}
\includegraphics[width=0.505\linewidth]{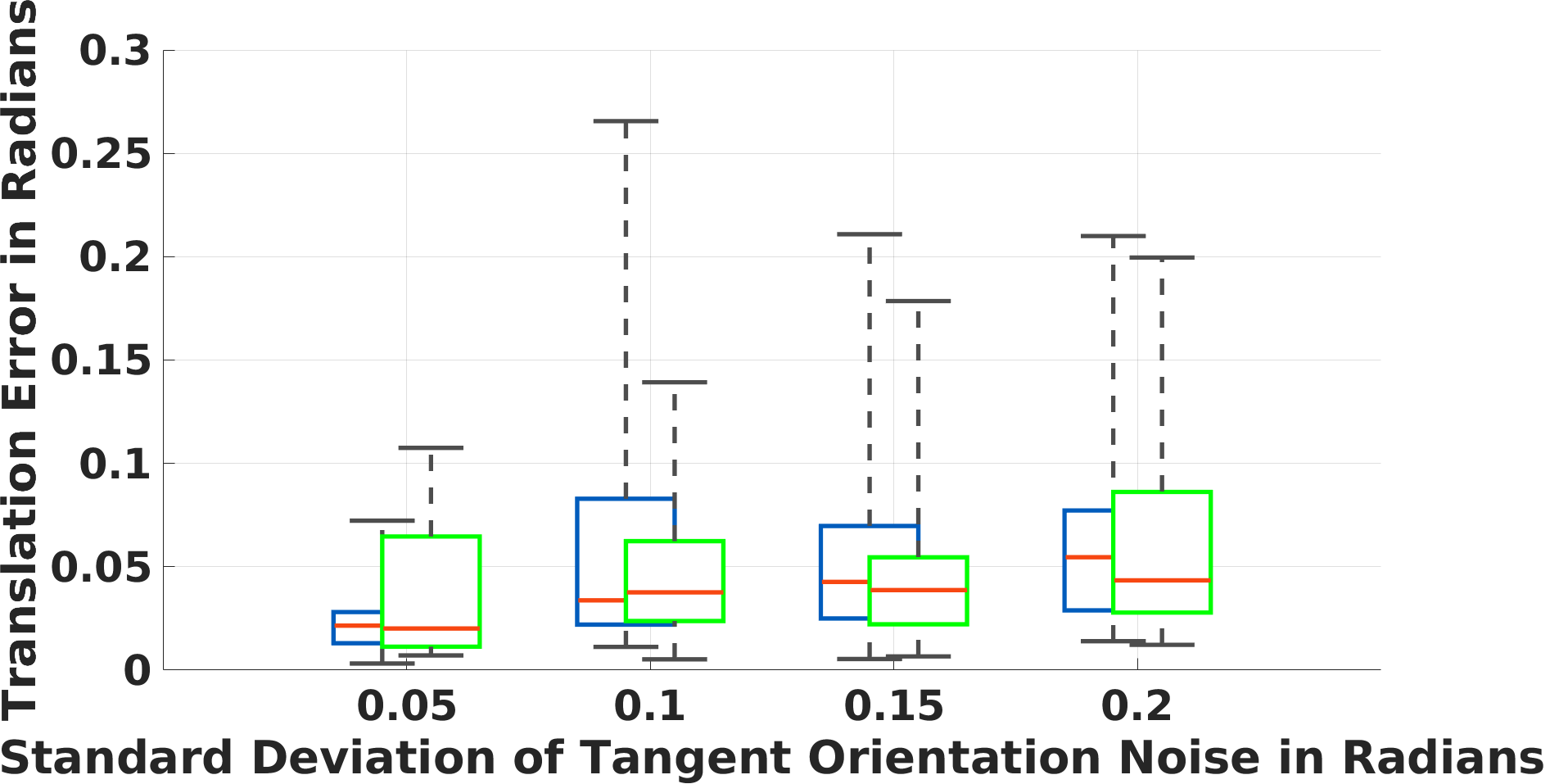}%
\includegraphics[width=0.505\linewidth]{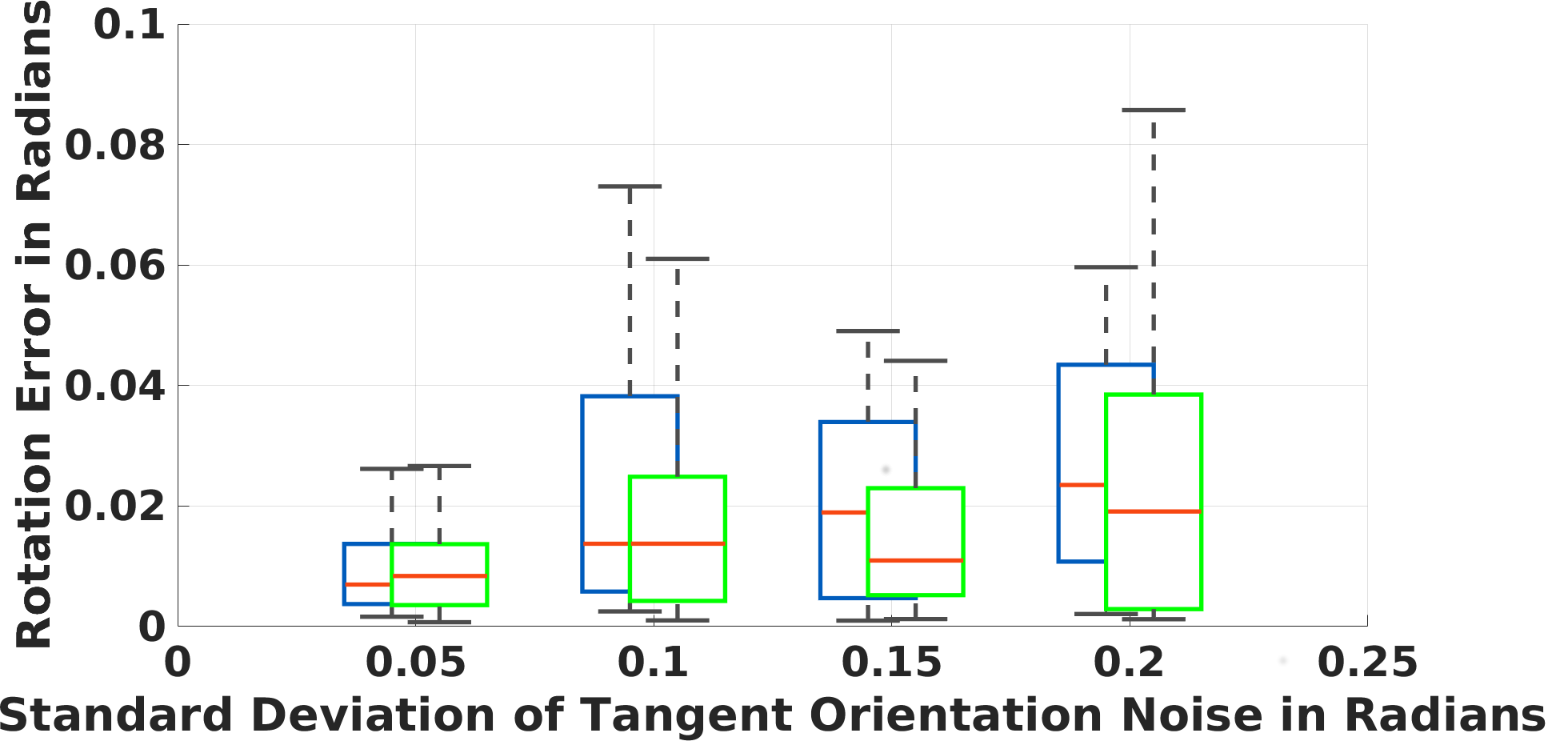}
\includegraphics[width=0.4605 \linewidth]{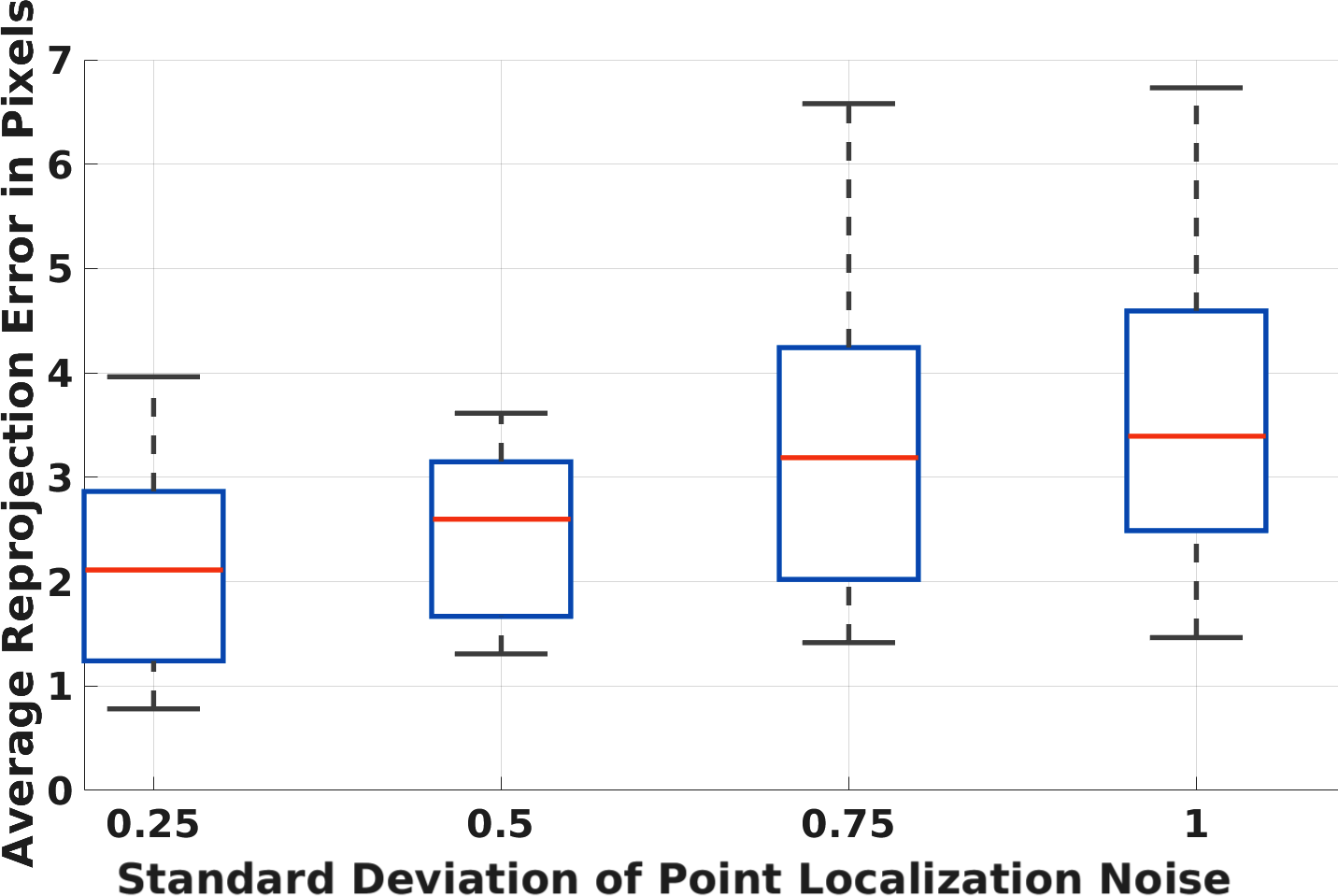}%
\includegraphics[width=0.53 \linewidth]{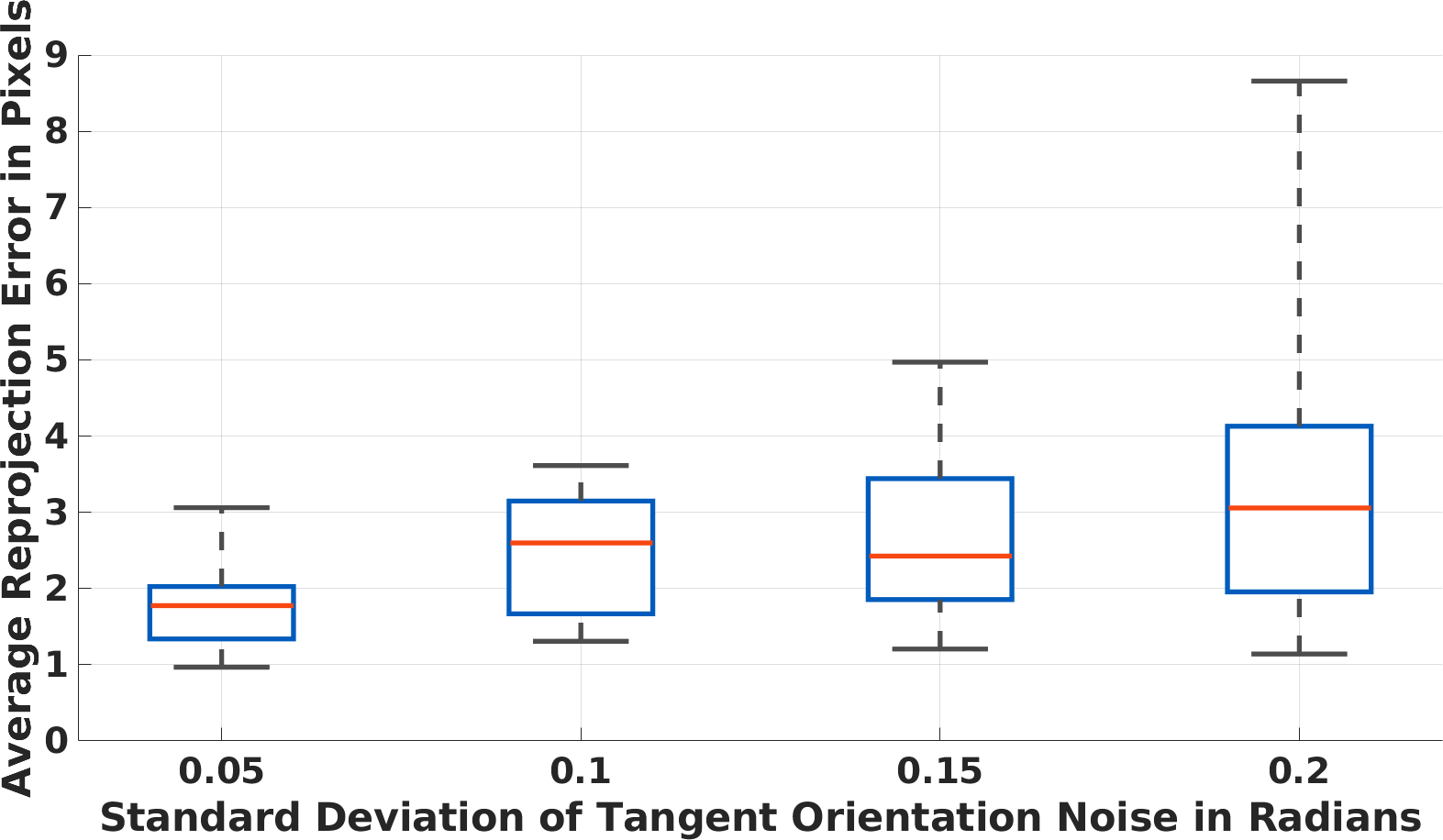}
 \caption{Pose error between views 1 and 2 (blue) and 1 and 3 (green)
   \emph{vs.}\ feature localization (top) and orientation noise (middle), and
   point reprojection error \emph{vs.}\ localization and orientation noise (bottom).
.} \label{fig:NoiseAnalysis}
\end{figure}
The more meaningful reprojection error, {\em i.e.}, the distance of a point from
the location determined by the other two points, is shown in
Fig.~\ref{fig:NoiseAnalysis}(bottom), averaged over 100 triplets.

The third experiment shows the system can consistently
estimate trifocal pose in the presence of outliers. With a feature localization
error of $\SI{0.25}{\pixel}$ and orientation error of 
$\SI{0.1}{\radian}$, 200 triplets of features were generated, with a fraction
having random location and orientation. The
ratio of outliers is varied over 10\%, 25\% and 40\%, with the experiment repeated 100 times each.
The resulting reprojection error is small and extremely stable,
with median $\SI{2}{\pixel}$ and maximum $\SI{3.6}{\pixel}$ for all outlier ratios.

\noindent{\bf Computational efficiency:} Each solve using our software \minus\
with conservative parameters takes $\SI{440}{\ms}$ ($\SI{660}{\ms}$ in the
worst case), compared to over $1$ minute on average for the best prototypes using general
purpose \hc\ software~\cite{BertiniBook,NAG4M2}, on an Intel core i7-7920HQ with
processor, \textsc{gcc} 5, and four threads. More aggressive but potentially unsafe optimizations towards
microseconds are feasible, but require assessing failure rate.
\begin{figure}[t]
    \centering
   \includegraphics[width=\linewidth,height=0.5\linewidth]{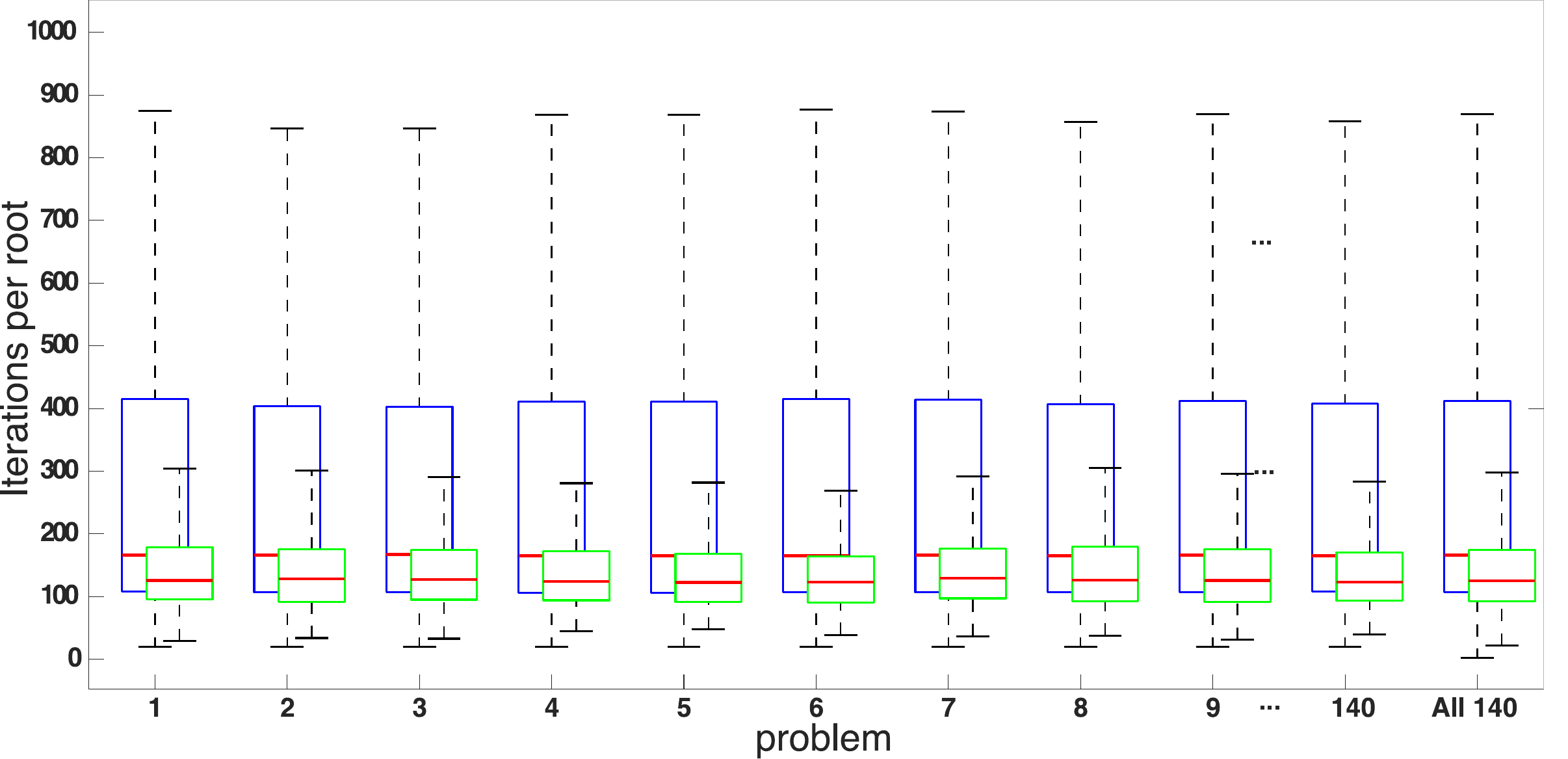}
    \caption{
      Time ($1$ iteration $\approx 1\,\mu\text{s}$)
      spent in root paths leading to ground-truth (green) \emph{vs.} undesired
      roots (blue), is stable across configurations.
      %thus the variance comes from the geometry of the solution variety under
      %randomization of gamma trick and affine charts.\todo{translate,buff}
    }
    \label{fig:niter}
\end{figure}
\begin{figure}
    \centering
    % final obtained from keynote in trifocal.key
    \includegraphics[width=\linewidth]{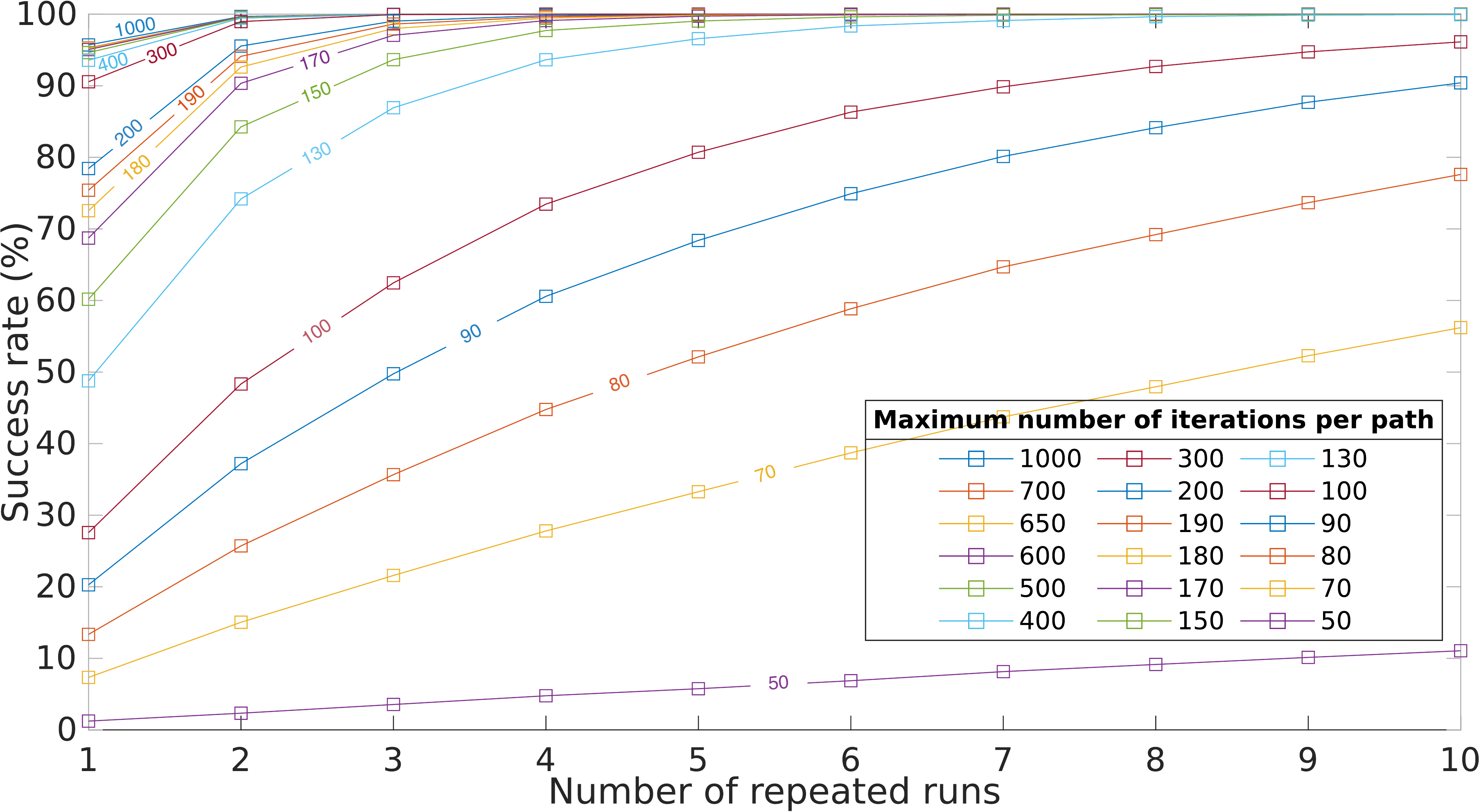}
    \caption{Tradeoff of success rate \textit{vs.} number of iterations per root.}
    \label{fig:nrep}
\end{figure}
%\begin{figure}
%    \centering
%    \includegraphics[width=0.67\linewidth]{figs/synthetic/Outlier/Outlier.png}
%    \caption{Reprojection error on ground truth with different ratio of outliers.}
%    \label{fig:OutlierResult}
%\end{figure}

To assess putting a cap $N_{max}$ on the number of predictor iterations per root,
we first observe that after $10^4$ random solves on synthetic data, the maximum number of
iterations for paths leading to ground-truth was close to $10^3$,
\emph{vs.}\ about $254\times 10^3$ for the wasted paths.
Given that the solve is $\approx 1-4$~$\mu\text{s}$ per iteration, this leads to
concrete routes to optimization. Fig.~\ref{fig:niter} shows that 
the time for roots leading to ground truth \emph{vs.} undesired
paths differ but remain strikingly stable across 140 different random input configurations.
For each configuration out of 140, \minus\ was run 500 times with different randomizations
$\gamma$'s and affine patch parameters.  The minimum number of iterations for
all configurations was 18.

Setting $N_{max} < 10^3$ costs a decrease in
success rate. However, we can regain success rate by
re-runing $N_{rep}$ times with different randomizations. Fig.~\ref{fig:nrep}
shows that running once with $N_{max} = 500$ yields a success of $97\%$,
which is the current default for \minus, providing the average figure of $401ms$.
Running twice with $N_{max} = 200$ yields a similar success rate.
For each $(N_{max},N_{rep})$ operating point, a success is counted if \minus\
found the solution in \emph{any} $N_{rep}$ runs; the final success rate
is averaged by performing this procedure $7000$ times.
If all points have tangents, \eg, 3 \sift\ features, 
as soon as a root reached an \hc\ stop condition we test for positive depth
and stop upon compliance with the third tangent to produce a
hypothesis for \ransac, cutting down average execution time further
with a modest decrease in success rate. The run time remains on the order of
$\SI{100}{\ms}$.\\[1em]
%
%Even with a
%severe cap, the root always appears to
%be found with sufficiently large $N_{rep}$. This shows that tens of \hc\
%iterations are actually sufficient to solve trifocal pose  --- microsecond-scale
%performance is in theory achievable even with generic start solutions; the key
%lies in finding a good initial path direction, perhaps through training of
%algorithm random distributions and by backtracking long paths to other
%randomization directions while keeping maximum depth.
%
%
%No gráfico de acerto/sucesso tds os problemas estão juntos, ou seja, são 500*138
%testes/rodadas. Esses testes foram divididos em 50*138 grupos/experimentos de n
%testes (no máximo 10). No eixo horixontal está o número de testes considerados
%em cd caso, de modo que haja sempre a mesma quantidade de experimentos.
%Considerou-se sucesso encontrar a solução em pelo menos 1  dos n testes. Cd
%curva corresponde aos testes para diferentes nmax de iterações em cd raiz.
%
%
\noindent\textbf{Real data experiments:} 
Much like the standard pipeline, \sift\ features are first extracted from all
images. Pairwise features are found by rank-ordering measured similarities and
making sure each feature’s match in another image is not ambiguous and is above
accepted similarity. Pairs of features from the first and second views are then
grouped with the pairs of features from the second and third views into
triplets. A cycle consistency check enforces that the triplets must also support
a pair from the first and third views. Three feature triplets are then selected
using \ransac\ and the relative pose of the three cameras is determined from two
\sift\ orientations and a third point without orientation. 

Fig.~\ref{fig:realData} shows that camera pose is reliabily and accurately found using triplets of images from 
the \textsc{epfl} dense multi-view stereo image dataset~\cite{Strecha:etal:CVPR08}.
Our quantitative estimates on 150 random triplets from this dataset give pose
errors of $\SI{1.5e-3}{\radian}$ in
translation and $\SI{3.24e-4}{\radian}$ in rotation. The average
reprojection error is $\SI{0.31}{\pixel}$. These are comparable to or better than the interest point-based
trifocal relative pose estimation methods reported in~\cite{Julia:Monasse:SIVT2017}.
Our conclusion for this dataset, whose purpose is to validate the solver, is that our method is at least as good and often
better than the traditional ones. Note that we do not advocate replacing the
traditional method for this dataset. We simply state that our method works just
as well, of course at a higher cost.

The \textsc{epfl} dataset is feature-rich, typically yielding on the order of
$10^3$ triplet features per image triplet. As such it does not portray some
of the typical problems faced in challenging situations when there are few
features available. The Amsterdam Teahouse
Dataset~\cite{Usumezbas:Fabbri:Kimia:ECCV16}, which also has ground-truth
relative pose data, depicts scenes with fewer features. Fig.~\ref{fig:fail}
(top) shows a triplet of images from this dataset where there is a sufficient
set of features (the soup can) to support a bifocal relative pose estimation
followed by a \textsc{p3p} registration to a third view (using
\colmap~\cite{schoenberger2016sfm}). However, when the number of features is
reduced, as in Fig.~\ref{fig:fail} (bottom) where the soup can is occluded,
\colmap\ fails to find the relative pose between pairs of these images. In contrast, our approach, which relies on three and not five features, is able to recover the camera pose for this scene. 

We also created another featureless dataset similar to the one
in~\cite{nurutdinova2015towards} but with the calibration board manually
removed. This scene lacks point features, which is extremely challenging for
traditional structure from motion. We built 20 triplets of images within this
dataset. Within these 20 triplets, camera poses of only 5 triplets can be
generated with \colmap, but with our method, 10 out of 20 camera poses can be
estimated. We reached a 100\% improvement over the standard pipeline on image
triplets. The sample successful cases are shown in Figs.~\ref{fig:teaser} and~\ref{fig:cups}.

\begin{figure}
  \includegraphics[height=0.19\linewidth]{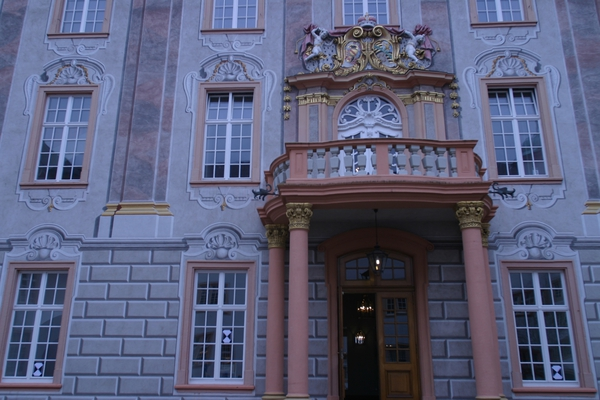}%
  \includegraphics[height=0.19\linewidth]{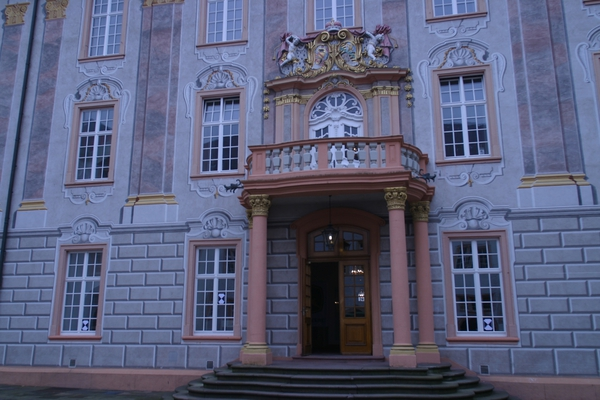}%
  \includegraphics[height=0.19\linewidth]{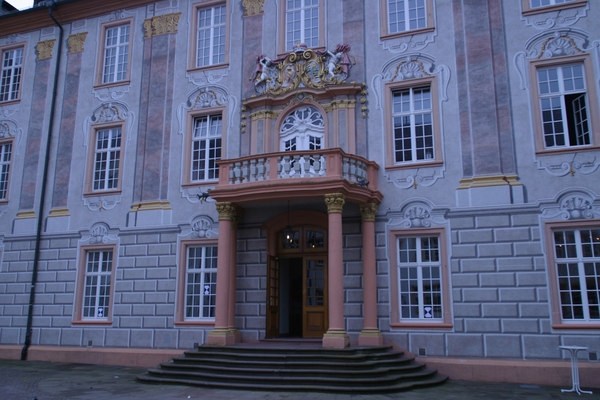}%
  \includegraphics[height=0.195\linewidth]{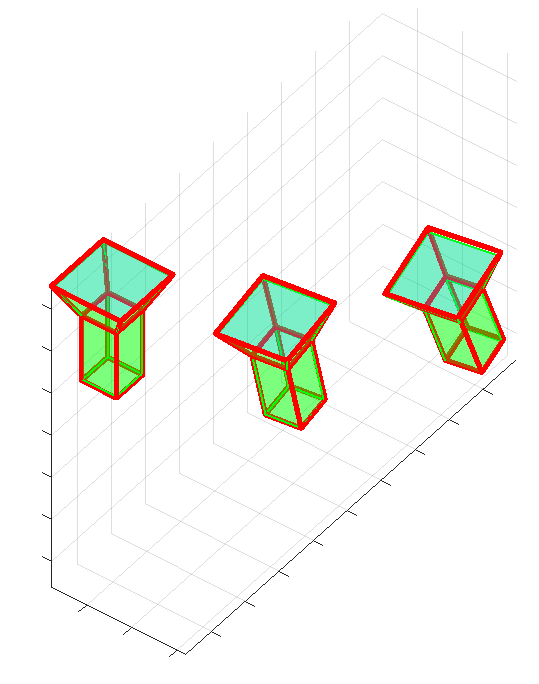}\\
  \includegraphics[height=0.19\linewidth]{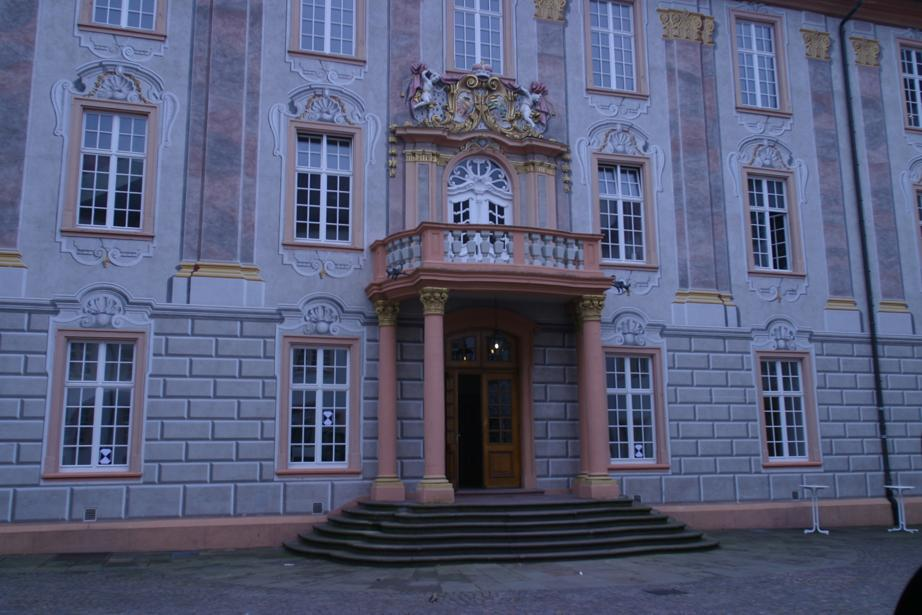}%
  \includegraphics[height=0.19\linewidth]{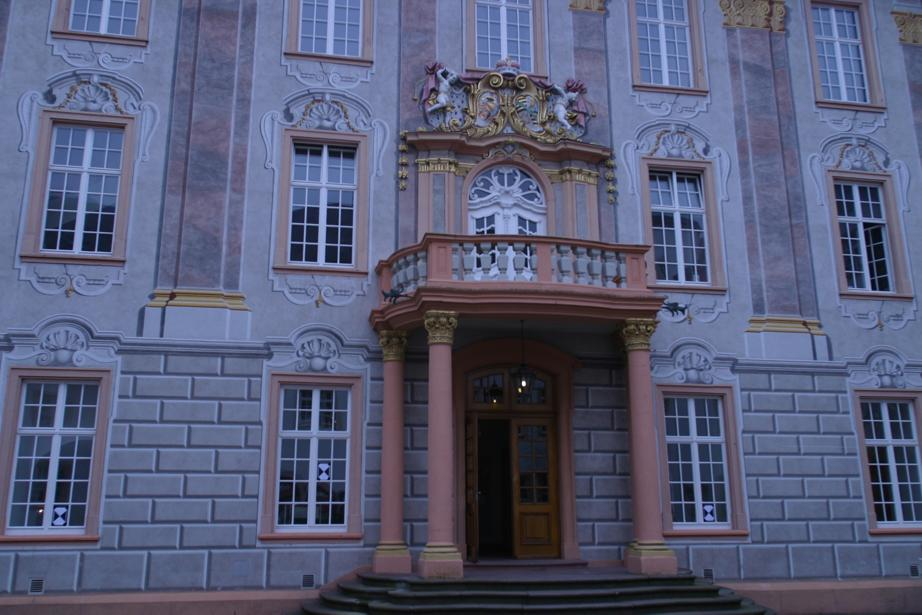}%
  \includegraphics[height=0.19\linewidth]{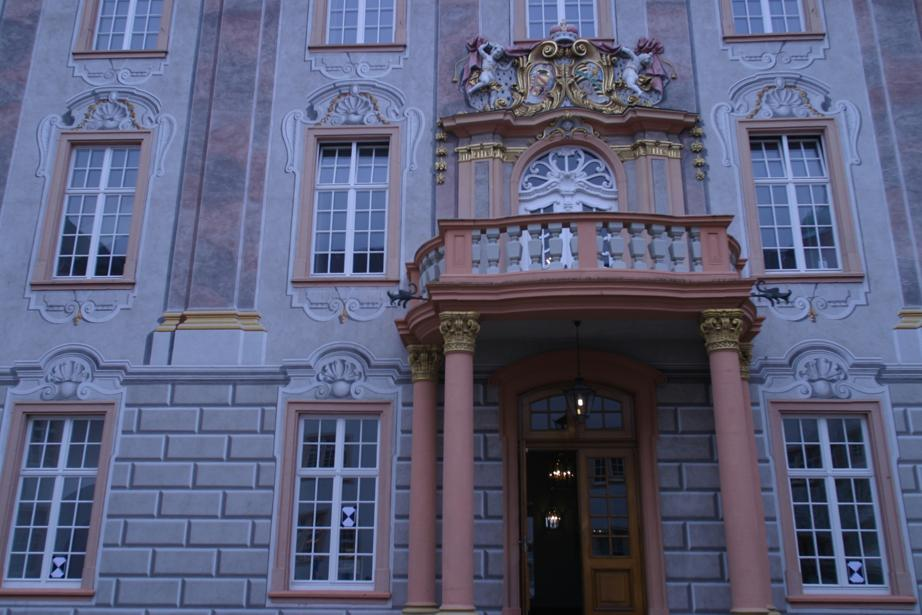}%
  \includegraphics[height=0.19\linewidth]{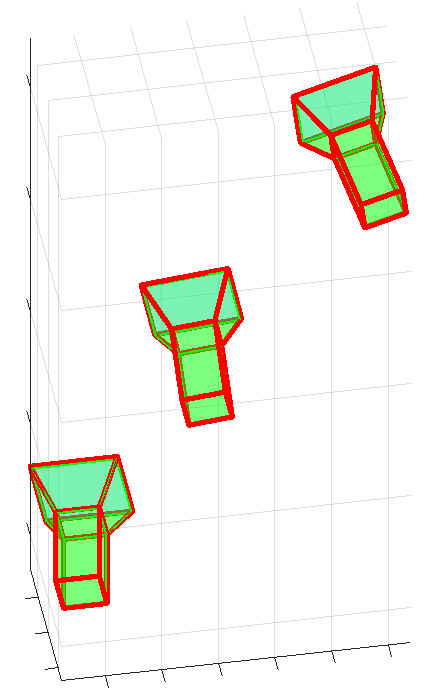}\\
   \includegraphics[height=0.19\linewidth]{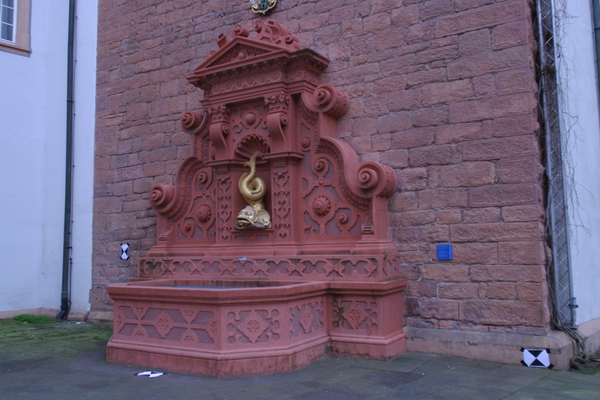}%
  \includegraphics[height=0.19\linewidth]{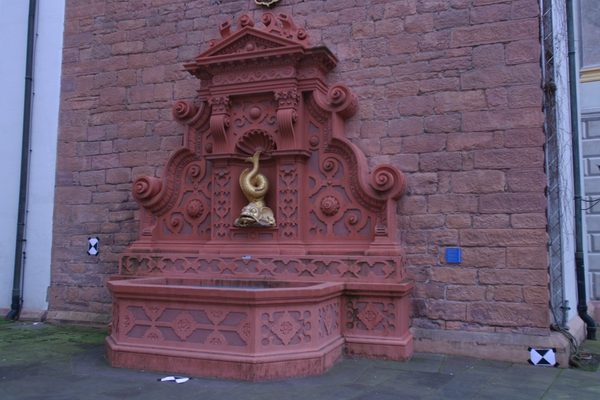}%
  \includegraphics[height=0.19\linewidth]{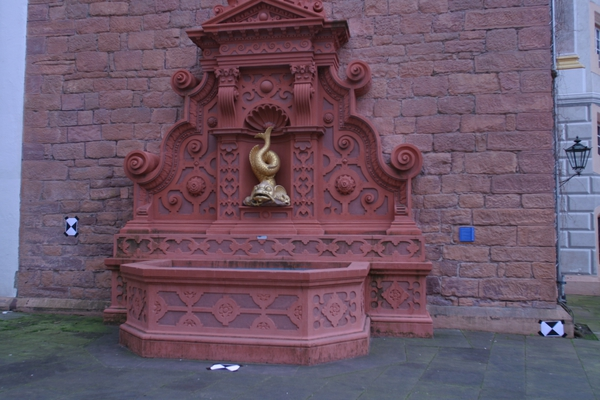}%
  \includegraphics[height=0.14\linewidth]{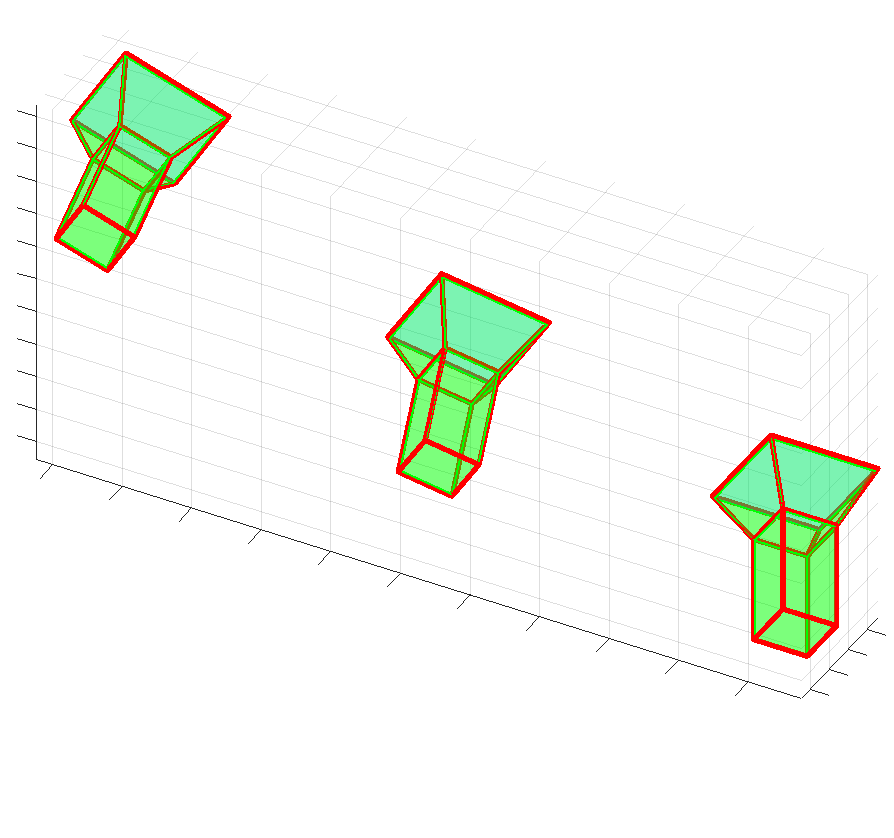}\\
  \includegraphics[height=0.19\linewidth]{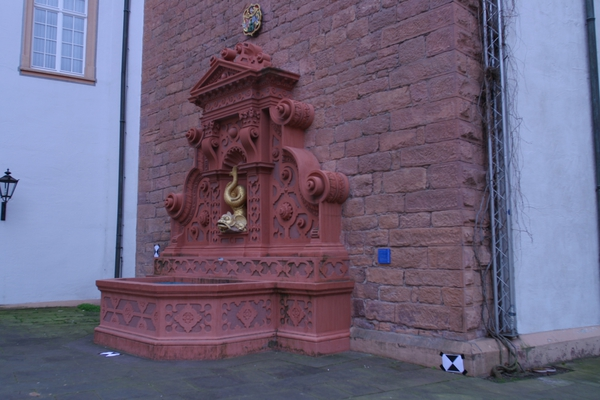}%
  \includegraphics[height=0.19\linewidth]{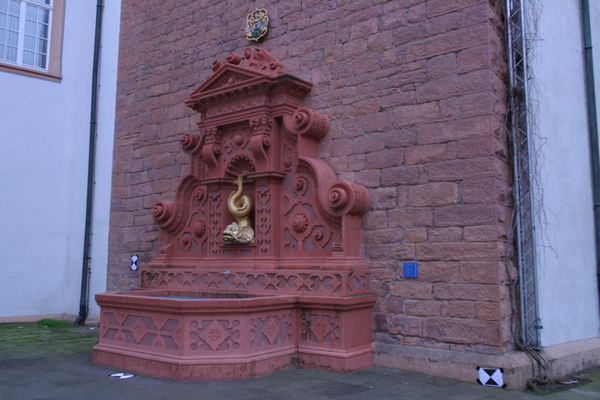}%
  \includegraphics[height=0.19\linewidth]{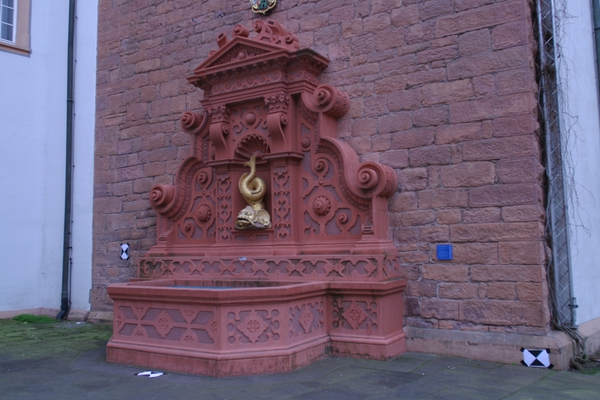}%
  \includegraphics[height=0.19\linewidth]{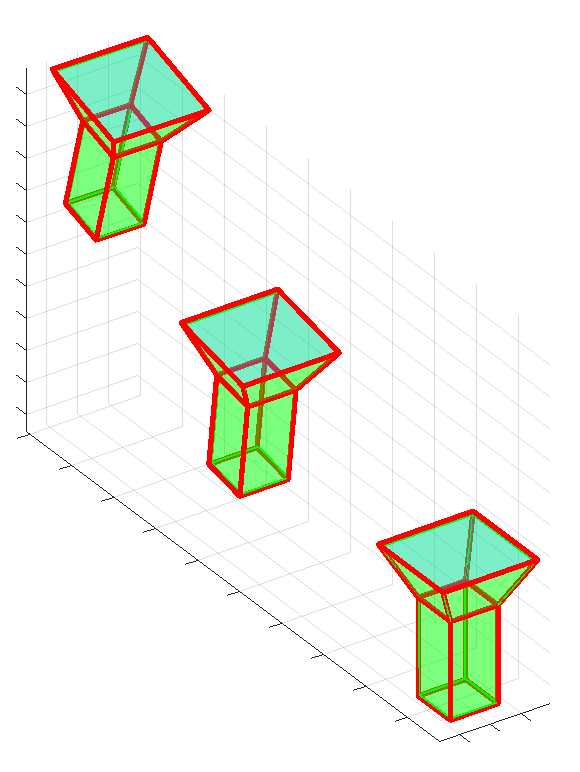}\\
   \includegraphics[height=0.19\linewidth]{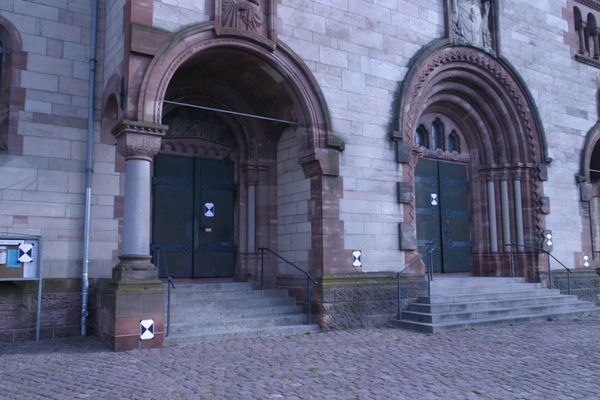}%
  \includegraphics[height=0.19\linewidth]{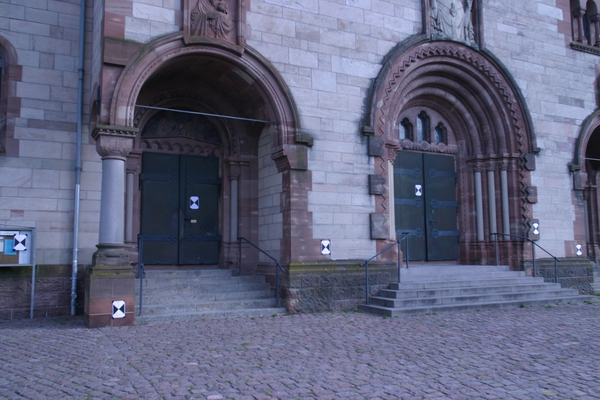}%
  \includegraphics[height=0.19\linewidth]{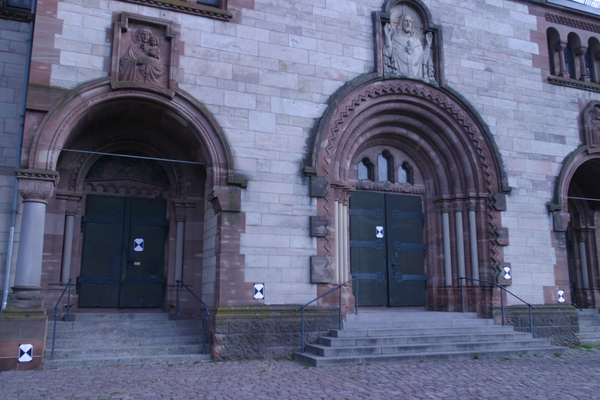}%
  \includegraphics[height=0.145\linewidth]{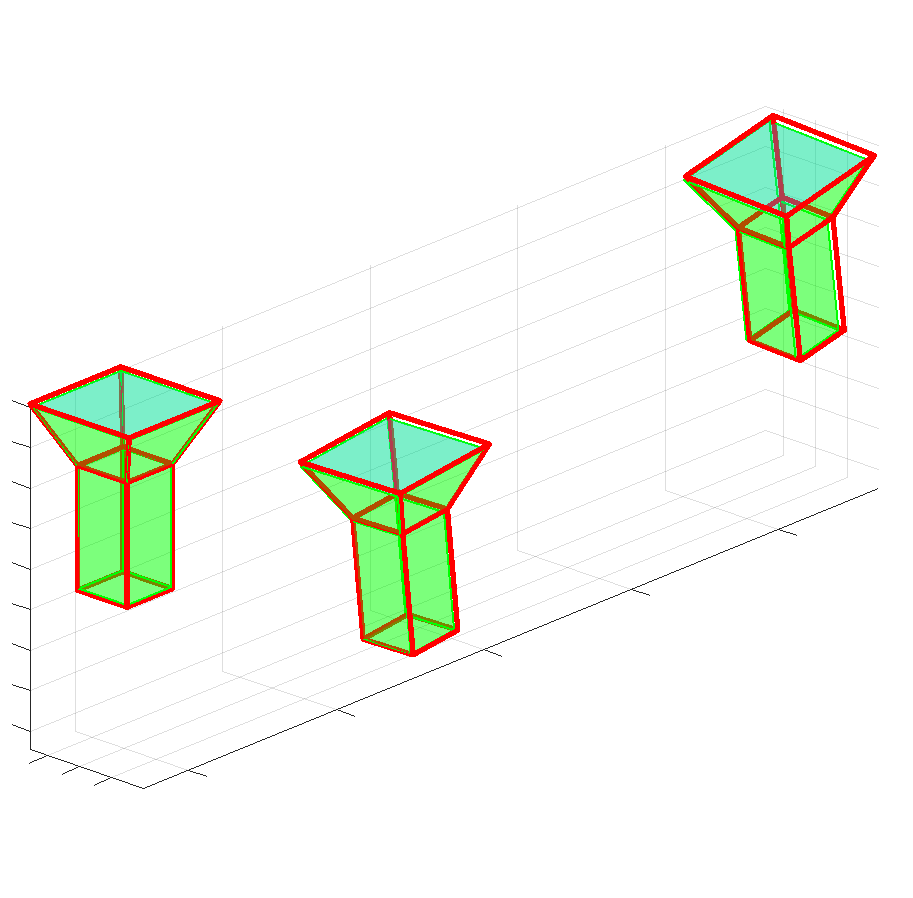}\\
  \includegraphics[height=0.19\linewidth]{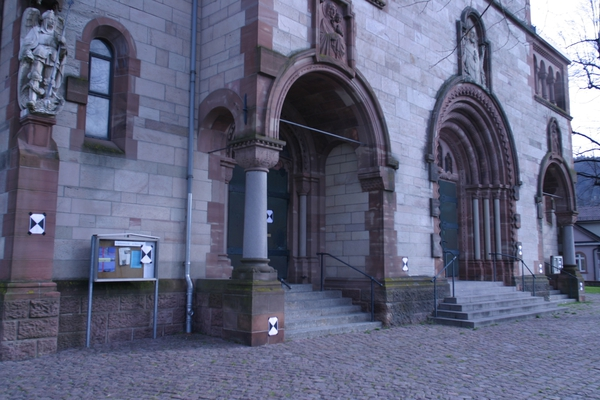}%
  \includegraphics[height=0.19\linewidth]{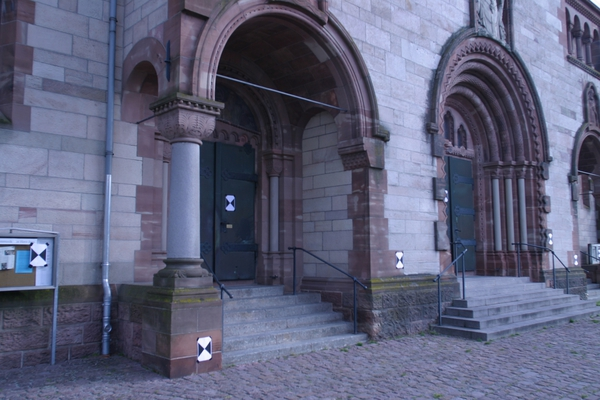}%
  \includegraphics[height=0.19\linewidth]{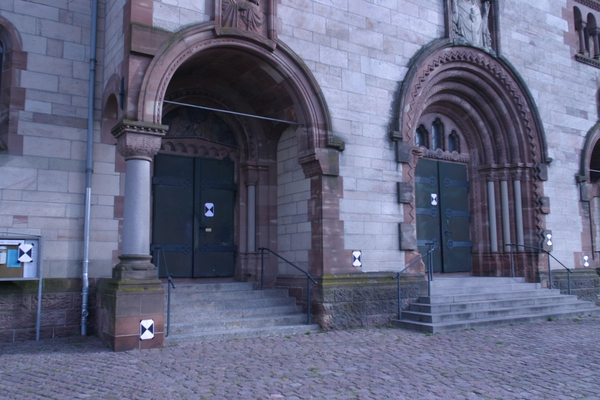}%
  \includegraphics[height=0.158\linewidth]{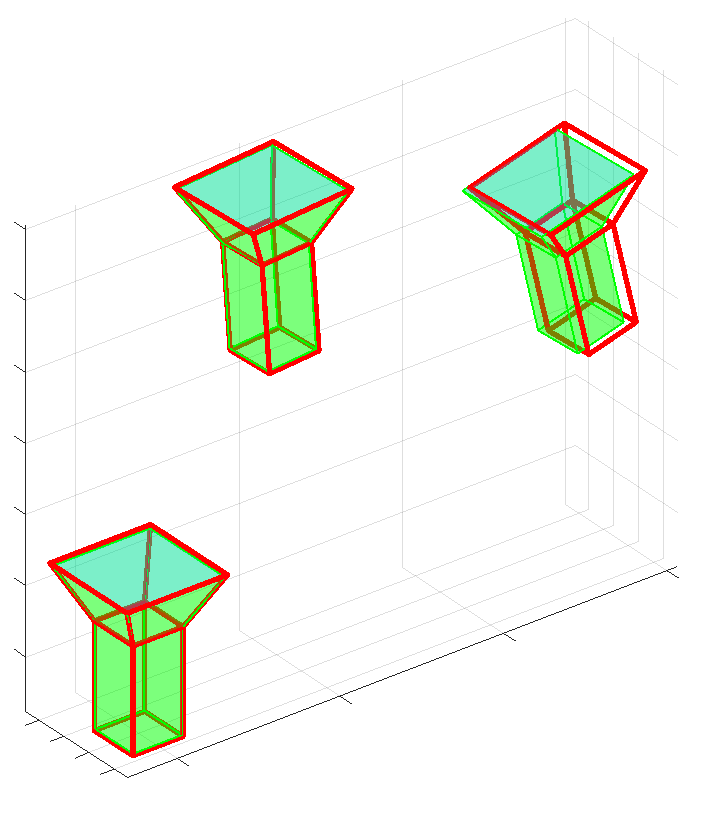}
  \caption{Trifocal relative pose results for \textsc{epfl} dataset. Each row
  shows images with ground truth (green) and estimated poses (red
outline).}
  \label{fig:realData}
\end{figure}
%The error in camera pose with respect to the ground-truth pose is measured and shown in Figure XX. The reprojection error is also shown in Figure XX.
\begin{figure}
   \includegraphics[height=0.19\linewidth]{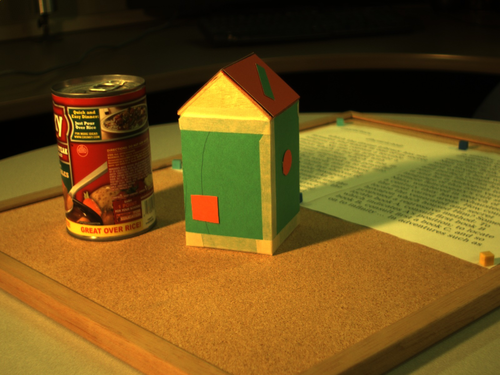}%
  \includegraphics[height=0.19\linewidth]{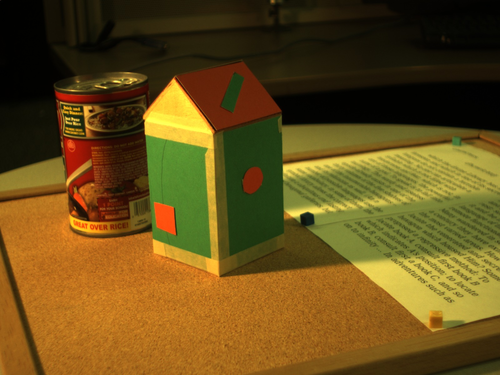}%
  \includegraphics[height=0.19\linewidth]{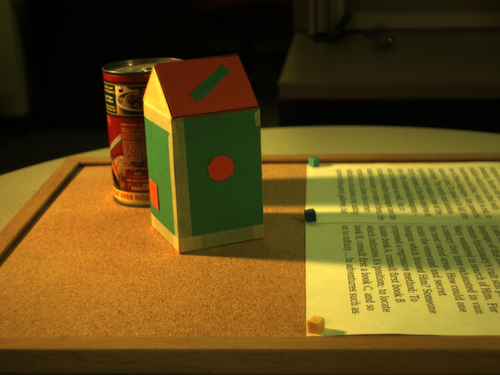}%
  \includegraphics[height=0.14\linewidth]{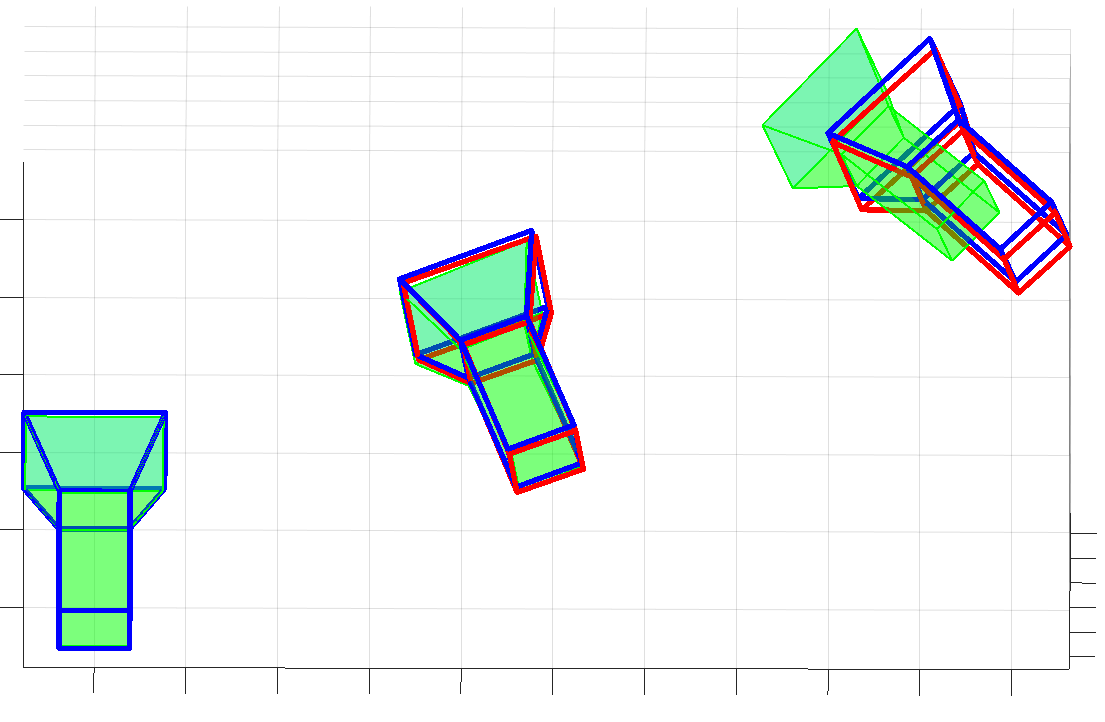}\\
   \includegraphics[height=0.19\linewidth]{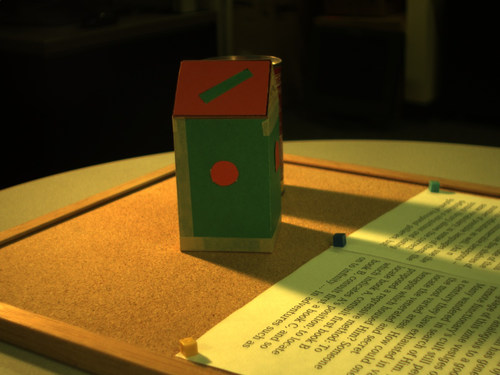}%
  \includegraphics[height=0.19\linewidth]{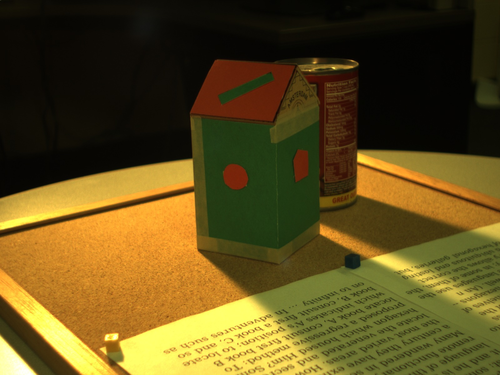}% 
  \includegraphics[height=0.19\linewidth]{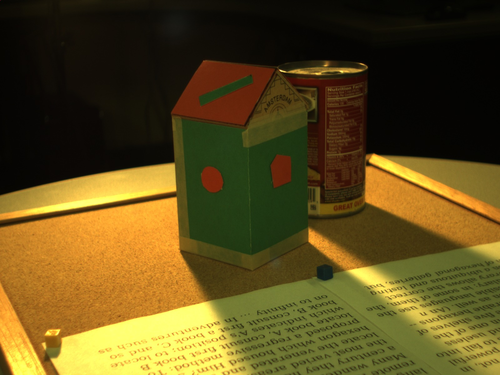}%
  \includegraphics[height=0.14\linewidth]{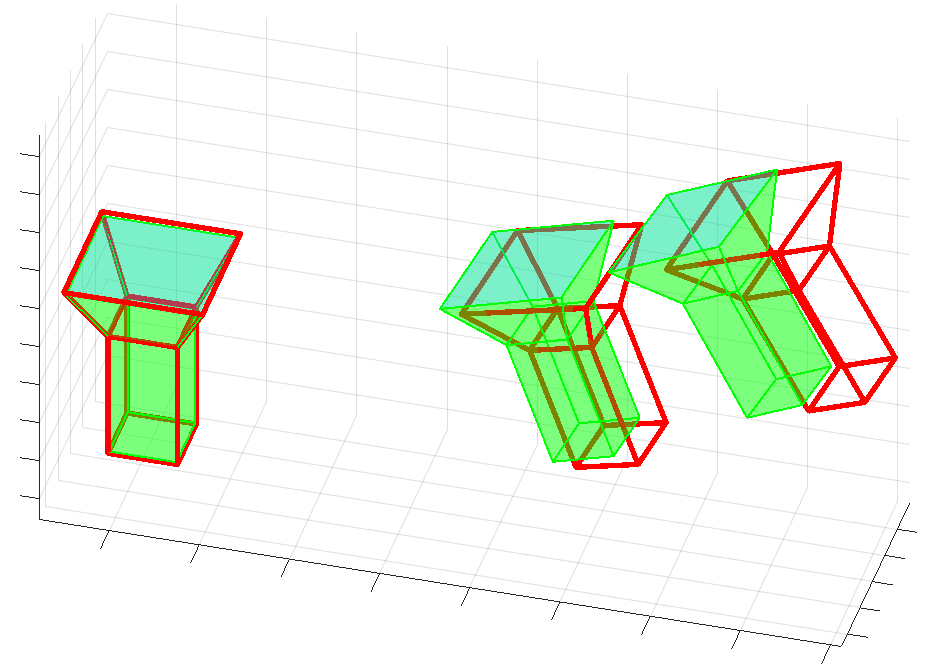}
    \caption{Trifocal relative pose results for 
    Amsterdam Teahouse: a triplet of images that
    \colmap\ is able to tackle (top) and where
    it fails (bottom). Results: \colmap\ (blue outline), ours (red), and ground
truth (green).} \label{fig:fail}
\end{figure}

\begin{figure}[t]
\includegraphics[height=0.21\linewidth]{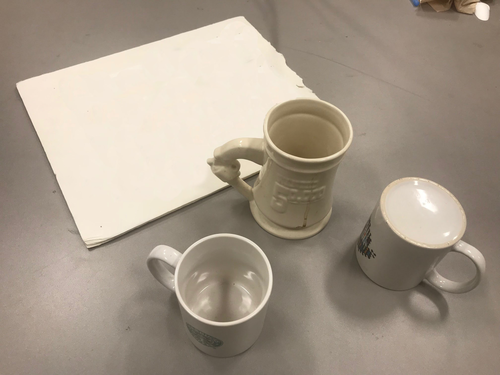}%
\includegraphics[height=0.21\linewidth]{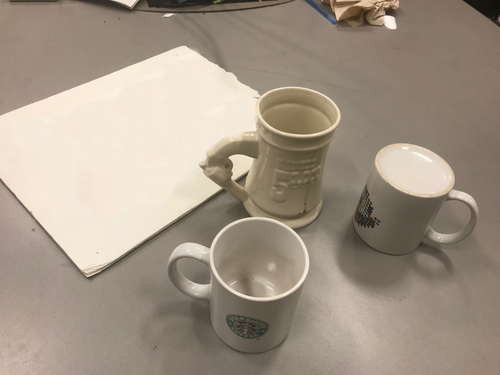}%
\includegraphics[height=0.21\linewidth]{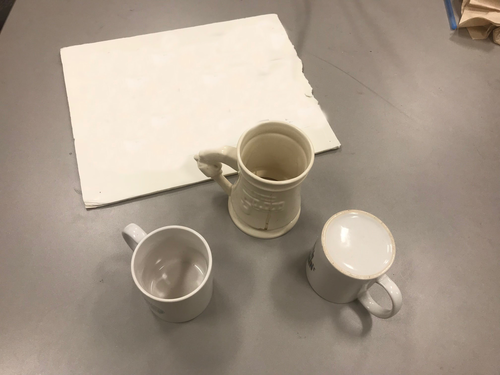}%
\includegraphics[height=0.21\linewidth]{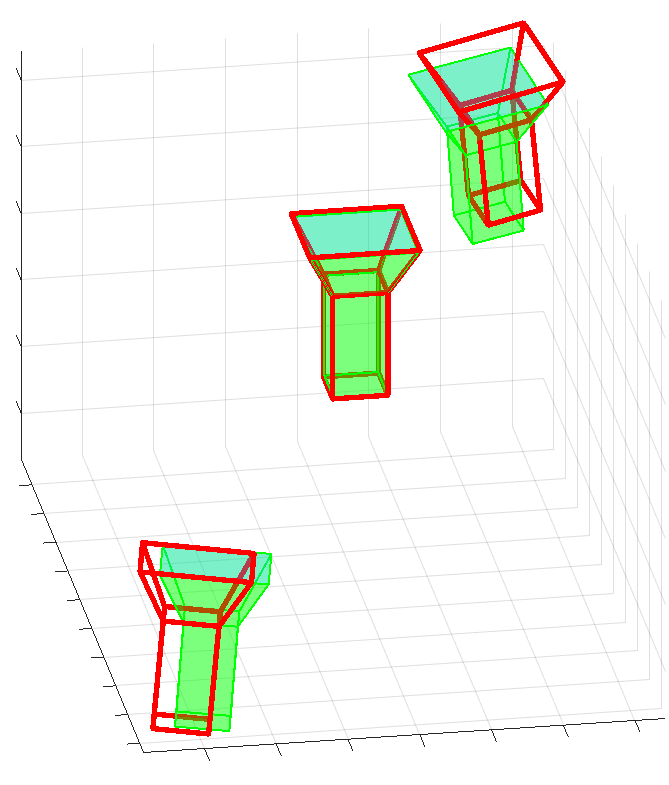}\\
\includegraphics[height=0.21\linewidth]{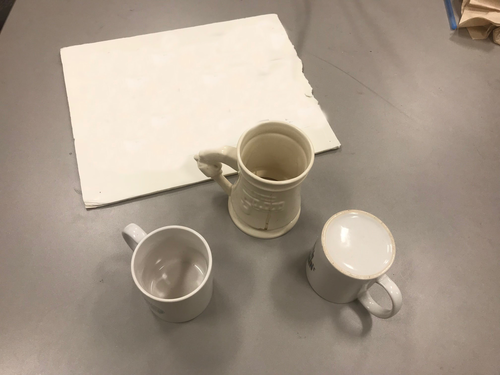}%
\includegraphics[height=0.21\linewidth]{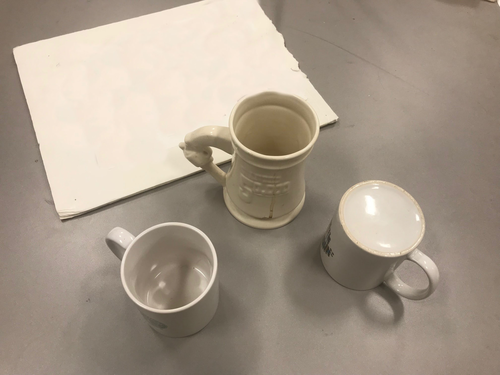}%
\includegraphics[height=0.21\linewidth]{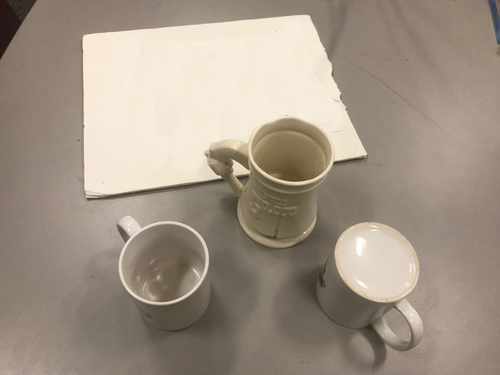}%
\includegraphics[height=0.21\linewidth]{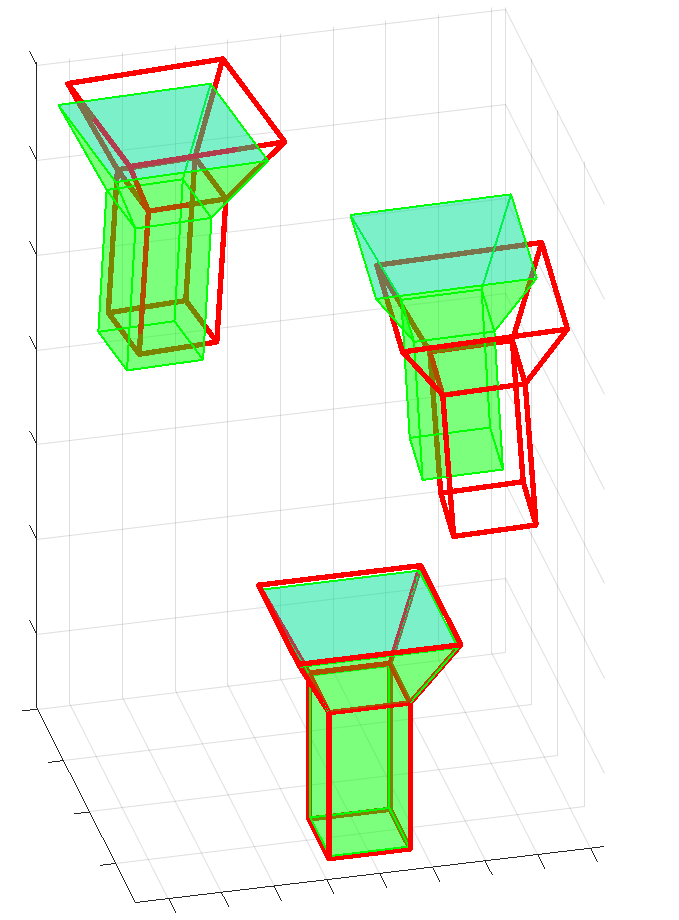}\\
\includegraphics[height=0.21\linewidth]{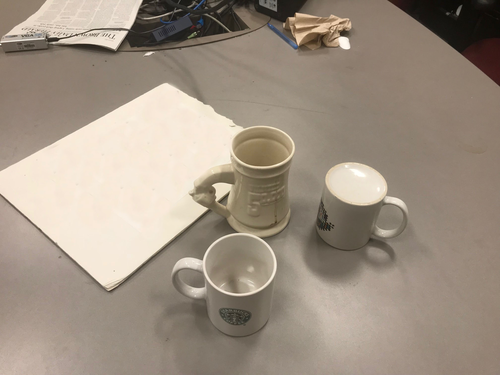}%
\includegraphics[height=0.21\linewidth]{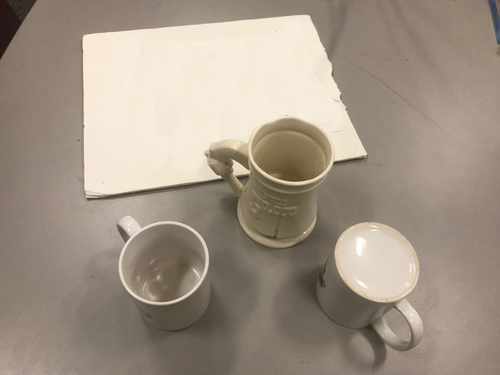}%
\includegraphics[height=0.21\linewidth]{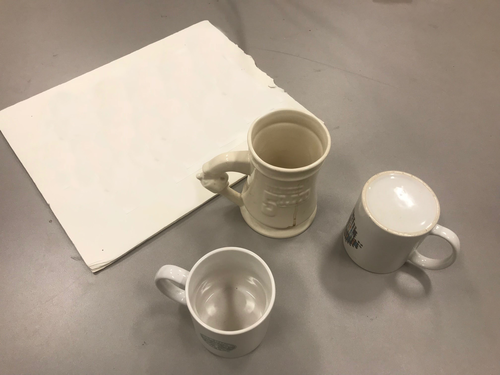}%
\includegraphics[height=0.21\linewidth]{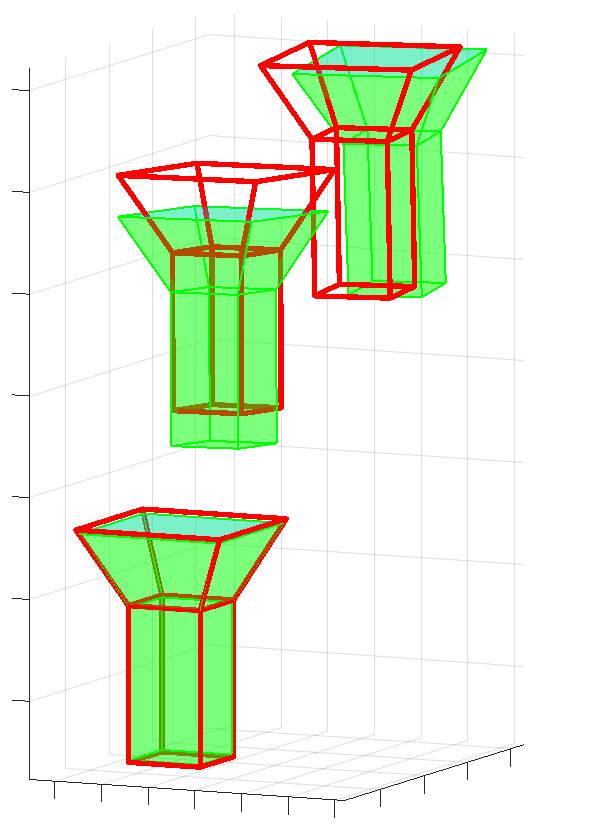}
   \caption{Trifocal relative pose results for a dataset comprising three mugs,
     which is challenging for traditional \sfm. Each shows images 
   with ground truth (green) and estimated poses (red outline).}
  \label{fig:cups}
\end{figure}

A quantitative comparison with other trifocal methods reported
in~\cite{Julia:Monasse:SIVT2017} on datasets Fountain P-11 and Herz-Jesu-P8 is
shown in Table~\ref{table:trifocalCompare} for the Chicago problem,
illustrating that our method is comparable to or better than other trifocal methods. 
\begin{table}[h]
    \centering
    \begin{tabular}{ccc}
      \textbf{Methods} & $R$ error (deg) & $T$ error (deg) \\
        \hline
        TFT-L & 0.292 & 0.638\\
        TFT-R & 0.257 & 0.534\\
        TFT-N & 0.337 & 0.548\\
        TFT-FP & 0.283 & 0.618\\
        TFT-PH & 0.269 & 0.537\\
        \textbf{MINUS (Ours)} &\textbf{ 0.137 }& \textbf{0.673}\\
    \end{tabular}
    \caption{Pose error of our method \emph{vs.}\ other trifocal methods. Our method has better rotation error and comparable translation error.}
    \label{table:trifocalCompare}
\end{table}

\section{Conclusion} 
\noindent We presented a new calibrated trifocal minimal problem, an analysis
demonstrating its number of solutions, and a practical solver by specializing
computation techniques from numerical algebraic geometry. We showed our approach
generalizes to characterize and solve a similar difficult minimal problem with mixed
points and lines in three views. Both problems are representative of a myriad
of similar minimal problems in multiple
views analyzed with the techniques initiated with the present work~\cite{duff2019plmp,Duff:Kohn:Leykin:Pajdla:PLMP:2019,Duff:Korotynskiy:Pajdla:Regan:Arxiv2021,Fabbri:NSF:Highlight2020,Fabbri:Kimia:NSF:Grant}.
The increased ability to solve trifocal problems with points and lines is key to future work on
broader problems appearing when observing general \textsc{3d} curves, \eg, in scenes
without enough point features, using differential geometry~\cite{Fabbri:PHD:2010,Fabbri:Kimia:IJCV2016}. As a first step,
our trifocal solvers have been partially integrated into the \sfm\ 
pipeline \textsc{o}pen\textsc{mvg}~\cite{Moulon:etal:OpenMVG:IWRRPR2016} for use with \sift\
orientation, and we are working to integrate and verify their robustness
advantages also with \colmap.
Our ``100 lines of
custom-made solution tracking code'' 
have also already been employed to build
practical, fast solvers~\cite{Hruby:Duff:Leykin:Pajdla:CVPR2022} for other minimal problems which have not been efficiently
solved with Gr\"obner bases~\cite{Larsson-CVPR-2018}.

\appendices
\section{Cleveland Minimal Problem: Formulations}

\noindent\textbf{Essential Parametric Equations.}
Explicitly derived from first principles (\ie, point projection), the free line correspondence
equations are analogous to~\eqref{eq:tangent:final:v3:parametric} but with a
translation term since the base points $\x_v$ in principle do not correspond. This enables
direct comparison to the incident line case and generalizes well to
arbitrary smooth curves~\cite{Fabbri:Kimia:IJCV2016}. We provide a condensed derivation. Suppose each given image line is 
parametrized as $\y_v = \x_v +\delta_v\dir_v$ for some $\x_v$ that need not
correspond. Using elemental point projection we require $\exists\,\X_0,\D$
such that, $\forall \epsilon$, $\exists\,\delta_v$ such that $\Y = \X_0 + \epsilon \D$ projects to
corresponding $\y_v$ in $\bl_v$. We thus
have~\eqref{eq:points:simplified:notation} but for $\y_v$ and $\beta_v$.
Using the parametric form of $\bl_1$,
\begin{equation}\label{eq:free:intermediate:line:arametric}
  \beta_v\x_v + \beta_v\delta_v\dir_v = \srot_v(\beta_1 \x_1 + \delta_1\dir_1) +
  \stransl_v,
\end{equation}
which reduces to~\eqref{eq:tangent:intermediate:v3:parametric} when the $\x_v$ match.
Together with~\eqref{eq:points:simplified:notation} it forms the essential
parametric equations for Cleveland.\\[1em]
\noindent\textbf{Equations based on Minors.}
We may describe the Cleveland problem with equations based on minors in an
analogous way to Chicago. 
We are given lines $\bl_{1v},\ldots , \bl_{4v}$ for $v\in \{ 1,2,3 \}$, where
$\bl_{1v},\bl_{2v},\bl_{3v}$ are pairwise lines and $\bl_{4v}$ is
the free line. We
enforce line correspondences for matrices $\LL_1,\ldots , \LL_4$ 
and common point constraints by requiring that the
$4\times 4$ minors of $[\LL_1\mid \LL_2]$, $[\LL_1\mid \LL_3]$, and $[\LL_2\mid \LL_3]$ all vanish, 
leading to $64 = 4*\binom{4}{3} + 3  \binom{6}{4}$ equations total in $14$ unknowns.
 
\section*{Acknowledgments}
The authors would like to thank Juliana Santos Barcellos Chagas Ventura
for helping with performance experiments, Figs.~\ref{fig:niter} and~\ref{fig:nrep}.

% Interesting refs:
% 
%\nocite{
% Line-based SfM: Bartoli:Sturm:CVIU2005,Salaun:etal:3DV2017
% Point+line: Salaun:etal:ECCV2016
% Zhao:Kneip:etal:PAMI2019,Camposeco:Pollefeys:etal:ECCV2016}
\ifCLASSOPTIONcaptionsoff
  \newpage
\fi
\Urlmuskip=0mu plus 1mu
{\small

%\bibliographystyle{./IEEEtran}
%Include ieee-config below to control how refs appear
%\bibliography{multiview,rf-multiview,Pajdla,bib-fabbri/fabbri-multiview,bib-fabbri/kimia-multiview,bib-fabbri/fabbri-video,bib-fabbri/pajdla}
}


\begin{IEEEbiography}[{\includegraphics[width=1in,height=1.25in,clip,keepaspectratio]{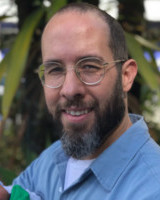}}]{Ricardo Fabbri}
is a tenured professor of 
computational modeling at the Rio de Janeiro State University, 
Brazil. He holds a Ph.D.\ from Brown University and has worked at
Google in offline indexing from images.
He has recently organized the Algebraic Vision Research
Cluster at ICERM, on multiple view geometry and differential geometry. He is 
currently interested in mathematics and diffusion maps
for the photogrammetric modeling of the water surface.
\end{IEEEbiography}
\begin{IEEEbiography}[{\includegraphics[width=1in,height=1.25in,clip,keepaspectratio]{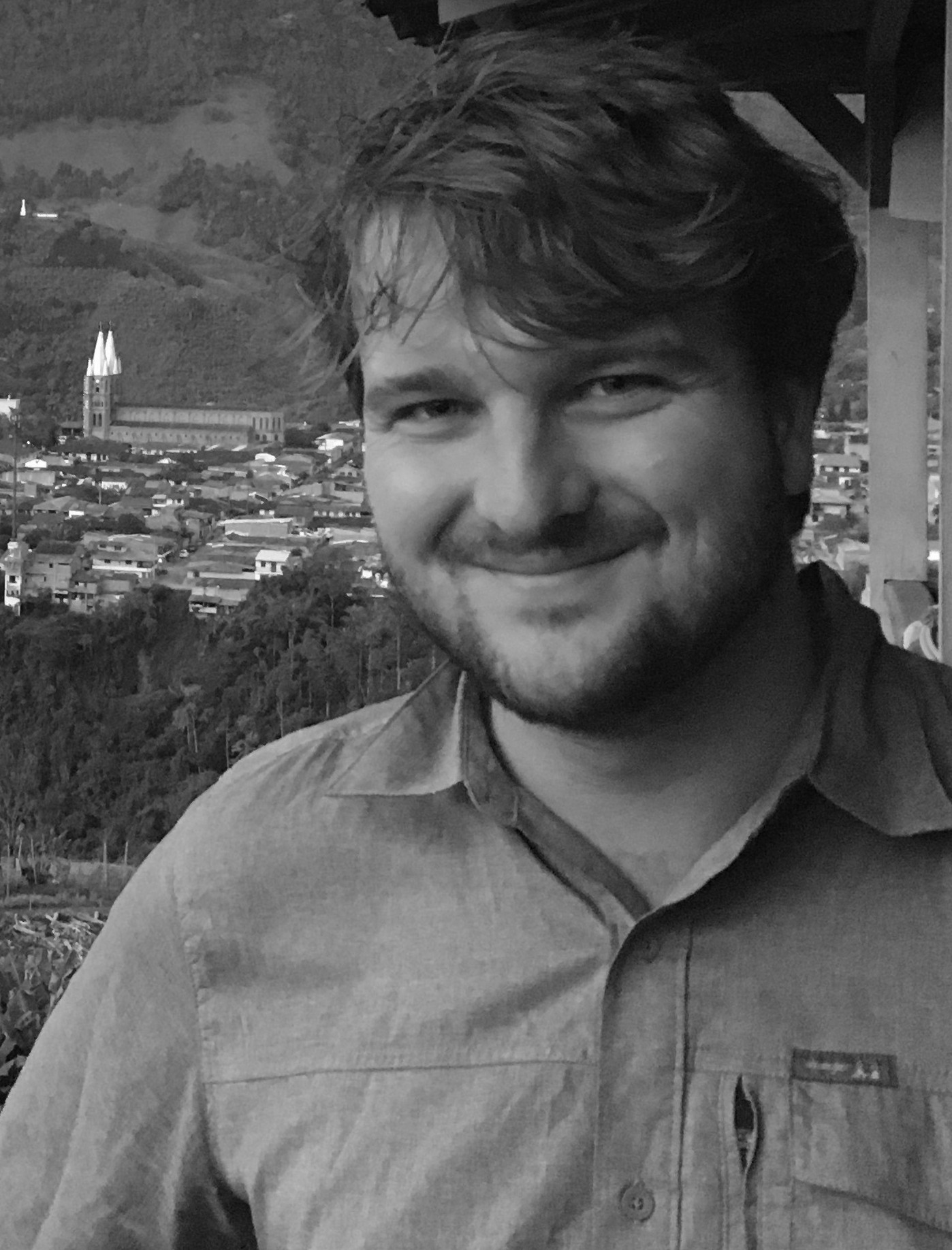}}]{Timothy~Duff} is an NSF Mathematical Sciences Postdoctoral Research Fellow and Postdoctoral Scholar at the University of Washington.
He previously completed his PhD at Georgia Tech in Algorithms, Combinatorics, and Optimization, supervised by Anton Leykin.
He is a member of SIAM and the AMS.
 \end{IEEEbiography}
\begin{IEEEbiography}[{\includegraphics[width=1in,height=1.25in,clip,keepaspectratio]{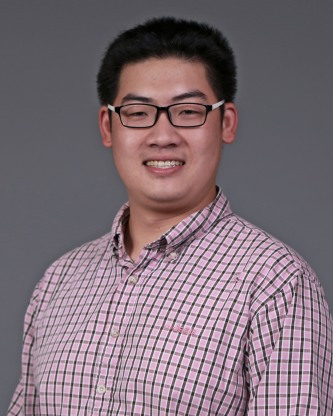}}]{Hongyi~Fan}
received the M.Sc. in computer engineering from Brown University in 2016, where he is
currently pursuing the Ph.D. degree. His research
interests include computer vision, \textsc{3d} reconstruction, minimal problems and their application. 
\end{IEEEbiography}
\begin{IEEEbiography}[{\includegraphics[width=1in,height=1.25in,clip,keepaspectratio]{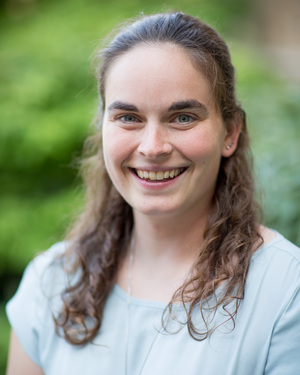}}]{Margaret~H.~Regan}
received her PhD degree in applied mathematics from the University of Notre Dame, USA, in 2020.  She is currently an Elliott Assistant Research Professor with the Department of Mathematics at Duke University.  Her research interests include numerical algebraic geometry, commutative and homological algebra, and their applications to various fields within science and engineering. She is a member of SIAM and the AMS.
\end{IEEEbiography}
\begin{IEEEbiography}[{\includegraphics[width=1in,height=1.25in,clip,keepaspectratio]{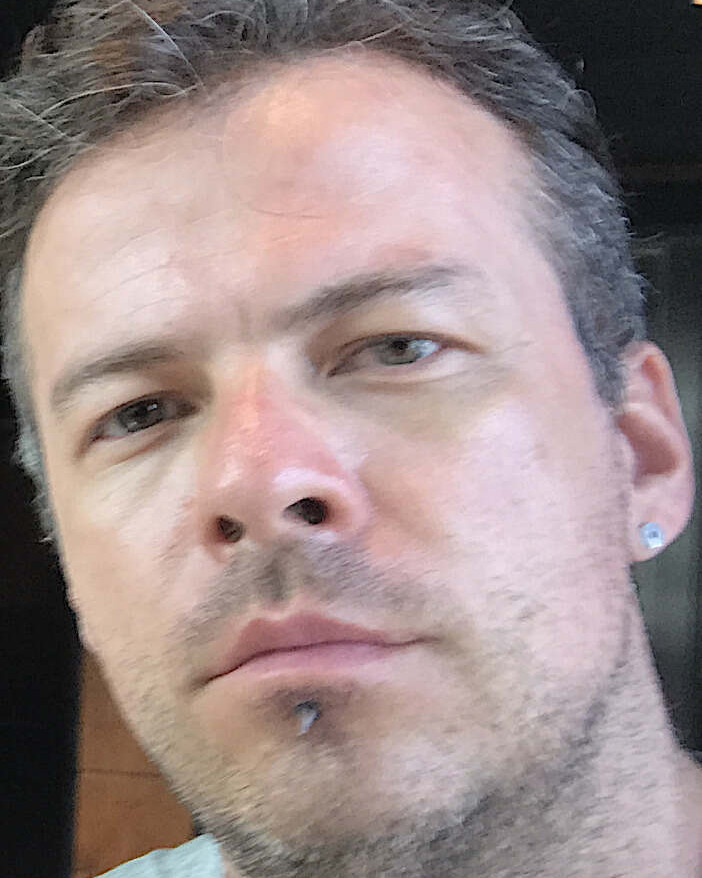}}]{David~da~C.~de~Pinho}
  is a Professor at Fluminense Federal Institute, having received his Ph.D. in
  Oil Exploration Engineering from UENF - Brazil (2020), a MSc. in Computational
  Modeling and \textsc{3d} Computer Vision from the Polytechnic Institute at UERJ -
  Brazil, and an undergraduate degree in Mathematics. 
\end{IEEEbiography}
\begin{IEEEbiography}[{\includegraphics[width=1in,height=1.25in,clip,keepaspectratio]{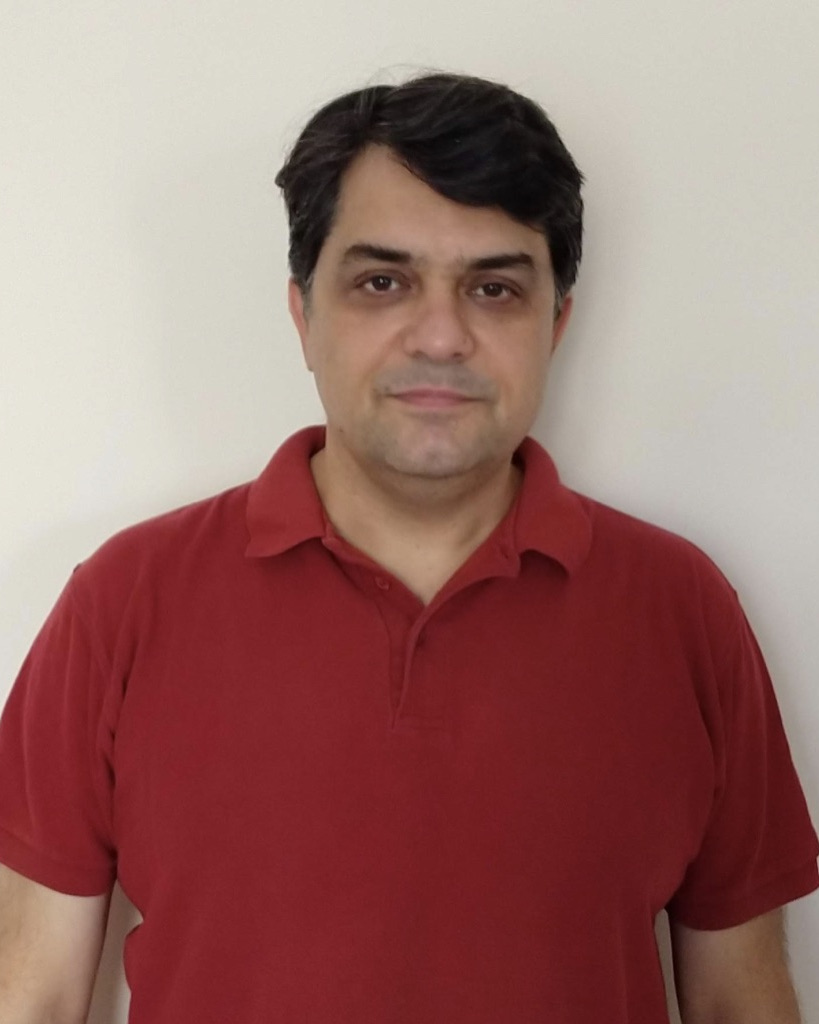}}]{Elias~Tsigaridas}
received his PhD degree from the National Kapodistrian University of Athens, Greece, in 2006. Since 2012 is a permanent senior scientist at Inria Paris; since 2020 he also holds a part-time teaching position at \'Ecole Polytechnique. He research interests lie at the intersection of computational algebra and geometry
and their application in science and engineering. 
\end{IEEEbiography}
\begin{IEEEbiography}[{\includegraphics[width=1in,height=1.25in,clip,keepaspectratio]{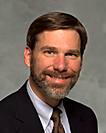}}]{Charles~W.~Wampler}
(Ph.D., Stanford University, Mechanical Engineering) is a Sr.
Technical Fellow at the General Motors R\&D Center, Warren, Michigan, USA, where he has been employed since 1985, currently as a member of the Battery Cell Systems Research Lab.  He is a Fellow of IEEE, ASME, and SIAM and is a member of the National Academy of Engineering. 
\end{IEEEbiography}
\begin{IEEEbiography}[{\includegraphics[width=1in,height=1.25in,clip,keepaspectratio]{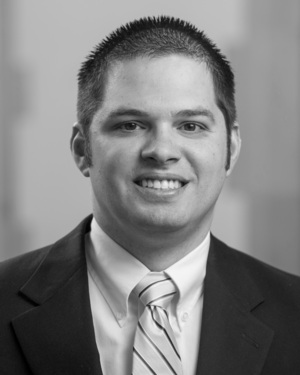}}]{Jonathan D. Hauenstein}
(Ph.D. University of Notre Dame, Mathematics).  He held positions at the Fields Institute,
Texas A\&M University, Mittag-Leffler Institute, North Carolina State
University, ICERM and the Simons Institute before returning to the University of
Notre Dame as a professor in 2014. His research includes numerical algebraic geometry and its applications involving a variety
of fields in science and~engineering.
\end{IEEEbiography}
\begin{IEEEbiography}[{\includegraphics[width=1in,height=1.25in,clip,keepaspectratio]{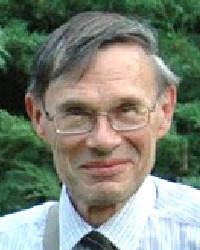}}]{Peter J.\ Giblin}
is Professor of Mathematics Emeritus at the University of Liverpool, UK.
His current research interests are in applications of singularity theory to problems of differential
geometry. He also works with local high schools in supervising students in advanced
project work in mathematics. In the queen’s birthday honours 2018 he was awarded
an OBE for services to mathematics.
\end{IEEEbiography}

%\newpage

\begin{IEEEbiography}[{\includegraphics[width=1in,height=1.25in,clip,keepaspectratio]{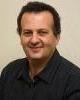}}]{Benjamin B.\ Kimia}
is Professor of Engineering at Brown Unviersity and
holds a Ph.D.\ in Electrical and
Computer Engineering from McGill University. His research includes
computer vision and medical imaging inspired by neurophysiology and psychophysics.
His expertise includes the representation of shape in \textsc{2d}, \textsc{3d} and multiview 
reconstruction, applied to large image databases,
archaeology, assistance for the blind, odometry, and image-guided treatments. 
\end{IEEEbiography}
\begin{IEEEbiography}[{\includegraphics[width=1in,height=1.25in,clip,keepaspectratio]{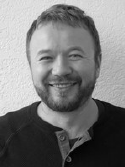}}]{Anton~Leykin}
  received his PhD from the University of Minnesota in Twin Cities. He works in
  nonlinear algebra with a view towards algorithms and applications. A large
  part of his recent work concerns homotopy continuation methods, which includes
  both theory and implementation in Macaulay2 computer algebra system. He is a
  member of the ACM, AMS, and SIAM. Google Scholar: \url{scholar.google.com/citations?user=ztNR6w8AAAAJ}
 \end{IEEEbiography}
\begin{IEEEbiography}[{\includegraphics[width=1in,height=1.25in,clip,keepaspectratio]{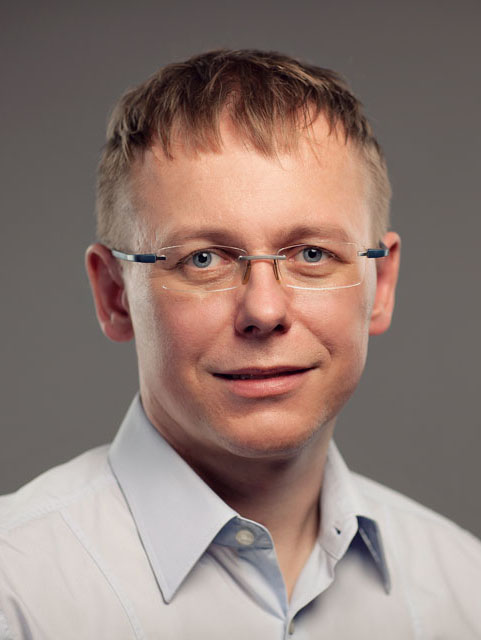}}]{Tomas~Pajdla}
(Member, IEEE) received the M.Sc.\ and Ph.D.\ from the Czech Technical University, in Prague. He works in geometry and algebra of computer vision and robotics, with emphasis on non-classical cameras, 3D reconstruction, and industrial vision. He contributed to the epipolar geometry of panoramic cameras, non-central camera models, generalized epipolar geometries, and to developing solvers for minimal problems in structure from motion. \end{IEEEbiography}


\end{document}